\documentclass[twoside,letterpaper]{article}

\usepackage{mathtools}
\mathtoolsset{showonlyrefs=true}
\usepackage{cases}

\usepackage{amsthm, amsmath, amssymb, amsfonts, graphicx, epsfig}
\usepackage{accents}
\usepackage{bm}
\usepackage{algorithm, algorithmic}
\usepackage{enumitem}
\usepackage[usenames,dvipsnames]{xcolor}

\usepackage{bbm}
\usepackage{mathrsfs}

\usepackage{ulem}

\usepackage{cancel}

\usepackage{epstopdf}

\usepackage{graphicx}
\usepackage[font=sl,labelfont=bf]{caption}
\usepackage{subcaption}

\usepackage{float}
\restylefloat{table}

\usepackage{enumerate}
\usepackage[colorlinks=true, pdfstartview=FitV, linkcolor=blue,
citecolor=blue, urlcolor=blue]{hyperref}

\usepackage{url}
\makeatletter\def\url@leostyle{%
\@ifundefined{selectfont}{\def\UrlFont{\sf}}{\def\UrlFont{\scriptsize\ttfamily}}} \makeatother

\usepackage[framemethod=default]{mdframed}

\usepackage{array}

\usepackage[margin=2.5cm, letterpaper]{geometry}

\usepackage{multicol}
\usepackage{tikz,pgfplots}
\usetikzlibrary{shapes,arrows,trees,automata,positioning,matrix,decorations.pathreplacing,calc}
\pgfplotsset{compat=1.16}
\newcommand{\tikzmark}[1]{\tikz[overlay,remember picture] \node (#1) {};}
\newcommand*{\AddNote}[4]{%
\begin{tikzpicture}[overlay, remember picture]
\draw [decoration={brace,amplitude=0.5em},decorate, thick]
($(#3)!(#1.north)!($(#3)-(0,1)$)$) --  
($(#3)!(#2.south)!($(#3)-(0,1)$)$)
    node [align=center, text width=2.5cm, pos=0.5, anchor=west] {#4};
\end{tikzpicture}
}%
\usepackage{makecell} 

\newtheorem{theorem}{Theorem}
\newtheorem{proposition}[theorem]{Proposition}
\newtheorem{lem}[theorem]{Lemma}

\theoremstyle{definition}
\newtheorem{definition}[theorem]{Definition}
\newtheorem{hyp}[theorem]{Assumption}

\theoremstyle{remark}
\newtheorem{remark}[theorem]{Remark}

\definecolor{Red}{rgb}{1.0,0,0.0}
\definecolor{Blue}{rgb}{0,0.0,1.0}
\definecolor{Green}{rgb}{0.2,0.5,0.2}

\def\cA{\mathcal{A}}
\def\cB{\mathcal{B}}

\def\cD{\mathcal{D}}

\def\cF{\mathcal{F}}

\def\cN{\mathcal{N}}

\def\cP{\mathcal{P}}
\def\cQ{\mathcal{Q}}

\def\cS{\mathcal{S}}

\def\cW{\mathcal{W}}
\def\cX{\mathcal{X}}

\def\bE{\mathbb{E}}

\def\bN{\mathbb{N}}

\def\bP{\mathbb{P}}

\def\bR{\mathbb{R}}

\def\bT{\mathbb{T}}

\def\bX{\mathbb{X}}
\def\bY{\mathbb{Y}}

\def\dX{d_{\bX}}
\def\dY{d_{\bY}}

\def\sB{\mathscr{B}}

\def\sE{\mathscr{E}}
\def\sF{\mathscr{F}}

\def\mA{\mathsf{A}}
\def\mB{\mathsf{B}}
\def\mC{\mathsf{C}}
\def\mD{\mathsf{D}}

\def\mK{\mathsf{K}}

\def\mS{\mathsf{S}}
\def\mT{\mathsf{T}}

\def\mW{\mathsf{W}}

\newcommand{\wt}{\widetilde}

\def\esp#1{\mathbb E\left[#1\right]}
\def\pro#1{\mathbb P \left[#1\right]}
\def\ind{\mathbbm{1}}

\def\Prbox{\hat P^{r}}
\def\Pwh{\widehat P^{\cA}}
\def\PkNN{\check P^{k}}
\def\PNN{\tilde P} 
\def\PNNt{\PNN^\theta}
\def\PNNT{\PNN^\Theta}
\def\Pgen{\overline P}

\newcommand{\1}{\mathbbm{1}}            

\newcommand{\set}[1]{\{#1\}}            

\DeclareMathOperator{\dif}{d \!}        

\DeclareMathOperator*{\argmin}{arg\,min} 
\DeclareMathOperator{\Var}{Var}          

\DeclareMathOperator{\ud}{\mathrm d\!}
\DeclareMathOperator{\diag}{diag}

\DeclareMathOperator{\Binomial}{\mathrm{Binomial}}
\DeclareMathOperator{\Uniform}{\mathrm{Uniform}}
\DeclareMathOperator{\Normal}{\mathrm{Normal}}

\usepackage{todonotes}
\title{Learning conditional distributions on continuous spaces}

\author{
Cyril B\'en\'ezet
\thanks{Ordered alphabetically.}
\thanks{Laboratoire de Mathématiques et Modélisation d'{\'E}vry (LaMME), Université d'{\'E}vry-Val-d'Essonne, ENSIIE, UMR CNRS 8071, IBGBI 23 Boulevard de France, 91037 {\'E}vry Cedex, France. \textbf{Email:}  cyril.benezet@ensiie.fr.}
\and
Ziteng Cheng     
\footnotemark[1]	
\thanks{Department of Statistical Sciences, University of Toronto, Canada. \textbf{Email:} ziteng.cheng@utoronto.ca.}
\and
Sebastian Jaimungal      
\footnotemark[1]
\thanks{Department of Statistical Sciences, University of Toronto, Canada. \textbf{Email:} sebastian.jaimungal@utoronto.ca. \url{http://sebastian.statistics.utoronto.ca/}} 
}

\date{
}

\begin{document}

\maketitle




\smallskip

\begin{abstract}
We investigate sample-based learning of conditional distributions on multi-dimensional unit boxes, allowing for different dimensions of the feature and target spaces. Our approach involves clustering data near varying query points in the feature space to create empirical measures in the target space. We employ two distinct clustering schemes: one based on a fixed-radius ball and the other on nearest neighbors. We establish upper bounds for the convergence rates of both methods and, from these bounds, deduce optimal configurations for the radius and the number of neighbors. We propose to incorporate the nearest neighbors method into neural network training, as our empirical analysis indicates it has better performance in practice. For efficiency, our training process utilizes approximate nearest neighbors search with random binary space partitioning. Additionally, we employ the Sinkhorn algorithm and a sparsity-enforced transport plan. Our empirical findings demonstrate that, with a suitably designed structure, the neural network has the ability to adapt to a suitable level of Lipschitz continuity locally. For reproducibility, our code is available at \url{https://github.com/zcheng-a/LCD_kNN}.    
\end{abstract}



%
%

\normalem
\section{Introduction}\label{sec:Intro}
Learning the conditional distribution is a crucial aspect of many decision-making scenarios. While this learning task is generally challenging, it presents unique complexities when explored in a continuous space setting. Below, we present a classic example (cf. \cite{Booth1992Bootstrap,Pflug2016From}) that highlights this core challenge.  

For simplicity, we suppose the following model 
\begin{align*}
    Y = \tfrac12 X + \tfrac12 U,
\end{align*}
where the feature variable $X$ and the noise $U$ are independent $\Uniform([0,1])$, and $Y$ is the target variable. Upon collecting a finite number of independent samples $\cD=\set{(X_m,Y_m)}_{m=1}^M$, we aim to estimate the conditional distribution of $Y$ given $X$. Throughout, we  treat this conditional distribution as a measure-valued function of $x$, denoted by $P_x$. A naive approach is to first form an empirical joint measure
\begin{align*}
\hat\psi := \frac1M \sum_{m=1}^M \delta_{(X_m, Y_m)},
\end{align*}
where $\delta$ stands for the Dirac meaasure, and then use the conditional distribution induced from $\hat\psi$ as an estimator. As the marginal distribution of $X$ is continuous, with probability $1$ (as $\bP(X_m = X_{m'})=0$ for all $m \neq m'$), we have that\footnote{In accordance to the model, we set the conditional distribution to $\Uniform([0,1])$ at points where it is not well-defined.}
\begin{align*}
\widehat P_x = \begin{cases}
\delta_{Y_m}, & x=X_m \text{ for some } m,\\
\Uniform([0,1]), & \text{otherwise.}
\end{cases} 
\end{align*}
Regardless of the sample size $M$, $\widehat P_x$ fails to approximate the true conditional distribution, 
$$P_x = \Uniform\left([x,x+\tfrac12]\right),\quad x\in[0,1].$$

Despite the well-known convergence of the (joint) empirical measure to the true distribution \cite{Dudley1969Speed,Fournier2015Rate}, the resulting conditional distribution often fails to provide an accurate approximation of the true distribution. This discrepancy could be due to the fact that calculating conditional distribution is an inherently unbounded operation. As a remedy, clustering is a widely employed technique. Specifically, given a query point $x$ in the feature space, we identify samples where $X_m$ is close to $x$ and use the corresponding $Y_m$'s to estimate $P_x$. Two prominent methods within the clustering approach are the kernel method and the nearest neighbors method\footnote{These should not be confused with similarly named methods used in density function estimation.}. Roughly speaking, the kernel method relies primarily on proximity to the query point for selecting $X_m$'s, while the nearest neighbors method focuses on the rank of proximity. Notably, discretizing the feature space (also known as quantization), a straightforward yet often effective strategy, can be seen as a variant of the kernel method with static query points and flat kernels. 

The problem of estimating conditional distributions can be addressed within the non-parametric regression framework, by employing clustering or resorting to non-parametric least squares, among others. Alternatively, it is feasible to estimate the conditional density function directly: a widely-used method involves estimating the joint and marginal density functions using kernel smoothing and then calculating their ratio. This method shares similarities with the clustering heuristics mentioned earlier. For a more detailed review of these approaches, we refer to Section \ref{subsec:RelatedWorks}.

This work draws inspiration from recent advancements in estimating discrete-time stochastic processes using conditional density function estimation \cite{Pflug2016From} and quantization methods \cite{Backhoff2022Estimating,Acciaio2023Convergence}. A notable feature of these works is their use of the Wasserstein distance to calculate local errors: the difference between the true and estimated conditional distributions at a query point $x$. One could average these local errors across different values of $x$'s to gauge the global error. Employing Wasserstein distances naturally frames the study within the context of weak convergence, thereby enabling discussions in a relatively general setting, although this approach may yield somewhat weaker results in terms of the mode of convergence. Moreover, utilizing a specific distance rather than the general notion of weak convergence enables a more tangible analysis of the convergence rates and fluctuations. We would like to point out that the advancements made in \cite{Backhoff2022Estimating,Acciaio2023Convergence}, as well as our analysis in this paper, relies on recent developments concerning the Wasserstein convergence rate of empirical measures under i.i.d.~sampling from a static distribution (cf. \cite{Fournier2015Rate}).

\subsection{Main contributions}
First, we introduce some notations to better illustrate the estimators that we study. Let $\bX$ and $\bY$ be multi-dimensional unit cubes, with potentially different dimensions, for feature and target spaces. For any integer $M \ge 1$, any $\mD = \{(x_m,y_m)\}_{m=1}^M \in (\bX\times\bY)^M$, and any Borel set $A \subset \bX$, we define a probability measure on $\bY$ by
\begin{align}\label{eq: mu hat}
\hat\mu^{\mD}_A :=
\begin{cases}
\left(\sum_{m=1}^M \1_A(x_m) \right)^{-1} \sum_{m=1}^{M} \1_A(x_m) \delta_{y_m}, & \sum_{m=1}^{M} \1_A(x_m) > 0, \\
\lambda_\bY, & \mbox{otherwise, }
\end{cases}
\end{align}
where $\lambda_\bY$ is the Lebesgue measure on $\bY$ and, for $y \in \bY$, $\delta_y$ is a Dirac measure with atom at $y$. In general, one could consider weighting $\delta_{y_m}$'s (cf. \cite[Section 5]{Gyorfi2002Distribution}, \cite[Chapter 5]{Biau2015Lectures}), which may offer additional benefits in specific applications. As such adjustments are unlikely to affect the convergence rate, however, we use uniform weighting for simplicity.

With (random) data $\cD=\set{(X_m,Y_m)}_{m=1}^M$, we aim to estimate the conditional distribution of $Y$ given $X$. We view this conditional distribution as a measure-valued function $P:\bX\to\cP(\bY)$ and use a subscript for the input argument and write $P_x$.
Consider a clustering scheme\footnote{In general, the clustering scheme may require information on $(x_1,\dots,x_M)$. For example, clustering the $k$-nearest-neighbor near a query point $x$ requires to know all $x_m$'s. } given by the map $\cA^{\mD}:\bX\to 2^{\bX}$.  
We investigate estimators of the form $x\mapsto\hat\mu^{\cD}_{\cA^{\cD}(x)}$.  
We use $\Pwh$ to denote said estimator and suppress $\cD$ from the notation for convenience. In later sections, we consider two kinds of maps $\cA^{\mD}$ (i) a ball with fixed radius centered at $x$, called an $r$-box and (ii) the $k$ nearest neighbors of $x$, called $k$-nearest-neighbor estimator. Wee Definitions \ref{def:rbox} and \ref{def:kNN} for more details.

One of our main contribution pertains to analyzing the error 
\begin{align}\label{eq:AveWError}
\int_{\bX} \cW\left(P_x, \Pwh_x\right)\;\nu(\dif x),
\end{align}
where $\cW$ is the $1$-Wasserstein distance (cf. \cite[Particular Case 6.2]{Villani2008book}) and $\nu\in\cP(\bX)$ is arbitrary  and provides versatility to the evaluation criterion. A canonical choice for $\nu$ is the Lebesgue measure on $\bX$, denoted by $\lambda_\bX$. This is particularly relevant in control settings where $\bX$ represents the state-action space and accurate approximations across various state and action scenarios are crucial for making informed decisions. The form of error above is also foundational in stochastic process estimation under the adapted Wasserstein distance (cf. \cite[Lemma 3.1]{Backhoff2022Estimating}), making the techniques we develop  potentially relevant in other contexts. 
Under the assumption that $P$ is Lipschitz continuous (Assumption \ref{hyp: kernel lip}) and  standard assumptions on the data collection process (Assumption \ref{hyp: data}), we analyze the convergence rate and fluctuation by bounding the following two quantities
\begin{gather*}
\bE\left[ \int_{\bX} \cW\left(P_x, \Pwh_x \right)\nu(\dif x) \right] \quad\text{and}\quad \Var\left[ \int_{\bX} \cW\left(P_x, \Pwh_x\right)\nu(\dif x) \right].
\end{gather*}
Moreover, by analyzing the above quantities, we gain insights into the optimal choice of the clustering mapping $\cA$. For the detail statements of these results, we refer to Theorems \ref{thm:ExpectedRaterbox}, \ref{thm:ExpectedRatekNN}, \ref{thm:Concenrboxnew}, and \ref{thm:ConcenkNNnew}. We also refer to Section \ref{subsec:Comments} for related comments.

To illustrate another aspect of our contribution, we note by design $x\mapsto\Pwh_x$ is piece-wise constant. This characteristic introduces limitations. Notably, it renders the analysis of performance at the worst-case $x$ elusive. Contrastingly, by building a Lipschitz-continuous parametric estimator $\PNNT$ from the raw estimator $\Pwh$, in Proposition \ref{prop:PThetanew} we demonstrate that an upper bound on the aforementioned expectation allows us to derive a worst-case performance guarantee. Guided by Proposition \ref{prop:PThetanew}, we explore a novel approach of training a neural network for estimation, by using $\Pwh$ as training data and incorporating suitably imposed Lipschitz continuity. To be comprehensive, we include in Section \ref{subsec:RelatedWorks} a review of studies on Lipschitz continuity in neural networks.

In Section \ref{subsec:ImplOverview}, we define $\PNNt$ as a neural network that approximates $P$, where $\theta$ represents the network parameters. We train $\PNNt$ with the objective:
\begin{align*}
\argmin_{\theta}\sum_{n=1}^N \cW\left(\Pwh_{\tilde X_n}, \PNNt_{\tilde X_n}\right),
\end{align*}
where $(\tilde X_n)_{n=1}^N$ is a set of randomly selected query points.  For implementation purposes, we use the $k$-nearest-neighbor estimator in the place of $\Pwh$ (see Definition \ref{def:kNN}). To mitigate the computational costs stemming from the nearest neighbors search, we employ the technique of Approximate Nearest Neighbor Search with Random Binary Space Partitioning (ANN-RBSP), as discussed in Section \ref{subsubsec:ANNS}. In Section \ref{subsubsec:W}, we compute $\cW$ using the Sinkhorn algorithm, incorporating  normalization and enforcing sparsity for improved accuracy. 
To impose a suitable level of local Lipschitz continuity on $\PNNt$, in Section \ref{subsubsec:Net}, we employ a neural network with a specific architecture and train the networks using a tailored procedure. The key component of this architecture is the convex potential layer introduced in \cite{Meunier2022Dynamical}. In contrast to most extant literature that imposes Lipschitz continuity on neural networks, our approach does not utilize specific constraint or regularization of the objective function, but relies on certain self-adjusting mechanism embedded in the training.  

In Section \ref{subsec:Experiments}, we evaluate the performance of the trained $\PNNt$, denoted by $\PNNT$, using three sets of synthetic data in $1$D and $3$D spaces. Our findings indicate that $\PNNT$ generally outperforms $\Pwh$, even though it is initially trained to match $\Pwh$. This superior performance persists even when comparing $\PNNT$ to different $\Pwh$ using various $k$ values, without retraining $\PNNT$. Furthermore, despite using the same training parameters, $\PNNT$ consistently demonstrates the ability to adapt to a satisfactory level of local Lipschitz continuity across all cases. Moreover, in one of the test cases, we consider a kernel that exhibits a jump discontinuity, and we find that $\PNNT$ handles this jump case well despite Lipschitz continuity does not hold.

Lastly, we provide further motivation of our approach by highlighting some potential applications for $\PNNT$. The first application is in model-based policy gradient method in reinforcement learning. We anticipate that the enforced Lipschitz continuity allows us to directly apply the policy gradient update via compositions of $\PNNT$ and cost function for more effective optimality searching. The second application of $\PNNT$ is in addressing optimisation in risk-averse Markov decision processes, where dynamic programming requires knowledge beyond the conditional expectation of the risk-to-go (cf. \cite{Chow2015Risk,Huang2017Risk,Coache2023conditionally,Cheng2023Distributional}). The study of these applications is left for further research.

\subsection{Related works}\label{subsec:RelatedWorks}
In this section, we will first review the clustering approach in estimating conditional distributions, and then proceed to review recent studies on Lipschitz continuity in neural networks.

\subsubsection{Estimating conditional distributions via clustering}
The problem of estimating conditional distributions is frequently framed as non-parametric regression problems for real-valued functions. For instance,  when $d_\bY = 1$, estimate the conditional $\alpha$-quantile of $Y$ given $X$. Therefore, we begin by reviewing some of the works in non-parametric regression. 

The kernel method in non-parametric regression traces its origins back to the Nadaraya-Watson estimator \cite{Nadaraya64Estimating, Watson64Smooth}, if not earlier. Subsequent improvements have been introduced, such as integral smoothing \cite{Gasser79Kernel} (also known as the Gasser-Müller estimator), local fitting with polynomials instead of constants \cite{Fan92Design}, and adaptive kernels \cite{Hall1999Methods}. Another significant area of discussion is the choice of kernel bandwidth, as detailed in works like \cite{Hardle1985Optimal, Gijbels2004Bandwidth, Kohler2014Review}. Regarding convergence rates, analyses under various settings can be found in \cite{Stone1982Optimal, Hall1990Nonparametric, Kohler2009Optimal, Li2010Optimal}, with \cite{Stone1982Optimal} being particularly relevant to our study for comparative purposes. According to \cite{Stone1982Optimal}, if the target function is Lipschitz continuous, with i.i.d.~sampling and that the sampling distribution in the feature space has a uniformly positive density, then the optimal rate of the $\|\cdot\|_1$-distance between the regression function and the estimator is of the order $M^{-\frac{1}{d_\bX+2}}$. For a more comprehensive review of non-parametric regression using kernel methods, we refer to the books \cite{Gyorfi2002Distribution, Ferraty2006Nonparametric, Wasserman2006All} and references therein.

Non-parametric regression using nearest neighbors methods originated from classification problems \cite{Fix1951Discriminatory}. Early developments in this field can be found in \cite{Mack1981Local, Devroye1982Necessary, Bhattacharya1990Kernel}. For a comprehensive introduction to nearest neighbors methods, we refer to \cite{Gyorfi2002Distribution}. More recent reference \cite{Biau2015Lectures} offers further detailed exploration of the topic. The nearest neighbor method can be viewed as a variant of the kernel method that adjusts the bandwidth based on the number of local data points—a property that has gained significant traction. Recently, the application of the nearest neighbor method has expanded into various less standard settings, including handling missing data \cite{Rachdi2021k}, reinforcement learning \cite{Shad2018Q,Giegrich2024K}, and time series forecasting \cite{Martinez2017methodology}. For recent advancements in convergence analysis beyond the classical setting, see \cite{Zhao2019Minimax, Padilla2020Adaptive, Ryu2022Minimax, Demirkayaa2024Optimal}.

Although the review above mostly focuses on clustering approach, other effective approaches exist, such as non-parametric least square, or more broadly, conditional elicitability (e.g., \cite{Gyorfi2002Distribution, Muandet2017Kernel, Wainwright2019High, Coache2023conditionally}). Non-parametric least square directly fits the data using a restricted class of functions. At first glance, this approach appears distinct from clustering. However, they share some similarities in their heuristics: the rigidity of the fitting function, due to imposed restrictions, allows data points near the query point to affect the estimation, thereby implicitly incorporating elements of clustering. 

Apart from non-parametric regression, conditional density function estimation is another significant method for estimating conditional distributions. One approach is based on estimating joint and marginal density functions, and then using the ratio of these two to produce an estimator for the conditional density function. A key technique used in this approach is kernel smoothing. Employing a static kernel for smoothing results in a conditional density estimator that shares similar clustering heuristics to those found in the kernel method of non-parametric regression. For a comprehensive overview of conditional density estimation, we refer to reference books \cite{Scott2015Multivariate, Simonoff1996Smoothing}. For completeness, we also refer to \cite[Section 5.1]{Muandet2017Kernel} for a perspective on static density function estimation from the standpoint of reproducing kernel Hilbert space. Further discussions on estimation using adaptive kernels can be found in, for example, \cite{Bashtannyk2001Bandwidth,Lacour2007Adaptive,Bertin2016Adaptive,Zhao2023Adaptive}. 

Despite extensive research in non-parametric regression and conditional density function estimation, investigations from the perspective of weak convergence have been relatively limited, only gaining more traction in the past decade. Below, we highlight a few recent studies conducted in the context of estimating discrete-time stochastic processes under adapted Wasserstein distance, as the essence of these studies are relevant to our evaluation criterion \eqref{eq:AveWError}. \cite{Pflug2016From} explores the problem asymptotically, employing tools from conditional density function estimation with kernel smoothing. Subsequently, \cite{Backhoff2022Estimating} investigates a similar problem with a hypercube as state space, employing the quantization method. Their approach removes the need to work with density functions. They calculate the convergence rate, by leveraging recent developments in the Wasserstein convergence rate of empirical measures \cite{Fournier2015Rate}. Moreover, a sub-Gaussian concentration with parameter $M^{-1}$ is established. The aforementioned results are later extended to $\bR^d$ in \cite{Acciaio2023Convergence}, where a non-uniform grid is used to mitigate assumptions on moment conditions. Most recently, \cite{Hou2024Convergence} examines smoothed variations of the estimators proposed in \cite{Backhoff2022Estimating, Acciaio2023Convergence}. Other developments on estimators constructed from smoothed quantization can be found in \cite{Smid2024Approximation}.

Lastly, regarding the machine learning techniques used in estimating conditional distributions, conditional generative models are particularly relevant. For reference, see \cite{Mirza2014Conditional, Papamakarios2017Masked, Vaswani2017Attention, Fetaya2020Understanding}. These models have achieved numerous successes in image generation and natural language processing. We suspect that, due to the relatively discrete (albeit massive) feature spaces in these applications, clustering is implicitly integrated into the training procedure. In continuous spaces, under suitable setting, clustering may also become an embedded part of the training procedure. For example, implementations in \cite{Li2020ACFlow,Vuletic2024Fin,Hosseini2024Conditional} do not explicitly involve clustering and use training objectives that do not specifically address the issues highlighted in the motivating example at the beginning of the introduction. Their effectiveness could possibly be attributed to certain regularization embedded within the neural network and training procedures. Nevertheless, research done in continuous spaces that explicitly uses clustering approaches when training conditional generative models holds merit. Such works are relatively scarce. For an example of this limited body of research, we refer to \cite{Xu2022Conditional}, where the conditional density function estimator from \cite{Pflug2016From} is used to train an adversarial generative network for stochastic process generation.

\subsubsection{Lipschitz continuity in neural networks}
Recently, there has been increasing interest in understanding and enforcing Lipschitz continuity in neural networks. The primary motivation is to provide a certifiable guarantee for classification tasks performed by neural networks: it is crucial that minor perturbations in the input object have a limited impact on the classification outcome. 

One strategy involves bounding the Lipschitz constant of a neural network, which can then be incorporated into the training process. For refined upper bounds on the (global) Lipschitz constant, see, for example, \cite{Bartlett2017Spectrally, Virmaux2018Lipschitz, Tsuzuku2018Lipschitz, Fazlyab2019Efficient, Xue2022Chordal,Fazlyab2024Certified}. For local bounds, we refer to \cite{Jordan2021Exactly, Bhowmick2021LipBaB, Shi2022Efficiently} and the references therein. We also refer to \cite{Zhang2022Rethinking} for a study of the Lipschitz property from the viewpoint of boolean functions.

Alternatively, designing neural network architectures that inherently ensure desirable Lipschitz constants is another viable strategy. Works in this direction include \cite{Meunier2022Dynamical, Singla2022Improved, Wang2023Direct, Araujo2023Unified}. Notably, the layer introduced in \cite{Meunier2022Dynamical} belongs to the category of residual connection \cite{He2016Deep}.

Below, we review several approaches that enforce Lipschitz constants during neural network training. \cite{Tsuzuku2018Lipschitz, Liu2022Learning} explore training with a regularized objective function that includes upper bounds on the network's Lipschitz constant. \cite{Gouk2021Regularisation} frame the training problem into constrained optimization and train with projected gradients descent. Given the specific structure of the refined bound established in \cite{Fazlyab2019Efficient}, \cite{Pauli2022Training} combines training with semi-definite programming. They develop a version with a regularized objective function and another that enforces the Lipschitz constant exactly. \cite{Fazlyab2024Certified} also investigates training with a regularized objective but considers Lipschitz constants along certain directions. \cite{Huang2021Training} devises a training procedure that removes components from the weight matrices to achieve smaller local Lipschitz constants. \cite{Trockman2021Orthogonalizing} initially imposes orthogonality on the weight matrices, and subsequently enforces a desirable Lipschitz constant based on that orthogonality. Ensuring desirable Lipschitz constants with tailored architectures, \cite{Singla2022Improved, Wang2023Direct} train the networks directly. 
Although the architecture proposed in \cite{Meunier2022Dynamical} theoretically ensures the Lipschitz constant, it requires knowledge of the spectral norm of the weight matrices, which does not admit explicit expression in general. Their training approach combines power iteration for spectral norm approximation with the regularization methods used in \cite{Tsuzuku2018Lipschitz}.

Finally, we note that due to their specific application scenarios, these implementations concern relatively stringent robustness requirements and thus necessitate more specific regularization or constraints. In our setting, it is generally desirable for the neural network to automatically adapt to a suitable level of Lipschitz continuity based on the data, while also avoiding excessive oscillations from over-fitting. The literature directly addressing this perspective is limited (especially in the setting of conditional distribution estimation). We refer to \cite{Bai2021Recent,Bountakas2023Defense, Cohen2019Certified} for discussions that could be relevant.

\subsection{Organization of the paper}
Our main theoretical results are presented in Section \ref{sec:Theory}. Section \ref{sec:ImplNN} is dedicated to the training of $\PNNT$. We will outline the key components of our training algorithm and demonstrate its performance on three sets of synthetic data. We will prove the theoretical results in Section \ref{sec:Proofs}. Further implementation details and ablation analysis are provided in Section \ref{sec:ImplD}. In Section \ref{sec:Weakness}, we discuss the weaknesses and potential improvements of our implementation. Appendix \ref{sec:MorePlots} and \ref{sec:Config} respectively contain additional plots and a table that summarizes the configuration of our implementation. Additionally, Appendix \ref{sec:AnotherFluc} includes a rougher version of the fluctuation results.

\section*{Notations and terminologies}
\noindent
Throughout,  we adopt the following set of notations and terminologies.
\begin{itemize}[label=$\circ$]

\item On any normed space $(E, \| \cdot \|)$, for all $x \in E$ and $\gamma > 0$, $B(x,\gamma)$ denotes the closed ball of radius $\gamma$ around $x$, namely $B(x,\gamma)=\left\{x' \in E \,\middle|\; \|x-x'\|\le\gamma \right\}$.

\item For any measurable space $(E,\sE)$, $\cP(E)$ denotes the set of probability distributions on $(E,\sE)$. For all $x \in E$, $\delta_x \in \cP(E)$ denotes the Dirac mass at $x$.

\item We endow normed spaces $(E,\|\cdot\|)$  with their Borel sigma-algebra $\sB(E)$, and $\cW$ denotes the $1$-Wasserstein distance on $\cP(E)$.

\item On $\bX = [0,1]^d$, we denote by $\lambda_\bX$ the Lebesgue measure. We say a measure $\nu\in\cP(\bX)$ is dominated by Lebesgue measure with a constant $\overline C>0$ if $\nu(A)\le\overline C\lambda_\bX(A)$ for all $A\in\sB([0,1]^d)$.

\item The symbol $\sim$ denotes equivalence in the sense of big O notation, indicating that each side dominates the other up to a multiplication of some positive absolute constant. More precisely, $a_n\sim b_n$ means there are finite constants $c,C>0$ such that 
\begin{align*}
c\, a_n \le b_n \le C\, a_n,\quad n\in\bN.
\end{align*}
Similarly, $\lesssim$ implies that one side is of a lesser or equal, in the sense of big O notation, compared to the other.
\end{itemize}

\section{Theoretical results} \label{sec:Theory}
In Section \ref{subsec:Setup}, we first formally set up the problem and introduce some technical assumption. We then study in Section \ref{subsec:rBox} and \ref{subsec:kNN} the convergence and fluctuation of two versions of $\Pwh$, namely, the $r$-box estimator and the $k$-nearest-neighbor estimator. Related comments are organized in Section \ref{subsec:Comments}. Moreover, in Section \ref{subsec:ToNN}, we provide a theoretical motivation for the use of $\PNNT$, the Lipschitz-continuous parametric estimator trained from $\Pwh$.

\subsection{Setup}\label{subsec:Setup}
For $\dX, \dY \ge 1$ two integers, we consider $\bX := [0,1]^{\dX}$ and $\bY := [0,1]^{\dY}$, endowed with their respective sup-norm $\|\cdot\|_\infty$. 

\begin{remark}
The sup-norm is chosen for simplicity of the theoretical analysis only: as all  norms on $\bR^n$ are equivalent (for any generic $n \ge 1$), our results are  valid, up to different multiplicative constants, for any other choice of norm.
\end{remark}

We aim to estimate an unknown probabilistic kernel 
\begin{align}
\label{eq: kernel}
P : \bX &\to \cP(\bY) \\
x &\mapsto P_x(\ud y).
\end{align}

To this end, given an integer-valued \emph{sampled size} $M \ge 1$, we consider a set of (random) data points $\cD := \{(X_m, Y_m)\}_{m=1}^M$ associated to $P$. We also define the set of projections of the data points onto the feature space as $\cD_\bX := \{X_m\}_{m=1}^M$.

Throughout this section, we work under the following technical assumptions.
\begin{hyp}[Lipschitz continuity of kernel] 
\label{hyp: kernel lip}
There exists $L \ge 0$ such that, for all $(x,x') \in \bX^2$,
\begin{align*}
\cW(P_x, P_{x'}) \le L\|x-x'\|_\infty.
\end{align*}
\end{hyp}

\begin{hyp} \label{hyp: data}
The following is true:
\begin{itemize}

\item[(i)] $\cD$ is i.i.d.\!\! with probability distribution $\psi := \xi \otimes P$, where $\xi \in \cP(\bX)$ and where $\xi \otimes P \in \cP(\bX \times \bY)$ is (uniquely, by Caratheodory extension theorem) defined by
\begin{align*}
\left(\xi \otimes P\right)(A \times B) := \int_\bX \1_A(x) P_x(B) \xi(\ud x), \quad A \in \sB(\bX), B \in \sB(\bY).
\end{align*}

\item[(ii)] There exists $\underline c \in (0,1]$ such that, for all $A \in \sB(\bX)$, $\xi(A) \ge \underline c \;\lambda_{\bX}(A)$.
\end{itemize}
\end{hyp}

These assumptions allow us to analyze  convergence and gain insights into the optimal clustering hyper-parameters without delving into excessive technical details. Assumption \ref{hyp: kernel lip} is mainly used for determining the convergence rate. If the convergence rate is not of concern, it is possible to establish asymptotic results with less assumptions. We refer to \cite{Devroye1982Necessary,Backhoff2022Estimating} for relevant results. The conditions placed on $\xi$ in Assumption \ref{hyp: data} are fairly standard, though less stringent alternatives are available. For instance, Assumption \ref{hyp: data} (i) can be weakened by considering suitable dependence \cite{Hall1990Nonparametric} or ergodicity in the context of stochastic processes \cite{Rudolf2018Perturbation}. Assumption \ref{hyp: data} (ii), implies there is mass almost everywhere and is aligned with the motivation from control settings discussed in the introduction. 
Assumptions \ref{hyp: kernel lip} and  \ref{hyp: data} are not exceedingly stringent and provides a number of insights into the estimation problem. More general settings are left for further research.

The estimators discussed in subsequent sections are of the form $\Pwh$, as introduced right after \eqref{eq: mu hat}, for two specific choices of clustering schemes $\cA$ constructed with the data $\cD$.
\begin{remark}
In the following study, we assert all the measurability needed for $\Pwh$ to be well-defined. These measurability can be verified using standard measure-theoretic tools listed in, for example, \cite[Section 4 and 15]{Aliprantis2006book}.
\end{remark}

\subsection{Results on $r$-box estimator}\label{subsec:rBox}
The first estimator, which we term the \emph{$r$-box estimator}, is defined as follows.

\begin{definition}\label{def:rbox}
Choose $r$, a real number, s.t. $0<r<\frac12$. The \emph{$r$-box estimator} for $P$ is defined by
\begin{align}\label{eq: r-box def}
\Prbox : \bX &\to \cP(\bY) \\
x &\mapsto \Prbox_x := \hat\mu^\cD_{\cB^r(x)},
\end{align}
where, for all $x \in \bX$, $\cB^r(x) := B(\beta^r(x), r)$ and $\beta^r(x) := r \vee x \wedge (1-r)$, where $r\vee\cdot$ and $\cdot\wedge (1-r)$ are applied entry-wise. 

\end{definition}

\begin{remark}\label{rem: rbox radius}
The set $\cB^r(x)$ is defined such that it is a ball of radius around $x$ whenever $x$ is at least $r$ away from the boundary $\partial\bX$ (in all of its components), otherwise, we move the point $x$ in whichever components are within $r$ from $\partial\bX$ to be a distance $r$ away from $\partial\bX$. Consequently, for all $0<r<\frac12$ and for all $x \in \bX$, $\cB^r(x)$ has a \emph{bona fide} radius of $r$, as the center $\beta^r(x)$ is smaller or equal to $r$ away from $\partial\bX$. 
\end{remark}

For the $r$-box estimator, we have the following convergence results. 
The theorem below discusses the convergence rate of the average Wasserstein distance between the unknown kernel evaluated at any point and its estimator, when the radius $r$ is chosen optimally with respect to the data sample $M$. Section \ref{subsec:Pf:thm:ExpectedRaterbox} is dedicated to its proof.

\begin{theorem}\label{thm:ExpectedRaterbox}
Under Assumptions \ref{hyp: kernel lip} and \ref{hyp: data}, choose $r$ as follows
\begin{align*}
 	r \sim \begin{cases} 
M^{-\frac{1}{\dX + 2}}, & \dY=1, 2 \\
M^{-\frac{1}{\dX + \dY}}, & \dY\ge 3.
\end{cases}
\end{align*}
Then, there is a constant $C>0$ (which depends only on $\dX,\dY,L,\underline c$), such that, for all probability distribution $\nu \in \cP(\bX)$, we have
\begin{align}\label{eq:ExpectedRaterbox}
\esp{ \int_{\bX} \cW\left(P_x, \Prbox_x\right) \nu(\dif x) } \le \sup_{x\in\bX} \esp{ \cW\left(P_x, \Prbox\right) }  \le 
C\times \begin{cases}
M^{-\frac{1}{{d_\bX} + 2}}, & {d_\bY}=1,\\
M^{-\frac{1}{{d_\bX} + 2}}\ln(M), & {d_\bY}=2,\\
M^{-\frac{1}{{d_\bX} + {d_\bY}}}, & {d_\bY}\ge 3.
\end{cases}
\end{align}
\end{theorem}

Next, we bound the associated variance whose proof is postponed to Section \ref{subsec:Pf:thm:Concenrboxnew}.

\begin{theorem}\label{thm:Concenrboxnew}
Under Assumptions \ref{hyp: data}, consider $r\in(0,\frac12]$. Let $\nu \in \cP(\bX)$ be dominated by $\lambda_\bX$ with a constant $\overline C>0$. Then, 
\begin{align}\label{eq:Concenrboxnew}
\Var\left[ \int_{\bX} \cW\left(P_x, \Prbox_x\right) \dif \nu(x) \right] \le \frac{4^{\dX+1}\overline C^2}{\underline c^2 (M+1)}.
\end{align}

\end{theorem}

\subsection{Results on $k$-nearest-neighbor estimator}\label{subsec:kNN}

Here, we focus in the second estimator -- the \emph{$k$-nearest-neighbor estimator}, defined as follows. 
\begin{definition}\label{def:kNN}
Let $k \ge 1$ an integer. The \emph{$k$-nearest-neighbor estimator} for $P$ is defined by
\begin{align}\label{eq: k-nn def}
\PkNN : \bX &\to \cP(\bY) \\
x &\mapsto \PkNN_x := \hat\mu^\cD_{\cN^{k,\cD_\bX}(x)}, 
\end{align}
where, for any integer $M \ge 1$ and any $\mD_\bX \in \bX^M$, $\cN^{k,\mD_\bX}(x)$ contains (exactly) $k$ points of $\mD_\bX$ which are closest to $x$, namely
\begin{align*}
\cN^{k,\mD_\bX}(x) := \Big\{ x' \in \mD_\bX \,\big|\, \| x-x'\|_\infty \mbox{ is among the } k \mbox{-smallest of } (\|x-x'\|_\infty)_{x' \in \mD_\bX} \Big\},
\end{align*}
Here, in case of a tie when choosing the $k$-th smallest, we break the tie randomly with uniform probability.
\end{definition}

We have the following analogs of the convergence results (Theorems \ref{thm:ExpectedRaterbox} and \ref{thm:Concenrboxnew}) for the $k$-nearest-neighbor estimator. The proofs are postponed to Section \ref{subsec:Pf:thm:ExpectedRatekNN} and Section \ref{subsec:Pf:thm:ConcenkNNnew}, respectively.

\begin{theorem}\label{thm:ExpectedRatekNN}
Under Assumptions \ref{hyp: kernel lip} and \ref{hyp: data}, and choosing $k$ as
\begin{align*}
k \sim \begin{cases}
M^{\frac2{{d_\bX}+2}}, & {d_\bY} = 1,2,\\
M^{\frac{{d_\bY}}{{d_\bX}+{d_\bY}}}, & {d_\bY}\ge 3,
\end{cases}
\end{align*}
there is a constant $C>0$ (which depends only on ${d_\bX}, {d_\bY}, L, \underline c$), such that, for all probability distribution $\nu \in \cP(\bX)$, we have
\begin{align}\label{eq:ExpectedRatekNN}
\esp{ \int_{\bX} \cW\left(P_x, \PkNN_x\right) \nu(\dif x) } \le \sup_{x\in\bX} \esp{ \cW\left(P_x, \PkNN_x\right) }  \le C \times \begin{cases}
M^{-\frac{1}{{d_\bX} + 2}}, & {d_\bY}=1,\\
M^{-\frac{1}{{d_\bX} + 2}}\ln M, & {d_\bY}=2,\\
M^{-\frac{1}{{d_\bX} + {d_\bY}}}, & {d_\bY}\ge 3.
\end{cases}
\end{align}
\end{theorem}

\begin{theorem}\label{thm:ConcenkNNnew}
Under Assumptions \ref{hyp: data}, for any $\nu \in \cP(\bX)$, we have 
\begin{align}\label{eq:UBVarNoDom}
\Var\left[\int_{\bX} \cW\left(P_x, \PkNN_x\right) \nu(\dif x)\right] \le \frac{1}{k}.
\end{align}
Moreover, if $\nu$ is dominated by $\lambda_\bX$ with a constant $\overline C>0$, then 
\begin{align*}
&\Var\left[\int_{\bX} \cW\left(P_x, \PkNN_x\right) \nu(\dif x)\right]\\
&\quad\le \frac{2^{2d_\bX+1}\overline C^2 M}{\underline c^2 k^2} \Bigg( \left( 8\sqrt{\frac{2d_\bX\ln(M)}{M-1}} + \frac{k}{M-1}\right)^2  + \frac{\sqrt{2\pi}}{\sqrt{M-1}}\left(8 \sqrt{\frac{2d_\bX\ln(M)}{M-1}} + \frac{k}{M-1}\right) + \frac{4}{M-1} \Bigg).
\end{align*}
With $k$ chosen as in Theorem \ref{thm:ExpectedRatekNN}, this reduces to
\begin{align*}
\Var\left[\int_{\bX} \cW\left(P_x, \PkNN_x\right) \nu(\dif x)\right] \lesssim \begin{cases}
M^{-\frac{2(2\vee d_\bY)}{d_\bX+d_\bY}}\ln(M), & 2\vee d_\bY \le d_\bX,\\
M^{-1}, & 2\vee d_\bY > d_\bX.
\end{cases}
\end{align*}
\end{theorem}

\subsection{Comments on the convergence rate}\label{subsec:Comments}
This sections gathers several comments on the convergence results we have developed in Section \ref{subsec:rBox} and \ref{subsec:kNN}. 

\subsubsection{On the convergence rate}

We first comment on the expectations in Theorem \ref{thm:ExpectedRaterbox} and \ref{thm:ExpectedRatekNN}.

\textbf{Sharpness of the bounds.} Currently, we cannot establish the sharpness of the convergence rates in Theorems \ref{thm:ExpectedRaterbox} and \ref{thm:ExpectedRatekNN}. We can, however, compare our results to established results in similar settings. For $d_\bY=1$, we may compare it to the optimal rate of non-parametric regression of a Lipschitz continuous function. It is shown in \cite{Stone1982Optimal} that the optimal rate is $M^{-\frac{1}{d_\bX+2}}$, the same as in Theorems \ref{thm:ExpectedRaterbox} and \ref{thm:ExpectedRatekNN} when $d_\bY=1$. For $d_\bY\geq 3$, as noted in \cite{Backhoff2022Estimating}, we may compare to the Wasserstein convergence rate of empirical measure in the estimation of a static distribution on $\mathbb{R}^{d_\bX+d_\bY}$. We refer to \cite{Fournier2015Rate} for the optimal rate, which coincides with those in Theorems \ref{thm:ExpectedRaterbox} and \ref{thm:ExpectedRatekNN}.

\textbf{Error components.} 
We discuss the composition of our upper bound on the expected average error by dissecting the proof of Theorem \ref{thm:ExpectedRaterbox} and \ref{thm:ExpectedRatekNN}. In the proofs, we decompose the expected average errors into two components: approximation error and estimation error. The approximation error occurs when treating $P_{x'}$ as equal to $P_{x}$ when $x'$ is close to the query point $x$, leading to an error of size $L\|x-x'\|_\infty$. The estimation error is associated with the Wasserstein error of empirical measure under i.i.d.~sampling (see \eqref{eq:UBWConvEmp}). From Definitions \ref{def:rbox} and \ref{def:kNN}, the $r$-box estimator effectively manages the approximation error but struggles with controlling the estimation error, whereas the $k$-nearest-neighbor estimator exhibits the opposite behavior.

\textbf{Explicit bounds.} We primarily focus on analyzing the convergence rates of the $r$-box and $k$-nearest-neighbor estimators as $M \to \infty$.  Therefore, within the proofs of these results, we track only the rates (and ignore various constant coefficients). If more explicit bounds are preferred, intermediate results such as \eqref{eq:UBSupEspWrBox}, or \eqref{eq:UBSupEspWkNN} could be good starting points for computing them.

\subsubsection{On the fluctuation}
We next discuss the variances studied in Theorems \ref{thm:Concenrboxnew} and \ref{thm:ConcenkNNnew}. In Appendix \ref{sec:AnotherFluc}, we also include results derived from the Azuma-Hoeffding inequality (e.g., \cite[Corollary 2.20]{Wainwright2019High}), though they provide rougher rates.

\textbf{Condition that $\nu$ is dominated by $\lambda_\bX$.} In Theorems \ref{thm:Concenrboxnew} and \ref{thm:ConcenkNNnew}, we assume that the $\nu$ is dominated by $\lambda_\bX$. This assumption is somewhat necessary. To illustrate, let us examine the non-parametric regression problem under a comparable scenario. We consider a fixed query point. In this context, the central limit theorem for $k$-nearest-neighbor estimator is well-established, and the normalizing rate is $k^{-\frac12}$ (cf. \cite[Theorem 14.2]{Biau2015Lectures}). This suggests that the rate in \eqref{eq:UBVarNoDom} is sharp. For the $r$-box estimator, we believe that a supporting example can be constructed where $\nu$ is highly concentrated. On the other hand,  we conjecture that if $\xi \sim \nu$, the variance could potentially attain the order of $M^{-1}$. For a pertinent result, we direct the reader to \cite[Theorem 1.7]{Backhoff2022Estimating}.

\textbf{Sharpness of the bounds.} Regarding the variance in Theorem \ref{thm:Concenrboxnew}, it is upper bounded by the commonly observed order of $M^{-1}$. We believe that this rate is sharp, though we do not have a proof at this time. As for Theorem \ref{thm:ConcenkNNnew},  the variance is subject to a rougher rate when $2\vee d_\bY \le d_\bX$. We, however, conjecture that this variance attains the order of $M^{-1}$ as long as $\nu$ is dominated by $\lambda_\bX$.

\subsection{Towards implementation with neural networks}\label{subsec:ToNN}

In light of recent practices in machine learning, during the learning of $P$, we may combine the $r$-box method or $k$-nearest-neighbor method into the training of certain parameterized model. To this end we let 
\begin{align*}
\PNN : \bT \times \bX &\to \cP(\bY) \\
(\theta,x) &\mapsto \PNNt_x
\end{align*}
be a parameterized model (e.g., a neural network), where $\bT$ is the parameter space and $\theta \in \bT$ is the parameter to be optimized over. Given an integer $N \ge 1$, we may train $\PNNt$ on a set of query points $\cQ=(\tilde X_n)_{n=1}^N$ satisfying the assumption below.
\begin{hyp}\label{hyp:QueryPoints}
The query points $\cQ=\{(\tilde X_n)\}_{n=1}^N$ are i.i.d.~with uniform distribution over $\bX$, and are independent of the data points $\cD=\{(X_m,Y_m)\}_{m=1}^M$. 
\end{hyp}

We propose the training objectives below
\begin{gather}\label{eq:DefObj}
\argmin_{\theta\in\bT} \frac1N \sum_{n=1}^N \cW\left(\Prbox_{\tilde X_n}, \PNNt_{\tilde X_n}\right) \quad \text{or} \quad
\argmin_{\theta\in\bT} \frac1N \sum_{n=1}^N \cW\left(\PkNN_{\tilde X_n}, \PNNt_{\tilde X_n}\right),
\end{gather}
that is, minimize the mean of $1$-Wasserstein errors between the parametrized model and the empirical $r$-box (or $k$-nearest-neighbour) approximation of the conditional distribution at the location of the random query points.

The following proposition together with Theorem \ref{thm:ExpectedRaterbox} or Theorem \ref{thm:ExpectedRatekNN} justifies using the objectives in \eqref{eq:DefObj}. It is valid for any estimator for $P$ that satisfies the bounds in \eqref{eq:ExpectedRaterbox} or \eqref{eq:ExpectedRatekNN}. Moreover, due to Lipschitz continuity conditions in the proposition, the proposition provides insights into the worst-case performance guarantee. We also refer to \cite{Altekruger2023Conditional} for a worst-case performance guarantee for conditional generative models, which is contingent upon Lipschitz continuity. In contrast, similar guarantees for the $r$-box and $k$-nearest-neighbor estimators are more elusive due to their inherently piece-wise constant nature. We refer to Section \ref{subsec:Pf:prop:PThetanew} for the proof. 

\begin{proposition}\label{prop:PThetanew}
Suppose Assumptions \ref{hyp: kernel lip}, \ref{hyp: data}, and \ref{hyp:QueryPoints} hold. Let $\Pgen$ of $P$ be an estimator constructed using the data points $\cD$ only. Consider a training procedure that produces a (random) $\Theta=\Theta(\cD,\cQ)$ satisfying 
\begin{gather}
\sup_{x,x'\in\bX} \frac{\cW(\PNNT_{x},\PNNT_{x'})}{\|x-x'\|_\infty} \le L^\Theta \label{eq:UBTrainLip}
\end{gather} 
for some (random) $L^\Theta>0$. Then,
\begin{equation}\label{eq:UBExpnIntWPNNT}
\begin{split}
    \esp{\int_\bX \cW(P_x,\PNNT_x) \ud x} \le\;& \esp{(L+L^\Theta)\cW\left(\lambda_\bX,\frac1N\sum_{n=1}^{N}\delta_{\tilde X_n}\right)} 
    \\
    &+ \esp{\int_\bX \cW(P_x, \Pgen_x) \dif x} + \esp{ \frac1N\sum_{n=1}^N \cW\left(\Pgen_{\tilde X_ n}, \PNNT_{\tilde X_n}\right) }.        
\end{split}
\end{equation}
Moreover, with probability $1$,
\begin{align}\label{eq:UBSupW}
\sup_{x\in\bX} \cW\left(P_{x},\tilde P^\Theta_x\right)  \le (d_\bX+1)^{\frac{1}{d_\bX+1}}  (L+L^\Theta)^{\frac{d_\bX}{d_\bX+1}} \left(\int_\bX \cW(P_x, \PNNT_x) \dif x\right)^{\frac{1}{d_\bX+1}}.
\end{align}
\end{proposition}

\begin{remark}\label{rmk:ExpnSupWPNNT}
Assuming $L^\Theta\le\overline L$ for some (deterministic) $\overline L>0$, by \eqref{eq:UBSupW} and Jensen's inequality, we have
\begin{align*}
\esp{\sup_{x\in\bX} \cW\left(P_{x},\tilde P^\Theta_x\right) }  \le  (d_\bX+1)^{\frac{1}{d_\bX+1}}  (L+\overline L)^{\frac{d_\bX}{d_\bX+1}} \esp{ \int_\bX \cW(P_x, \PNNT_x) \dif x }^{\frac{1}{d_\bX+1}}.
\end{align*}
This together with \eqref{eq:UBExpnIntWPNNT} provides a worst-case performance guarantee for $\PNNT$.
\end{remark}

\begin{remark}
Proposition \ref{prop:PThetanew} along with Remark \ref{rmk:ExpnSupWPNNT} provides insights
into the worst-case performance guarantees, but more analysis is needed. Specifically, understanding the magnitude of $L^\Theta$ and $\esp{ \frac1N\sum_{n=1}^N \cW(\Pgen_{\tilde X_ n}, \allowbreak \PNNT_{\tilde X_n}) }$ requires deeper knowledge of the training processes for $\PNNT$, which are currently not well understood in the extant literature. Alternatively, in the hypothetical case where $\PNNT = P$, $L^\Theta$ would match $L$ as specified in Assumption \ref{hyp: kernel lip}, and $\esp{ \frac1N\sum_{n=1}^N \cW\left(\Pgen_{\tilde X_ n}, \PNNT_{\tilde X_n}\right) }$ would obey Theorem \ref{thm:ExpectedRaterbox} or \ref{thm:ExpectedRatekNN}. However, practical applications must also consider the universal approximation capability of $\PNNt$. Further discussion on this topic can be found in \cite{Kratsios2023Universal,Acciaio2024Designing}, although, to the best of our knowledge, recent universal approximation theorems in this subject do not yet concern continuity constraints.
\end{remark}

\section{Implementation with neural networks}\label{sec:ImplNN}
Let $\bX$ and $\bY$ be equipped with $\|\cdot\|_1$. Following the discussion in Section \ref{subsec:ToNN}, we let $\PNNt:\bX\to\cP(\bY)$ be parameterized by a neural network and develop an algorithm that trains $\PNNt$ based on $k$-nearest-neighbor estimator. The $k$-nearest-neighbor estimator $\PkNN$ is preferred as $\PkNN_x$ consistently outputs $k$ atoms. This regularity greatly facilities implementation. For instance, it enables the use of 3D tensors during Sinkhorn iterations to enhance execution speed (see Section \ref{subsubsec:W} later). We  refer also to the sparsity part of Section \ref{subsec:Sinkhorn} for another component that necessitates the aforementioned regularity of $\PkNN$. These components would not be feasible with the $r$-box estimator $\Prbox$, as $\Prbox_x$ produces an undetermined number of atoms. Furthermore, there is a concern that in some realizations, $\Prbox_x$ at certain $x$ may contain too few data points, potentially leading $\PNNT_x$ to exhibit unrealistic concentration.

We next provide some motivation for this implementation. For clarity, we refer to the $r$-box estimator and the $k$-nearest-neighbor estimator as raw estimators. Additionally, we refer to $\PNNT$, once trained, as the neural estimator. While raw estimators are adequate for estimating $P$ on their own, they are piece-wise constant in $x$ by design. On the other hand, a neural estimator is continuous in $x$. This continuity provides a performance guarantee in $\sup\cW$ distance, as outlined in Proposition \ref{prop:PThetanew} and the following remark.  Moreover, the neural estimator inherently possesses gradient information. As discussed in the introduction, this feature renders the neural estimators useful in downstream contexts where gradient information is important, e.g., when performing model-based reinforcement learning. 

We construct $\PNNt$ such that it maps $x\in\bX$ to atoms in $\bY$ with equal probabilities. For the related universal approximation theorems, we refer to \cite{Kratsios2023Universal,Acciaio2024Designing}. We represent these atoms with a vector with $N_{\text{atom}}$ entries denoted by $y^\theta(x)=(y^\theta_{1}(x),\dots,y^\theta_{N_{\text{atom}}}(x)) \in \bY^{N_{\text{atom}}}$, where $N_{\text{atom}}\in\bN$ is chosen by the user. In our implementation, we set $N_{\text{atom}}=k$. To be precise, we construct $\PNNt$ such that
\begin{align}\label{eq:LagrangianDisc}
\PNNt_x = \frac{1}{N_{\text{atom}}}\sum_{j=1}^{N_{\text{atom}}} \delta_{y^\theta_j(x)},\quad x\in\bN.
\end{align}
This is known as the Lagrangian discretization (see \cite[Section 9]{Peyre2019book}). In Algorithm \ref{algo:TrainNetkNN}, we present a high level description of our implementation of training $\PNNt$ based on the raw $k$-nearest-neighbor estimator. 

\begin{algorithm}
\caption{Deep learning conditional distribution in conjunction with $k$-NN estimator}
\footnotesize
\begin{algorithmic}[1]\label{algo:TrainNetkNN}
\renewcommand{\algorithmicrequire}{\textbf{Input:}}
\renewcommand{\algorithmicensure}{\textbf{Output:}}
\REQUIRE data $\set{(X_m,Y_m)}_{m=1}^{M}$ valued in $\bR^{d_\bX}\times\bR^{d_\bY}$, neural estimator $\PNNt$ represented by $y^\theta(x)$ as elaborated in \eqref{eq:LagrangianDisc}, parameters such as $k, N_{\text{atoms}}, N_{\text{batch}}\in\bN_+$, and learning rate $\eta_\theta$
\ENSURE  trained parameter $\Theta$ for the neural estimator
\REPEAT
\FOR{$n=1,\dots,N_{\text{batch}}$}
\STATE generate a query point $\tilde X_n \sim \Uniform(\bX)$
\STATE find the $k$ nearest neighbors of $\tilde X_n$ from data $(X_m)_{m=1}^{M}$ and collect accordingly $(\tilde Y_{n,i})_{i=1}^k$
\ENDFOR 
\STATE compute with Sinkhorn algorithm ($\bY$ is equipped with $\|\cdot\|_1$)
{\scriptsize \begin{align}\label{eq:ImplW}
L[\theta]:=
\sum_{n=1}^{N_{\text{batch}}} \cW\left( \frac1k \sum_{i=1}^k \delta_{\tilde Y_{n,i}}, \frac{1}{N_{\text{atom}}}\sum_{j=1}^{N_{\text{atom}}} \delta_{y^\theta_j(\tilde X_n)} \right)
\end{align}}
\STATE update $\theta\leftarrow \theta - \eta_\theta\,\nabla_\theta L[\theta]$
\UNTIL{Convergence}
\RETURN $\Theta=\theta$ 
\end{algorithmic} 
\end{algorithm}

\subsection{Overview of key components}\label{subsec:ImplOverview}
In this section, we  outline the three key components of our implementation. Each of these components addresses a specific issue: 
\begin{itemize}[label=$\circ$]
    \item Managing the computational cost arising from the nearest neighbors search.
    \item Implementing gradient descent after computing $\cW$.
    \item Selecting an appropriate Lipschitz constant for the neural estimator, preferably at a local level.
\end{itemize}
Further details and ablation analysis on these three components can be found in Section \ref{sec:ImplD}.

\subsubsection{Approximate Nearest Neighbors Search with Random Binary Space Partitioning (ANNS-RBSP)}\label{subsubsec:ANNS}
Given a query point, performing an exact search for its $k$-nearest-neighbor requires $O(M)$ operations. While a single search is not overly demanding, executing multiple searches as outlined in Algorithm \ref{algo:TrainNetkNN} can result in significant computational time, even when leveraging GPU-accelerated parallel computing. To address this, we use ANNS-RBSP as a more cost-effective alternative. Prior to searching, we sort $(X_m)_{m=1}^M$ along each axis and record the order of indices. During the search, the data is divided into smaller subsets by repeatedly applying bisection on these sorted indices, with a random bisecting ratio, on a randomly chosen axis. Furthermore, we apply a restriction that mandates bisection along the longest edge of a rectangle when the edge ratio exceeds certain  value (a hyper-parameter of the model). We record the bounding rectangle for each subset created through this partitioning process. Once partitioning is complete, we generate a small batch of query points within each rectangle and identify the $k$ nearest neighbors for each query point within that same rectangle. For a visual representation of ANNS-BSP, we refer to Figure \ref{fig:rbsp}. Leveraging the sorted indices, we can reapply this partitioning method during every training episode without much computational cost. We refer to Section \ref{subsec:ANNSComparison} for additional details. There are similar ideas in the extant literature (cf. \cite{Hajebi2011Fast,Ram2019Revisiting,Li2020Approximate}).  Given the substantial differences in our setting, however, we conduct further empirical analysis in Section \ref{subsec:ANNSComparison} to showcase the advantage of our approach against exact search. 

\begin{figure}
\centering
\includegraphics[width=0.4\textwidth]{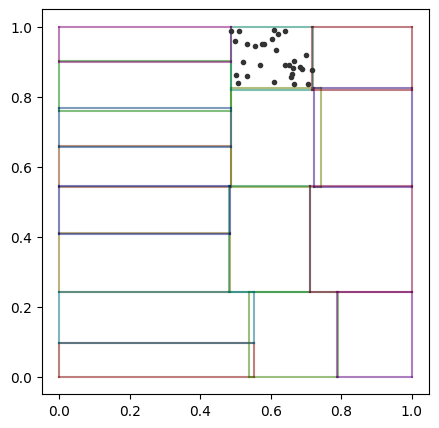}
\caption{An instance of RBSP in $[0,1]^2$.}
\label{fig:rbsp}
\medskip
\small
The 2D unit box is partitioned into 16 rectangles based on $500$ samples from $\Uniform([0,1])$. Note that the overlap between the bounding rectangles is intentionally maintained. Each partitioning is performed along an axis selected at random, dividing the samples within the pre-partitioned rectangle according to a random ratio drawn from $\Uniform([0.45,0.55])$. The edge ratio for mandatory bisecting along the longest edge is $5$. If this ratio is exceeded, partitioning along the longest edge is enforced. The black dots represent samples within the respective rectangle. 
\end{figure}

\subsubsection{Computing $\cW$ for gradient descent} \label{subsubsec:W}
The following discussion pertains to the computation of \eqref{eq:ImplW}, with the subsequent gradient descent in consideration. For simplicity, let us focus on the summand and reduce the problem to the following minimization.  Let $(\tilde y_1,\dots,\tilde y_k)\in\bY^k$ be fixed, we aim to find
\begin{align}\label{eq:ImplWReduced}
\argmin_{y\in\bY^n}\cW\left( \frac1k \sum_{i=1}^k \delta_{\tilde y_i},  \frac1n \sum_{j=1}^n \delta_{y_j} \right).
\end{align}
The criterion in \eqref{eq:ImplWReduced} is convex as $\cW$ is convex in both arguments (cf. \cite[Theorem 4.8]{Villani2008book}). To solve \eqref{eq:ImplWReduced}, as is standard, we cast it into a discrete optimal transport problem. To do so, first introduce the  $(k\times n)$-cost matrix $\mC_y$, where $\mC_{y, i j}:=\|\tilde y_i - y_j\|_1$. As the criterion in \eqref{eq:ImplWReduced} has uniform weights on the atoms, we next aim to solve the problem
\begin{gather}\label{eq:DOptTrans}
\argmin_{\mT\in[0,1]^{k\times n}} \left\{ \varphi_y(\mT) := \sum_{(i,j)\in\set{1,\dots,k}\times\set{1,\dots,n}} \mT_{ij} \mC_{y,ij} \right\} \\
\quad\text{subject to}\quad  \sum_{j=1}^n \mT_{ij} = \frac1k,\; i=1,\dots,k \quad \text{and} \quad \sum_{i=1}^k \mT_{ij} = \frac1n,\; j=1,\dots,n. \nonumber
\end{gather}
Let $\mT^*_y$  be an optimal transport plan that solves \eqref{eq:DOptTrans} for $y$ fixed. Taking derivative of $y\mapsto\varphi_y(\cdot)$ yields 
\begin{align}\label{eq:DphiDy}
\left.\partial_{y_j}\varphi_y(\mT)\right|_{\mT=\mT^*_y} = \sum_{i\in\set{1,\dots,k}} \mT^*_{y,ij}\; \partial_{y_j} \|\tilde y_i - y_j\|_1, \quad j=1,\dots,n.
\end{align}
This gradient is in general not the gradient corresponding to \eqref{eq:ImplWReduced}, as $\mT^*_y$ depends on $y$, while \eqref{eq:DphiDy} excludes such dependence. Nevertheless, it is still viable to update $y$ using the gradient descent that employs the partial gradient specified in \eqref{eq:DphiDy}. To justify this update rule, first consider $y'\in\bY$ satisfying $\varphi_{y'}(\mT^*_y)\le\varphi_y(\mT^*_y)$, then observe that 
\begin{align*}
\cW\left( \frac1k \sum_{i=1}^k \delta_{\tilde y_i},  \frac1n \sum_{j=1}^n \delta_{y'_j} \right) \le \varphi_{y'}(\mT^*_y) \le \varphi_y(\mT^*_y) = \cW\left( \frac1k \sum_{i=1}^k \delta_{\tilde y_i},  \frac1n \sum_{j=1}^n \delta_{y_j} \right).
\end{align*}
This inequality is strict if $\varphi_{y'}(\mT^*_y)<\varphi_y(\mT^*_y)$. We refer to \cite[Section 9.1]{Peyre2019book} and the reference therein for related discussions.

The Sinkhorn algorithm, which adds an entropy regularization, is a widely-used algorithm for approximating the solution to \eqref{eq:DOptTrans}. Specifically, here, it is an iterative scheme that approximately solves the following regularized problem, subject to the constraints in \eqref{eq:DOptTrans},
\begin{align}\label{eq:RegDOptTrans}
\argmin_{\mT^\epsilon\in[0,1]^{k\times n}} \left\{ \sum_{i,j\in\set{1,\dots,k}\times\set{i,\dots,n}} \mT^\epsilon_{ij} \mC_{ij} + \epsilon \sum_{i,j\in\set{1,\dots,k}\times\set{i,\dots,n}} \mT^\epsilon_{ij} (\log\mT_{ij} - 1) \right\},
\end{align}
where $\epsilon>0$ is a hyper-parameter, and should not be confused with the $\varepsilon$ used elsewhere. We refer to Section \ref{subsec:Sinkhorn} for further details. We also refer to \cite[Section 4]{Peyre2019book} and the reference therein for convergence analysis of the Sinkhorn algorithm. It is well known that the regularization term in \eqref{eq:RegDOptTrans} is related to the entropy of a discrete random variable. Larger values of $\epsilon$ encourages the regularized optimal transport plan to be more diffusive. That is, for larger values of $\epsilon$, the mass from each $y_j$ is distributed more evenly across all $\tilde y_i$'s. Performing gradient descent along the direction in \eqref{eq:DphiDy} tends to pull $y_j$'s towards the median of the $\tilde y_i$'s, as we are equipping $\bY$ with the norm $\|\cdot\|_1$. Conversely, small values of $\epsilon$ often leads to instability, resulting in \textsf{NaN} loss/gradient. To help with these issues, we implement the Sinkhorm algorithm after normalizing the cost matrix. Additionally, we use a large $\epsilon$ (e.g., $1$) in the first few training episodes, then switch to a smaller $\epsilon$ (e.g., $0.1$) in later episodes. Furthermore, we impose sparsity on the transport plan by manually setting the smaller entries of the transport plan to $0$. The specific detailed configurations and related ablation analysis are provided in Section \ref{subsec:Sinkhorn} and Appendix \ref{sec:Config}.

\subsubsection{Network structure that induces locally adaptive Lipschitz continuity} \label{subsubsec:Net}
As previously discussed, it is desirable for the neural estimator to exhibit certain Lipschitz continuity. In practice, however, determining an appropriate Lipschitz constant for training the neural estiamtor $\PNNt$ is challenging, largely because understanding the true Lipschitz continuity of $P$ (if it  exists) is very challenging. Additionally, the estimate provided in Proposition \ref{prop:PThetanew} is probabilistic. Fortunately, a specific network structure allows the neural estimator, when properly trained, to exhibit locally adaptive Lipschitz continuity. Subsequently, we provide a high-level overview of this network structure. Further detailed configurations and ablation analysis are presented in Section \ref{subsec:LipNet} and Appendix \ref{sec:Config}.

Consider a fully connected feed-forward neural network with equal width hidden layers and layer-wise residual connection \cite{He2016Deep}. Let $N_{\text{neuron}}$ denote the width of the hidden layers. For activation, we use Exponential Linear Unit (ELU) function \cite{Clevert2016Fast}, denoted by $\sigma$. For hidden layers, we employ the convex potential layer introduced in \cite{Meunier2022Dynamical}, 
\begin{align}\label{eq:CPLayer}
\mathsf{x}_{\text{out}} = \mathsf{x}_{\text{in}} - \|\mW\|_2^{-1} \mW^{\mT} \sigma\left( \mW \mathsf{x}_{\text{in}} + \mathsf{b} \right).
\end{align}
By \cite[Proposition 3]{Meunier2022Dynamical}, the convex potential layer is $1$-Lipschitz continuous in $\|\cdot\|_2$ sense. For the input layer, with a slight abuse of notation, we use
\begin{align}\label{eq:InputLayer}
\mathsf{x}_{\text{out}} = N_{\text{neuron}}^{-1}\text{diag}(|\mW|^{-1}_1\wedge 1)  \sigma\left( \mW \mathsf{x}_{\text{in}} + \mathsf{b} \right),
\end{align}
where $|\mW|_1$ computes the absolute sum of each row of the weight matrix to form a vector of size $N_{\text{neuron}}$, the reciprocal and $\cdot\wedge 1$ are applied entry-wise, and $\diag$ produces a diagonal square matrix based on the input vector. In short, the normalization in \eqref{eq:InputLayer} is only applied to the rows of $\mW$ with $\ell_1$-norm exceeding $1$. Consequently, the input layer is $1$-Lipschitz continuous in $\|\cdot\|_1$ sense. A similar treatment is used for the output layer but without activation,
\begin{align}\label{eq:OutputLayer}
\mathsf{x}_{\text{out}} = L\;d_\bY^{-1} \text{diag}( |\mW|^{-1}_1\wedge 1)  \left( \mW \mathsf{x}_{\text{in}} + \mathsf{b} \right).
\end{align}
where $L>0$ is a hyper-parameter. The output represents atoms on $\bY$ with uniform weight, therefore, no $N^{-1}_{\text{atom}}$ is required here.

The spectral norm $\|\mW\|_2$ in \eqref{eq:CPLayer}, however, does not, in general, have an explicit expression. Following the implementation in \cite{Meunier2022Dynamical}, we approximate each $\|\mW\|_2$ with power iteration. Power iterations are applied to all hidden layers simultaneously during training. To control the pace of iterations, we combine them with momentum-based updating. We refer to Algorithm \ref{algo:PwrIter} for the detailed implementation. Our implementation differs from that in \cite{Meunier2022Dynamical}, as the authors of \cite{Meunier2022Dynamical} control the frequency of  updates but not the momentum. In a similar manner, for input and output layers, instead of calculating the row-wise $\ell_1$-norm explicitly, we update them with the same momentum used in the hidden layers. 
Our numerical experiments consistently show that a small momentum value of $\tau=10^{-3}$ effectively maintains adaptive continuity while maintaining a satisfactory accuracy. The impact of $L$ in \eqref{eq:OutputLayer} and $\tau$ in Algorithm \ref{algo:PwrIter} is  discussed in Section \ref{subsec:LipNet}.

\begin{algorithm}
\caption{Power iteration with momentum for updating $\|\mW\|_2$ estimate, applied to all convex potential layers simultaneously at every epoch during training}
\begin{algorithmic}[1]\label{algo:PwrIter}
\renewcommand{\algorithmicrequire}{\textbf{Input:}}
\renewcommand{\algorithmicensure}{\textbf{Output:}}
\footnotesize
\REQUIRE weight matrix $\mW\in\bR^{d\times d}$ of a convex potential layer, previous estimate $\hat h\in\bR$ and auxiliary vector $\hat{\mathsf{u}}\in\bR^{d}$, momentum $\tau\in(0,1)$
\ENSURE  updated $\hat h$ and $\hat{\mathsf{u}}$, \tikzmark{right}in particular, $\hat h$ will be used as a substitute of $\|\mW\|_2$ in \eqref{eq:CPLayer}
\STATE $\mathsf{v} \leftarrow \mW \hat{\mathsf{u}} / \|\mW \hat{\mathsf{u}}\|_2 $ \tikzmark{top}
\STATE $\mathsf{u} \leftarrow \mW^\mT \mathsf{v} / \|\mW^\mT \mathsf{v}\|_2$
\STATE $h \leftarrow 2 / \left(\sum_{i} (\mW \mathsf{u} \cdot \mathsf{v})_i\right)^2$ \tikzmark{bottom}
\STATE $\hat h \leftarrow \tau \hat h + (1-\tau) h$ \tikzmark{top2}
\STATE $\hat{\mathsf{u}} \leftarrow \tau \hat{\mathsf{u}} + (1-\tau) \mathsf{u}$ \tikzmark{bottom2}
\RETURN $\hat h$, $\hat{\mathsf u}$ 
\end{algorithmic} 
\AddNote{top}{bottom}{right}{power iteration}
\AddNote{top2}{bottom2}{right}{momentum-based updating}
\end{algorithm}

During training, due to the nature of our updating schemes, the normalizing constants do not achieve the values required for the layers to be $1$-Lipschitz continuous. We hypothesize that this phenomenon leads to a balance that ultimately contributes to adaptive continuity: on one hand, the weights $\mW$ stretch to fit (or overfit) the data, while on the other, normalization through iterative methods prevents the network from excessive oscillation. As shown in Section \ref{subsubsec:L} and \ref{subsubsec:tau}, the $L$ value in \eqref{eq:OutputLayer} and the momentum $\tau$ in Algorithm \ref{algo:PwrIter} affect the performance significantly. For completeness, we also experiment with replacing \eqref{eq:CPLayer} by fully connected feedforward layers similar to \eqref{eq:InputLayer}, with or without batch normalization \cite{Ioffe2015Batch} after affine transformation. This alternative, however, failed to produce satisfactory results.

\subsection{Experiments with synthetic data}\label{subsec:Experiments}
We consider data simulated from three different models. The first two have $d_\bX=d_\bY=1$, while the third has $d_\bX=d_\bY=3$. Here we no longer restrict $\bY$ to be the unit box, however, we still consider $\bX$ to be a $d_\bX$-dimensional unit box (not necessarily centered at the origin). 

In Model 1 and 2, $X\sim\Uniform([0,1])$. Model 1 is a mixture of two independent Gaussian random variables with mean and variance depending on $x$,
\begin{align*}
Y = \xi \Big( 0.1\, \big(1+\cos(2\pi X)\big) + 0.12\,\big|1-\cos(2\pi X)\big| Z + 0.5 \Big),
\end{align*}
where $Z\sim\Normal(0,1)$ and $\xi$ is a Rademacher random variable independent of $Z$. For Model 2, we have 
\begin{align*}
Y = 0.5\,\1_{[0,1)}(X) + 0.5\, U,
\end{align*}
where $U\sim\Uniform([0,1])$. The conditional distribution in Model 2 is intentionally designed to be discontinuous in the feature space. This choice was made to evaluate performance in the absence of the Lipschitz continuity stipulated in Assumption \ref{hyp: kernel lip}. Model 3 is also a mixture of two independent Gaussian random variables, constructed by considering $X\sim\Uniform([-\frac12,\frac12]^3)$ and treating $X$ as a column vector (i.e., $X$ take values in $\bR^{3\times1}$), 
\begin{align*}
Y = \zeta \Big(\cos(\mA X) + 0.1 \cos(\mathsf{\Sigma}_X) W\Big) + (1-\zeta)\Big(\cos(\mA' X) + 0.1 \cos(\mathsf{\Sigma}'_X) W'\Big).
\end{align*}
Above, the $\cos$ functions act on vector/matrix entrywise, $\mA\in\bR^{3\times 3}$, and $\mathsf{\Sigma}_x$ also takes value in $\bR^{3\times 3}$. Each element of $\mathsf{\Sigma}_x$ is defined as $\mathsf v_{ij} x$ for some $\mathsf v_{ij}\in\bR^{1\times 3}$. The entries of $\mA$ and $\mathsf v_{ij}$ are drawn from standard normal in advance and remain fixed throughout the experiment. The matrices $\mA'$ and $\mathsf{\Sigma}'_x$ are similarly constructed. Furthermore, $W$ and $W'$ are independent three-dimensional standard normal r.v.s, while $\zeta$ represents the toss of a fair coin, independent of $X$, $W$, and $W'$.

For the purpose of comparison, two different network structures are examined. The first, termed LipNet, is illustrated in Section \ref{subsec:ImplOverview}. The second, termed StdNet, is a fully connected feedforward network with layer-wise residual connections \cite{He2016Deep}, ReLU activation, and batch normalization immediately following each affine transformation, without specifically targeting Lipschitz continuity. With a hyper-parameter $k$ for the $k$-nearest-neighbor estimator, which we specify later, each network contains 5 hidden layers with $2k$ neurons. These networks are trained using the Adam optimizer \cite{Kingma2017Adam} with a learning rate of $10^{-3}$.
For StdNet in Model $1$ and $2$, the learning rate is set to $0.01$, as it leads to better performance. Other than the learning rates, StdNet and LipNet are trained with identical hyper-parameters across all models. We refer to Appendix \ref{sec:Config} for a summary of hyper-parameters involved.

We generate $10^4$ samples for Models 1 and 2. Given the convergence rate specified in Theorem \ref{thm:ExpectedRatekNN}, we note that the sample size are considered relatively small. For these two models, we chose $k=100$ and utilized neural networks $\PNNt$ that output atoms of size $N_{\text{atom}}=k$. The choice of $k$ is determined by a rule of thumb. In particular, our considerations include the magnitude of $k$ suggested by Theorem \ref{thm:ExpectedRatekNN} and the computational costs associated with the Sinkhorn iterations discussed in Section \ref{subsubsec:W}. The results under Model 1 and 2 are plotted in Figure \ref{fig:1D}, \ref{fig:1DCondCDF} and \ref{fig:1DAllW}. Figure \ref{fig:1D} provides a perspective on joint distributions, while Figure \ref{fig:1DCondCDF} and \ref{fig:1DAllW} focus on conditional CDFs across different $x$ values.

\begin{figure}[!htbp]
\centering
\begin{subfigure}[t]{0.24\linewidth}
\centering
\makebox[1pt]{\raisebox{50pt}{\rotatebox[origin=c]{90}{Data}}}
\includegraphics[width=0.90\textwidth]{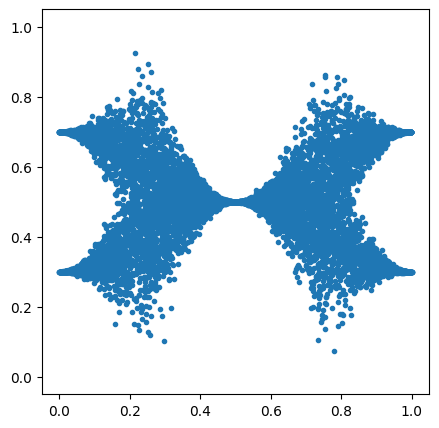}\\
\makebox[1pt]{\raisebox{50pt}{\rotatebox[origin=c]{90}{All atoms scattered}}}
\includegraphics[width=0.90\textwidth]{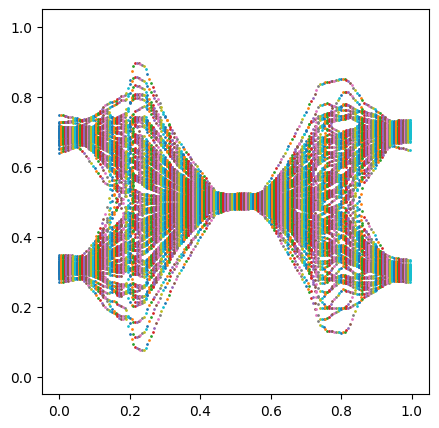}\\
\makebox[1pt]{\raisebox{50pt}{\rotatebox[origin=c]{90}{Traj.\! of 20 atoms}}}
\includegraphics[width=0.90\textwidth]{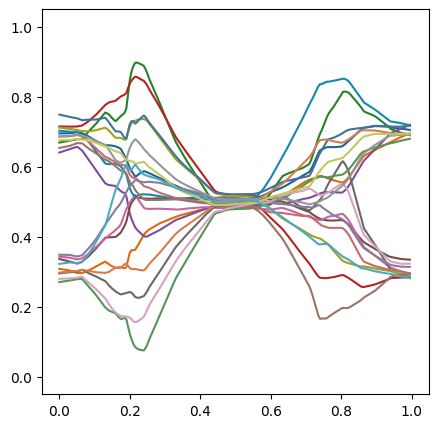}\\
\makebox[1pt]{\raisebox{50pt}{\rotatebox[origin=c]{90}{Ave.\! abs.\! der.}}}
\includegraphics[width=0.90\textwidth]{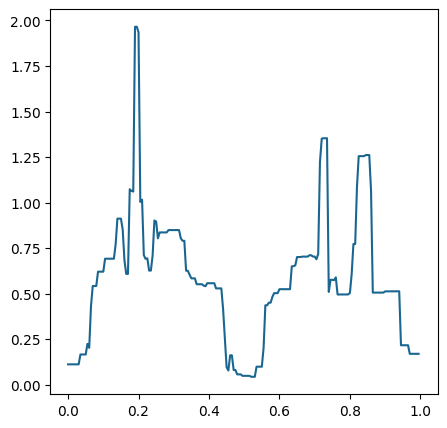}
\caption{Model 1, StdNet}
\end{subfigure}
\begin{subfigure}[t]{0.24\linewidth}
\centering
\includegraphics[width=0.90\textwidth]{plot/data1}
\includegraphics[width=0.90\textwidth]{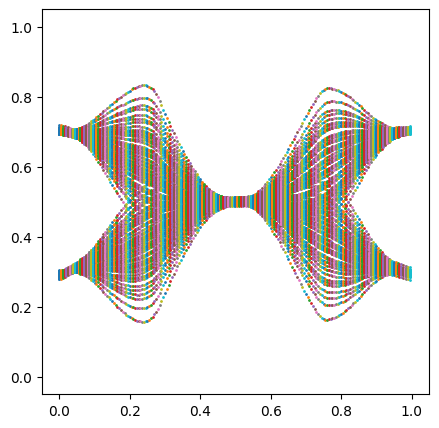}
\includegraphics[width=0.90\textwidth]{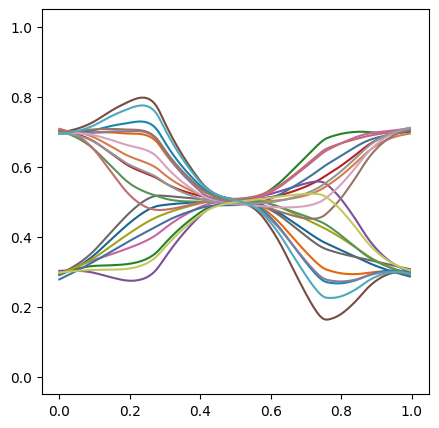}
\includegraphics[width=0.90\textwidth]{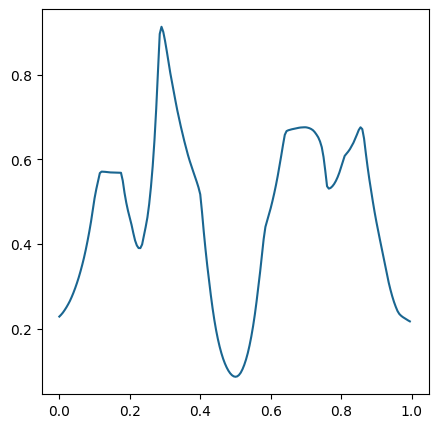}
\caption{Model 1, LipNet}
\end{subfigure}
\begin{subfigure}[t]{0.24\linewidth}
\centering
\includegraphics[width=0.90\textwidth]{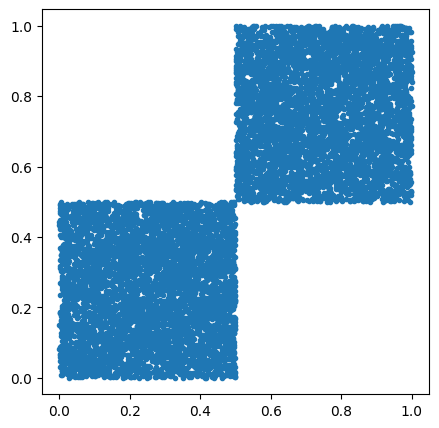}
\includegraphics[width=0.90\textwidth]{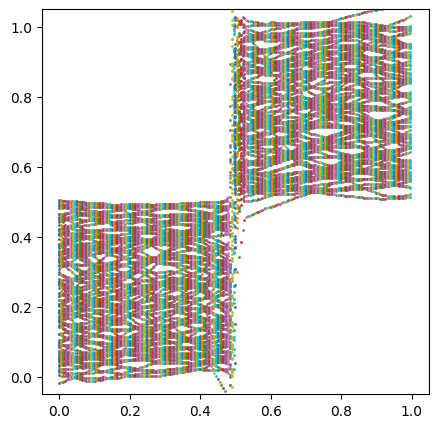}
\includegraphics[width=0.90\textwidth]{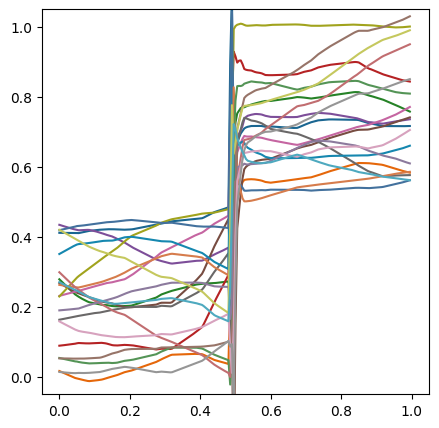}
\includegraphics[width=0.90\textwidth]{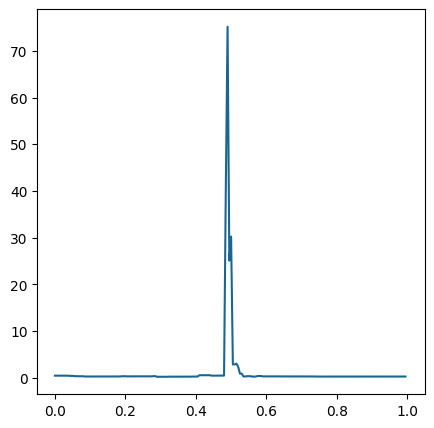}
\caption{Model 2, StdNet}
\end{subfigure}
\begin{subfigure}[t]{0.24\linewidth}
\centering
\includegraphics[width=0.90\textwidth]{plot/data2}
\includegraphics[width=0.90\textwidth]{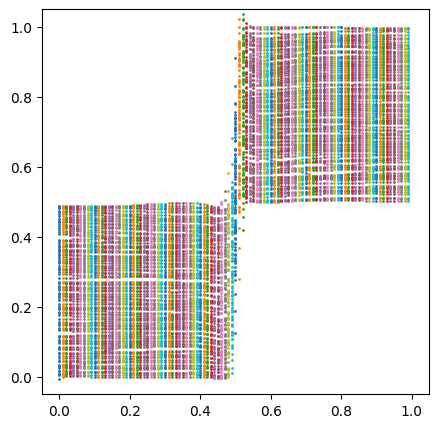}
\includegraphics[width=0.90\textwidth]{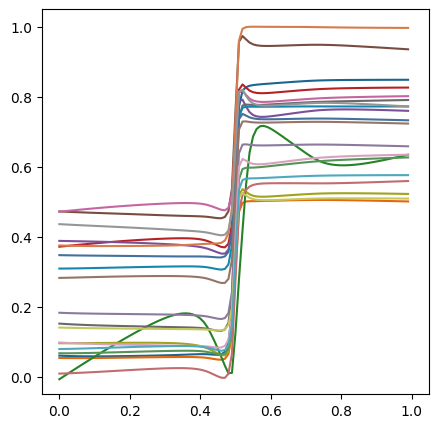}
\includegraphics[width=0.90\textwidth]{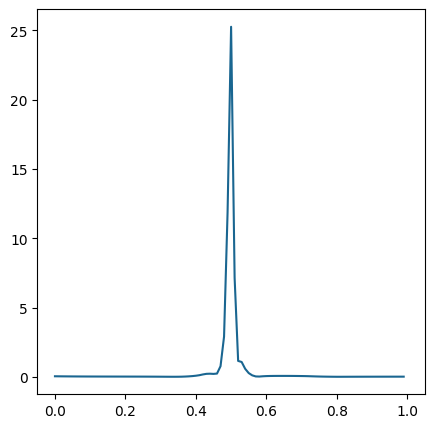}
\caption{Model 2, LipNet}
\end{subfigure}
\caption{Various estimators under Model 1 and 2, joint distributions.}
\label{fig:1D}
\medskip
\small
The first row presents the data. Neural networks with different structures are trained on the same set of data for comparison. The second row shows scatter plots of atoms at various $x$ values. The third row illustrates the evolution of 20 atoms as $x$ varies. The final row presents the average of the derivative of each atom with respect to $x$, with a notable difference in the $y$ axis scale. 
\end{figure}

\begin{figure}[!htbp]
\centering
\begin{subfigure}[t]{1\linewidth}
\centering
\makebox[1pt]{\raisebox{191pt}{\rotatebox[origin=c]{90}{Model 1}}}
\makebox[1pt]{\raisebox{80pt}{\rotatebox[origin=c]{90}{Model 2}}}
\includegraphics[width=0.90\textwidth]{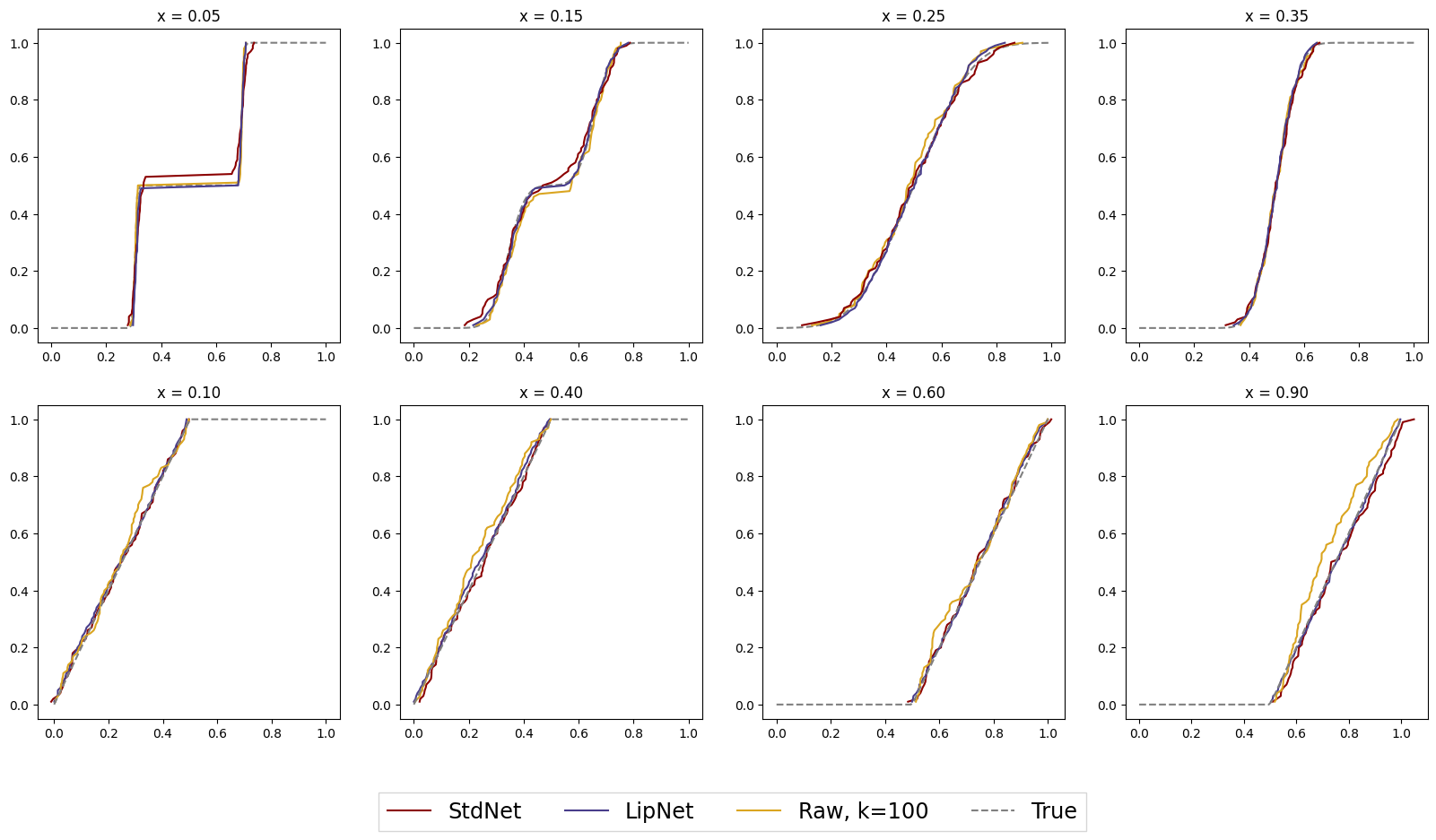}
\end{subfigure}
\caption{Various estimators under Model 1 and 2, conditional CDFs.}
\label{fig:1DCondCDF}
\medskip
\small
We compare the conditional CDFs at various values of $x$, derived from StdNet, LipNet, the raw $k$-nearest estimator with $k=100$ (also used in the training of StdNet and LipNet), and the ground truth. The first row pertains to data set 1. The second row pertains to data set 2. Subfigure titles display the values of $x$.
\end{figure}

\begin{figure}[!htbp]
\centering
\begin{subfigure}[t]{1\linewidth}
\centering
\makebox[1pt]{\raisebox{160pt}{\rotatebox[origin=c]{90}{Model 1}}}
\makebox[1pt]{\raisebox{65pt}{\rotatebox[origin=c]{90}{Model 2}}}
\includegraphics[width=0.80\textwidth]{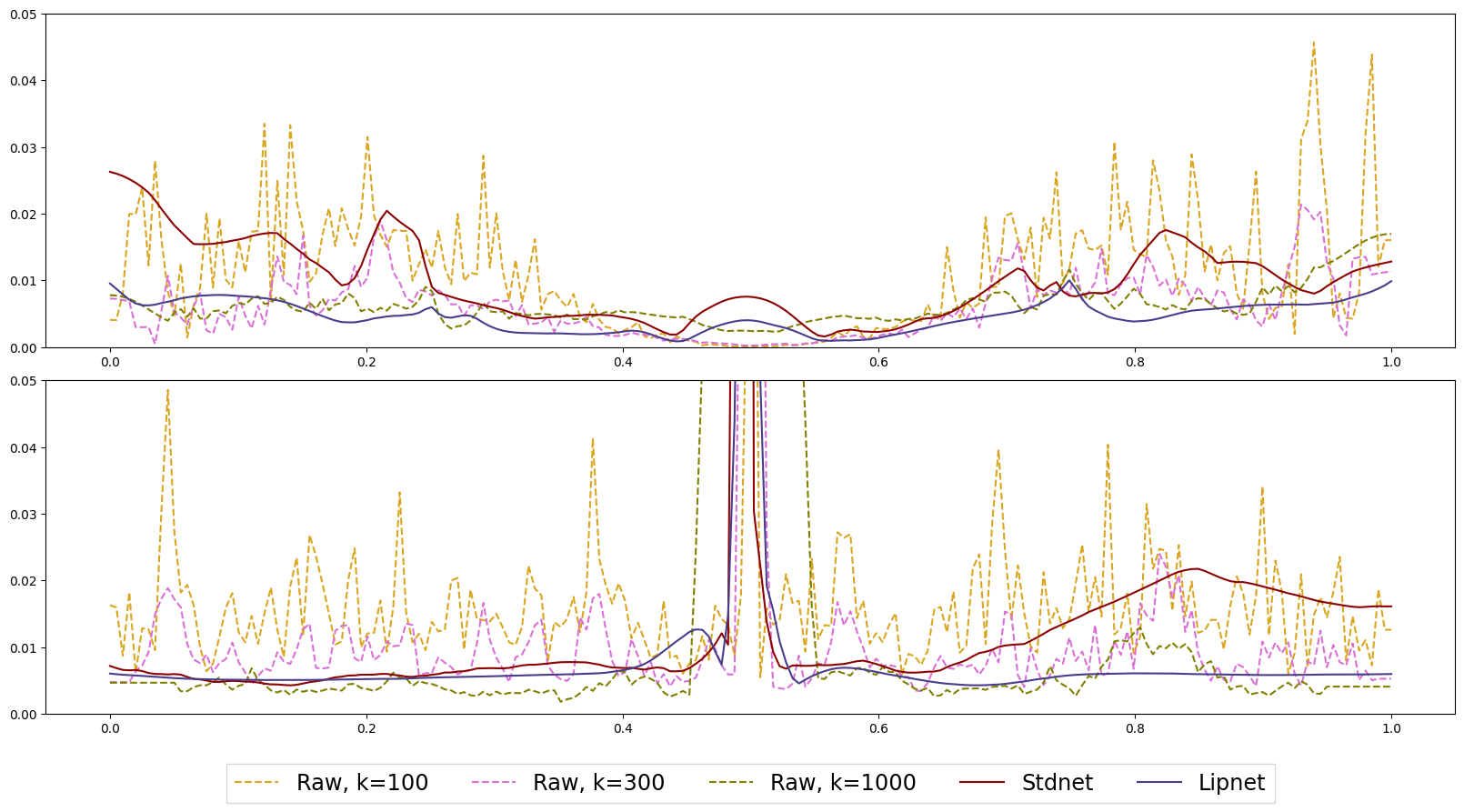}
\end{subfigure}
\caption{Errors at different $x$'s of various estimators under Model 1 and 2.}
\label{fig:1DAllW}
\medskip
\small
We compute the $\cW$-distance between estimators and the true conditional distribution at different $x$'s. StdNet and LipNet are trained based on the raw $k$-nearest-neighbor estimator with $k=100$.
\end{figure}

Figure \ref{fig:1D} suggets that both StdNet and LipNet adequately recover the joint distribution. The LipNet's accuracy is, however, notably superior and produces smooth movements of atoms (as seen in the third row of Figure \ref{fig:1D}).  Although further fine-tuning may provide slight improvements in StdNet's performance, StdNet will still not achieve the level of accuracy and smoothness observed in LipNet. The average absolute value of derivative of each atom (fourth row of Figure \ref{fig:1D}), makes it evident that LipNet demonstrates a capacity of automatically adapting to a suitable level of Lipschitz continuity locally. In particular, in Model 2, the atoms of LipNet respond promptly to jumps while remaining relatively stationary around values of $x$ where the kernel is constant. We emphasize that LipNet is trained using the same hyper-parameters across Models 1, 2, and 3. 

Figure \ref{fig:1DCondCDF} shows the estimated conditional distribution at different values of $x$. Figure \ref{fig:1DCondCDF} indicates that the raw $k$-nearest-neighbor estimator deviates frequently from the actual CDFs. This deviation of the raw $k$-nearest-neighbor estimator is expected, as it attempts to estimate an unknown CDF with only $k=100$ samples given an $x$. Conversely, the neural estimator, especially the LipNet, appears to offer extra corrections even if they are trained based on the raw $k$-nearest-neighbor estimator. This could be attributed to neural estimators implicitly leveraging information beyond the immediate neighborhood. 

Figure \ref{fig:1DAllW} compares the $\cW$-distance between each estimator and the true conditional distribution at various values of $x$, using the following formula (see \cite[Remark 2.28]{Peyre2019book}),
\begin{align}\label{eq:WCDFs}
\cW(F, G) = \int_\bR \big|F(r)-G(r)\big|\dif r,
\end{align}
where $F$ and $G$ are CDFs. This quantity can be accurately approximated with trapezoidal rule. In Model 1, the neural estimator generally outperforms the raw estimator with $k=100$ across most values of $x$, even though the raw estimator is used for training the neural estimators. Furthermore, LipNet continues to outperform raw estimators with larger values of $k$ -- even though LipNet is trained with a raw estimator with $k=100$. In Model 2, LipNet continues to demonstrate a superior performance, except when compared to the raw estimator with $k=1,000$ at $x$ distant from $0.5$, the reason is that, here, the conditional distribution is piece-wise constant in $x$, which enhances the performance of the raw estimator at larger $k$ values. 

The aforementioned findings indicate superior performance by LipNet. We, however, recognize that improvements are not always guaranteed, as demonstrated in Figures \ref{fig:1DCondCDF} and \ref{fig:1DAllW}.

\begin{figure}[!htbp]
\centering
\begin{subfigure}[t]{1\linewidth}
\centering
\includegraphics[width=0.90\textwidth]{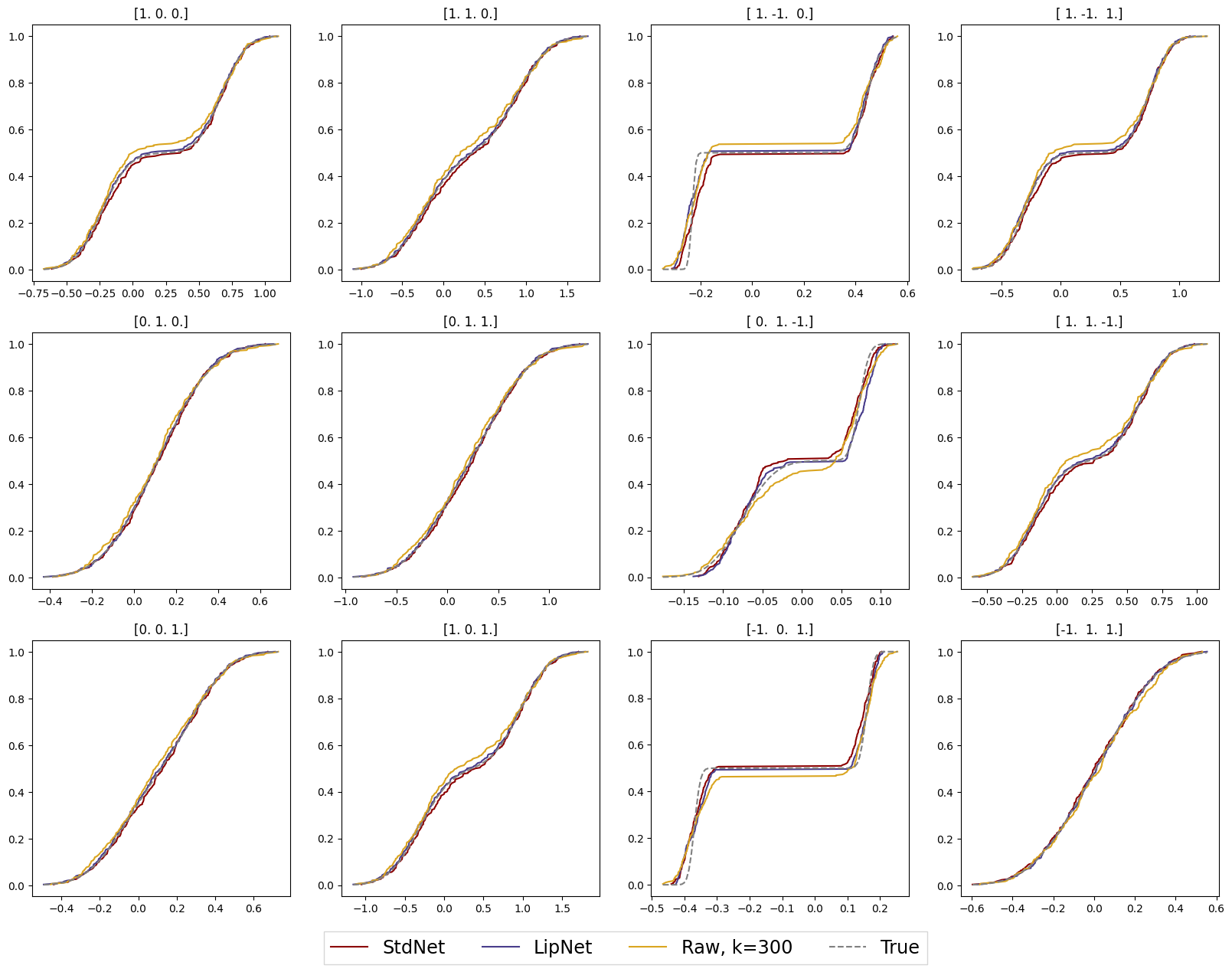}
\end{subfigure}

\caption{Various estimators under Model 3, projections of conditional CDFs.}
\label{fig:3DCondCDF}
\medskip
\small
We compare the projected conditional CDFs at $x=(0.12, -0.33,  0.1)$. The estimations are obtained from StdNet, LipNet, the raw estimator with $k=300$ (also used in training StdNet and LipNet), and the ground truth. Subfigure titles display the vectors used for projection. Note the difference in the $x$ axis scale. 
\end{figure}

\begin{figure}[htbp]
\centering
\begin{subfigure}[t]{0.24\linewidth}
\centering
\includegraphics[width=0.95\textwidth]{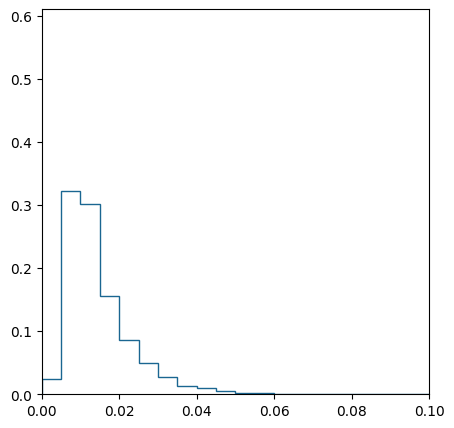}
\caption*{Raw $k=1,000$}
\end{subfigure}
\begin{subfigure}[t]{0.24\linewidth}
\centering
\includegraphics[width=0.95\textwidth]{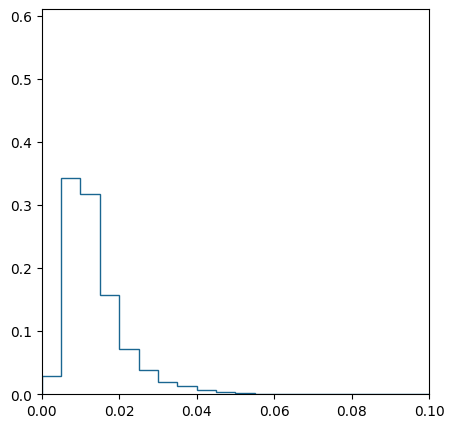}
\caption*{Raw $k=3,000$}
\end{subfigure}
\begin{subfigure}[t]{0.24\linewidth}
\centering
\includegraphics[width=0.95\textwidth]{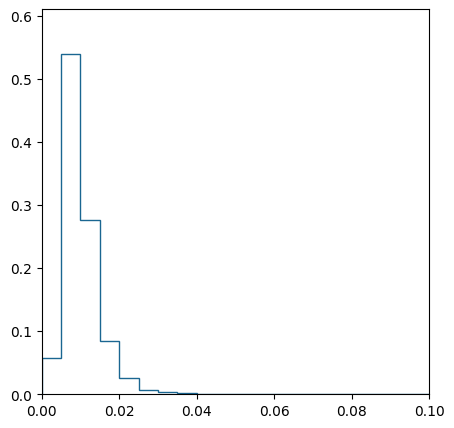}
\caption*{StdNet}
\end{subfigure}
\begin{subfigure}[t]{0.24\linewidth}
\centering
\includegraphics[width=0.95\textwidth]{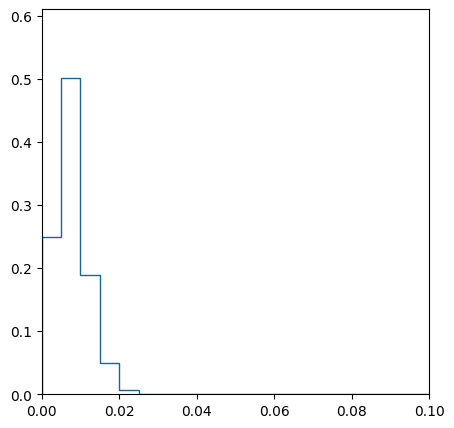}
\caption*{LipNet}
\end{subfigure}
\caption{Histogram of $10,000$ projected Wasserstein-$1$ errors.}
\label{fig:HistProjW}
Each histogram consists of $20$ uniformly positioned bins between $0$ to $0.1$. The errors of different estimators are computed with the same set of query points and projection vectors. Errors larger than $0.1$ will be placed in the right-most bins. Note StdNet and LipNet are trained with $k=300$.
\end{figure}

For Model 3, we generate $10^6$ samples and select $k=300$. We train both neural estimators using Adam optimizer with a learning rate of $10^{-3}$. Hyperparameters such as $L$ in \eqref{eq:OutputLayer} and $\tau$ in Algorithm \ref{algo:PwrIter} are consistent with those used for Models 1 and 2. We refer to Appendix \ref{sec:Config} for the detailed configuration. 

In Figure \ref{fig:3DCondCDF}, we visualize the outcomes in Model 3: the conditional CDFs at an arbitrarily chosen $x$ are projected onto various vectors. We observe that the neural estimators considerably outperform the raw $k$-nearest-neighbor estimator, likely owing due to their implicit use of global information outside of the immediate neighbors during training. For further comparisons, we present additional figures in Appendix \ref{sec:MorePlots}: Figures \ref{fig:3DCondCDFk1e3}, \ref{fig:3DCondCDFk3e3} and \ref{fig:3DCondCDFk1e4} feature the exact same neural estimators as shown in Figure \ref{fig:3DCondCDF}, but with the raw $k$-nearest-neighbor estimators employing different $k$ values, $k=1,000,\, 3,000,\, 10,000$. 
Raw $k$-nearest-neighbor estimators with $k=1,000,\, 3,000$ are superior to that with $k=300$, while at $k=10,000$, the accuracy begins to decline. Upon comparison, the neural estimator trained with $k=300$ consistently outperforms the raw $k$-nearest-neighbor estimators for all values of $k$. 

\begin{figure}[!htbp]
\centering
\begin{subfigure}[t]{1\linewidth}
\centering
\includegraphics[width=0.90\textwidth]{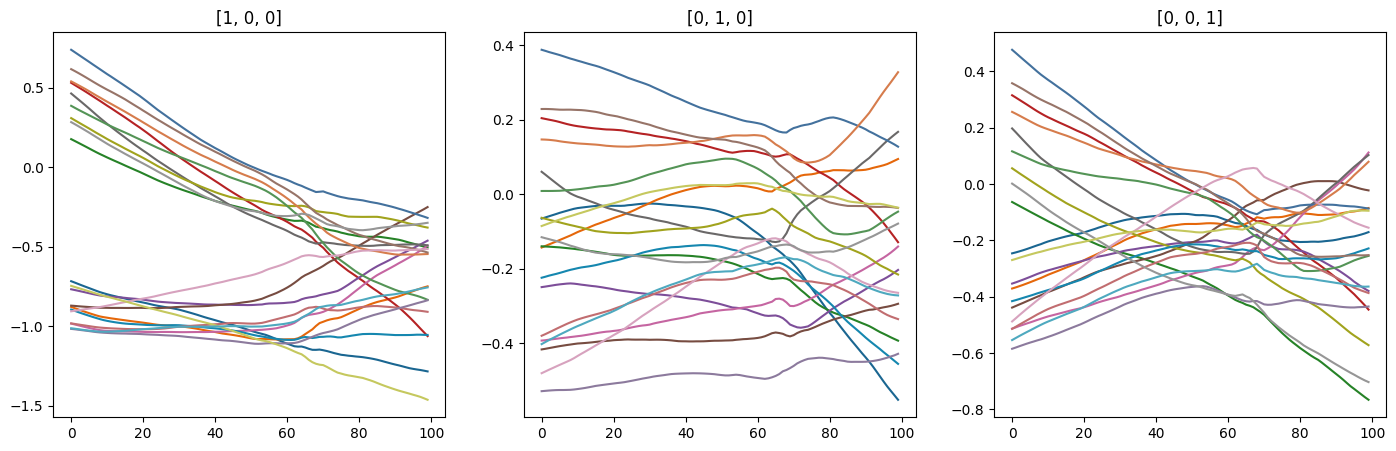}
\includegraphics[width=0.90\textwidth]{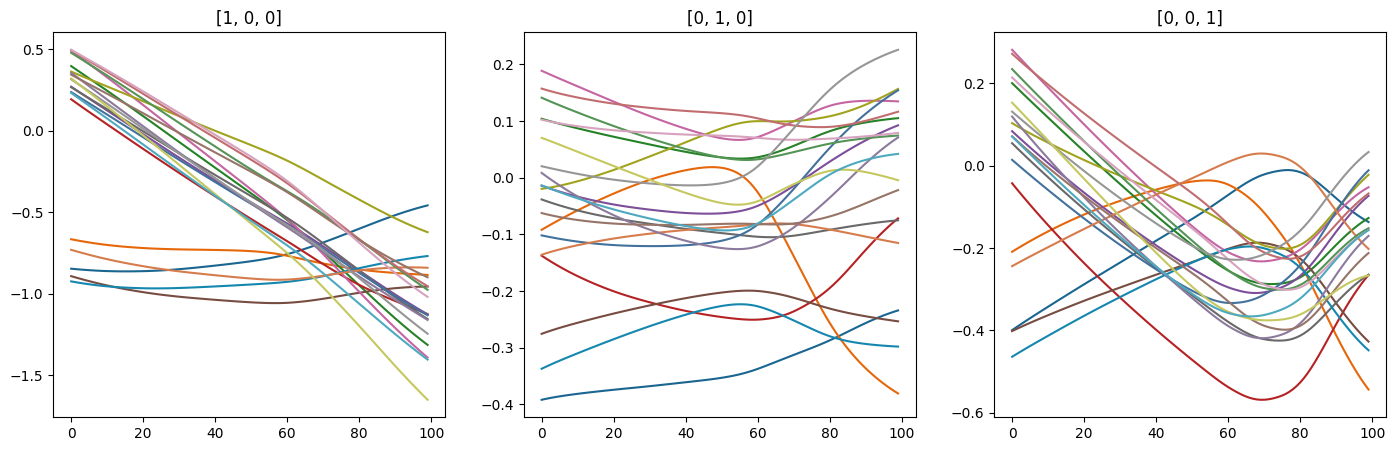}
\end{subfigure}
\caption{LipNet under Model 3, projected trajectories of $20$ atoms.}
\label{fig:3DAtom}
\medskip
\small
We illustrate the projected trajectories of $20$ atoms by evaluating LipNet at $100$ evenly allocated points along the straight line that intersects the origin and $x=(0.12, -0.33,  0.1)$, situated within $[0,1]^3$. The $x$-axis denotes the specific points along the line, consistent across all subfigures. 
Subfigure titles display the vectors used for projection.  Note the difference in the $y$ axis scale.
\end{figure}

For a more comprehensive comparison, we randomly select $10,000$ query points. For each query point, we randomly generate a vector in $\bR^3$, normalized under $\|\cdot\|_1$, and project the atoms produced by the estimators onto said vector. With the same vector, we also compute the corresponding true CDFs of the projected $Y$ given the query point. We then approximately compute the $\cW$-distance between the projected distributions via \eqref{eq:WCDFs}. The resulting histograms are shown in Figure \ref{fig:HistProjW}, which suggests that LipNet performs best.  The rationale for employing this projection approach, rather than directly computing the $\cW$-distance between discrete and continuous distributions over $\bR^3$, is due to the higher cost and lower accuracy of the latter approach (see also the discussion in Section \ref{subsec:Sinkhorn}). While this projection approach provides a cost-effective alternative for performance evaluation, it may not fully capture the differences between the estimations and ground truth.

Lastly, to demonstrate how atoms, in the neural estimator, move  as $x$ varies, Figure \ref{fig:3DAtom} shows the projected trajectories along a randomly selected straight line through the origin. The movement of atoms in LipNet is smooth, consistent with previous observations. Interestingly, the movement of atoms in StdNet isn't excessively oscillatory either, although its continuity is slightly rougher compared to LipNet. The reader may execute the \texttt{Jupyter} notebook on our github repisitory \url{https://github.com/zcheng-a/LCD_kNN}  to explore the projected conditional CDFs and atoms' trajectories for different $x$ values.

\section{Proofs}\label{sec:Proofs}

\subsection{Auxiliary notations and lemmas}\label{subsec:Aux}
In this section, we will introduce a few technical results that will be used in the subsequent proofs. 

We first define
\begin{align}\label{eq:DefR}
R(m) := \sup_{x\in\bX}\int_{\bY^m}\cW\left(P_x, \frac1m\sum_{\ell=1}^m\delta_{y_\ell}\right) \bigotimes_{\ell=1}^m P_x(\dif y_\ell), \quad m\in\bN.
\end{align}
We stipulate that $R(0)=1$. By \cite{Fournier2015Rate}, we have
\begin{align}\label{eq:UBWConvEmp}
R(m) \le\; \stackrel\frown{C} \times \begin{cases}
m^{-\frac{1}2}, & {d_\bY} = 1,\\
m^{-\frac{1}{2}}\ln(m), & {d_\bY} = 2,\\
m^{-\frac{1}{{d_\bY}}}, & {d_\bY} \ge 3,
\end{cases}
\end{align}
for some constant $\stackrel\frown{C}\;>0$ depending only on $d_\bY$. For comprehension, we also point to \cite{Kloeckner2020Empirical,Fournier2023Convergence} for results that are potentially useful in analyzing explicit constant, though it is out of the scope of this paper.

The lemma below pertains to the so-called approximation error, which arises when treating data points $Y_j$ with $X_j$ around an query point as though they are generated from the conditional distribution at the query point.
\begin{lem}\label{lem:AbsDiffIntegrals}
Under Assumption \ref{hyp: kernel lip}, for any integer $J\ge1$ and $x,x_1,\dots,x_J\in \bX^{J+1}$, we have 
\begin{align*}
\left| \int_{\bY^J} \cW\left(\frac{1}{J}\sum_{j=1}^J \delta_{y_j}, P_x\right) \bigotimes_{j=1}^J P_{x_j}(\ud y_j) - \int_{\bY^J} \cW\left(\frac{1}{J}\sum_{j=1}^J \delta_{y_j}, P_x\right) \bigotimes_{j=1}^J P_x(\ud y_j) \right| \le \frac{L}{J} \sum_{j=1}^J \| x_j- x\|_\infty.
\end{align*}
\end{lem}

\begin{proof}
For $x,x_1,\dots,x_J\in \bX^{J+1}$, note that
\begin{align*}
&\left| \int_{\bY^J} \cW\left(\frac{1}{J}\sum_{j=1}^J \delta_{y_j}, P_x\right) \bigotimes_{j=1}^J P_{x_j}(\ud y_j) - \int_{\bY^J} \cW\left(\frac{1}{J}\sum_{k=1}^J \delta_{y_j}, P_x\right) \bigotimes_{j=1}^J P_{x}(\ud y_j) \right| \\
&\quad\le \sum_{\ell=1}^{J} \left| \int_{\bY^J} \cW\left(\frac{1}{J}\sum_{j=1}^J \delta_{y_j}, P_x\right)  \bigotimes_{j=1}^{\ell-1} P_{x}(\ud y_j) \otimes \bigotimes_{j=\ell}^J P_{x_j}(\ud y_j) - \int_{\bY^J} \cW\left(\frac{1}{J}\sum_{j=1}^J \delta_{y_j}, P_x\right) \bigotimes_{j=1}^{\ell} P_{x}(\ud y_j)\otimes\bigotimes_{j=\ell+1}^J P_{x_j}(\ud y_j) \right|,
\end{align*}
where for the sake of neatness, at $\ell=1,J$, we set 
\begin{align*}
\bigotimes_{j=1}^{0} P_{x}(\ud y_j)\otimes\bigotimes_{j=1}^J P_{x_j}(\ud y_j) = \bigotimes_{j=1}^J P_{x_j}(\ud y_j) \quad\text{and}\quad \bigotimes_{j=1}^{J} P_{x}(\ud y_j)\otimes\bigotimes_{j=J+1}^J P_{x_j}(\ud y_j) = \bigotimes_{j=1}^{J} P_{x_j}(\ud y_j).
\end{align*}
Regarding the $\ell$-th summand, invoking Fubini-Toneli theorem to integrate $y_\ell$ first then combining the integrals on outer layers using linearity, we obtain
\begin{align*}
&\left| \int_{\bY^J} \cW\left(\frac{1}{J}\sum_{j=1}^J \delta_{y_j}, P_x\right)  \bigotimes_{j=1}^{\ell-1} P_{x}(\ud y_j)\otimes\bigotimes_{j=\ell}^J P_{x_j}(\ud y_j) - \int_{\bY^J} \cW\left(\frac{1}{J}\sum_{j=1}^J \delta_{y_j}, P_x\right) \bigotimes_{j=1}^{\ell}P_{j}(\ud y_j)\otimes\bigotimes_{j=\ell+1}^J P_{x_j}(\ud y_j) \right|\\
&\quad= \left| \int_{\bY^{J-1}} \int_{\bY} \cW\left(\frac{1}{J}\sum_{j=1}^J \delta_{y_j}, P_x\right) \left(P_x - P_{x_\ell}\right)(\ud y_\ell) \bigotimes_{j=1}^{\ell-1}\ud P_{x}(y_j)\otimes\bigotimes_{j=\ell+1}^J P_{x_j}(\ud y_j) \right| \\
&\quad\le \sup_{(y_j)_{j\neq \ell} \in\bY^{J-1}} \left| \int_{\bY}\cW\left(\frac{1}{J}\sum_{j=1}^J \delta_{y_j}, P_x\right) \left(P_{x_\ell} - P_{x}\right)(\ud y_\ell) \right| \le \frac{1}{J} \cW(P_{x_\ell},P_x) \le \frac{L}{J}\|x_\ell-x\|_{\infty},
\end{align*}
where in the second last inequality we have invoked Kantorovich-Rubinstein duality (cf. \cite[Particular case 5.16]{Villani2008book}) and the fact that, for all $(y_j)_{j\neq\ell}\in\bY^{J-1}$, the map $y_\ell\mapsto \cW\left(\frac{1}{J}\sum_{j=1}^J \delta_{y_j}, P_x\right)$ is $\frac1J$-Lipschitz, and where in the last equality, we have used Assumption \ref{hyp: kernel lip}.
\end{proof}

We will be using the lemma below, which regards the stochastic dominance between two binomial random variables.
\begin{lem}\label{lem:BinStochDom}
Let $n\in\bN$ and $0 \le p < p'\le 1$. Then, $\Binomial(n,p')$ stochastically dominates $\Binomial(n,p)$.
\end{lem}
\begin{proof}
Let $U_1,\dots,U_n\stackrel{\text{i.i.d.}}{\sim}\Uniform[0,1]$ and define 
$$H := \sum_{i=1}^n\1_{[0,p]}(U_i), \quad H' := \sum_{i=1}^n\1_{[0,p']}(U_i).$$ 
Clearly, $H\sim\Binomial(n,p)$ and $H'\sim\Binomial(n,p')$. Moreover, we have $H\le H'$, and thus $\bP(H > r) \le \bP(H' > r)$, which completes the proof.
\end{proof}

\subsection{Proof of Theorem \ref{thm:ExpectedRaterbox}}\label{subsec:Pf:thm:ExpectedRaterbox}

The proof of Theorem \ref{thm:ExpectedRaterbox} relies the technical lemma below that we state and prove now.
\begin{lem}\label{lem:sum p}
Let $p \in [0,1]$ a real number, and let $M \ge 1$ and $d \ge 1$ two integers. We then have
\begin{align*}
\sum_{m=1}^M \binom{M}{m} p^m (1-p)^{M-m} m^{-\frac{1}{d}} \le \left((M+1)p\right)^{-\frac{1}{d}} + \left((M+1)p\right)^{-1}.
\end{align*}
\end{lem}
\begin{proof}
We compute
\begin{align*}
&\sum_{m=1}^M \binom{M}{m} p^m (1-p)^{M-m} m^{-\frac{1}{d}} = \frac{1}{(M+1)p} \sum_{m=1}^M \binom{M+1}{m+1} p^{m+1}(1-p)^{M-m} (m+1) m^{-\frac{1}{d}} \\
&\quad= \frac{1}{(M+1)p} \sum_{m=2}^{M+1} \binom{M+1}{m} p^m (1-p)^{M+1-m} m (m-1)^{-\frac{1}{d}} \\
&\quad= \frac{1}{(M+1)p} \sum_{m=2}^{M+1} \binom{M+1}{m} p^m(1-p)^{M+1-m} (m-1)^{1-\frac{1}{d}} \\ 
&\qquad+ \frac{1}{(M+1)p} \sum_{m=2}^M \binom{M+1}{m} p^m(1-p)^{M+1-m}(m-1)^{-\frac{1}{d}},
\end{align*}
where we used that $m=m-1+1$ in the last equality. Then, using that $(m-1)^{1-\frac 1 d} \le m^{1-\frac 1 d}$ and $(m-1)^{-\frac 1 d} \le 1$ for all $m \ge 2$, we continue to obtain
\begin{align*}
\sum_{m=1}^M \binom{M}{m} p^m (1-p)^{M-m} m^{-\frac{1}{d}} \le &\frac{1}{(M+1)p} \sum_{m=2}^{M+1} \binom{M+1}{m} p^m(1-p)^{M+1-m} m^{1-\frac{1}{d}} \\ &+ \frac{1}{(M+1)p} \sum_{m=2}^M \binom{M+1}{m} p^m(1-p)^{M+1-m} \\
\le &\frac{1}{(M+1)p} \sum_{m=0}^{M+1} \binom{M+1}{m} p^m(1-p)^{M+1-m} m^{1-\frac{1}{d}} + \frac{1}{(M+1)p},
\end{align*}
where the second term in the last equality are derived from the binomial formula. Finally, introducing a random variable $V$ with binomial distribution $\cB(M+1,p)$, and using Jensen inequality for the concave function $\bR^+ \ni x \mapsto x^{1-\frac{1}{d}} \in \bR^+$, we obtain
\begin{align*}
\sum_{m=1}^M \binom{M}{m} p^m (1-p)^{M-m} m^{-\frac{1}{d}} \le& \frac{1}{(M+1)p} \esp{V^{1-\frac{1}{d}}} +  \frac{1}{(M+1)p} \\ 
\le& \frac{\left((M+1)p\right)^{1-\frac{1}{d}}}{(M+1)p} + \frac{1}{(M+1)p} = \left((M+1)p\right)^{-\frac{1}{d}} + \left((M+1)p\right)^{-1},
\end{align*}
which conclude the proof.
\end{proof}

We are now ready to prove Theorem \ref{thm:ExpectedRaterbox}.

\begin{proof}[Proof of Theorem \ref{thm:ExpectedRaterbox}]

For $\nu \in \cP(\bX)$, we obviously have
\begin{align*}
\esp{\int_\bX \cW(P_x, \Prbox_x) \nu(\ud x)} \le \sup_{x \in \bX} \esp{\cW(P_x, \Prbox_x)},
\end{align*}
we then focus on proving the right hand side inequality in Theorem \ref{thm:ExpectedRaterbox}. To this end, we fix $x \in \bX$ and, to alleviate the notations, we let $B := \cB^r(x)$ as introduced in Definition \ref{def:rbox}. Let $N_B:=\sum_{m=1}^M \1_B(X_m)$. By Definition \ref{def:rbox} and Assumption \ref{hyp: data} (i), we have 
\begin{align}
&\nonumber \esp{ \cW(P_x,\Prbox_x) } = \esp{\cW(P_x,\hat\mu^\cD_B)} = \sum_{m=0}^M \esp{\ind_{ N_B = m }  \cW(P_x, \hat \mu^\cD_B)} \\
&\nonumber \quad= \sum_{m=1}^M \binom{M}{m}  \esp{\ind_{X_1, \dots, X_m \in B} \ind_{X_{m+1}, \dots, X_M \not\in B} \cW\left(\frac{1}{m} \sum_{l=1}^m \delta_{Y_l}, P_x\right)} + \pro{X_1, \dots, X_M \not\in B} \cW(\lambda_\bY, P_x) \\
&\nonumber \quad\le \sum_{m=1}^M \binom{M}{m} \pro{ X_{m+1}, \dots, X_M \not\in B} \esp{\ind_{X_1, \dots, X_m \in B} \cW\left(\frac{1}{m} \sum_{l=1}^m \delta_{Y_l}, P_x\right)} + \xi(B^c)^M R(0) \\
&\label{eq esp} \quad= \sum_{m=1}^M \binom{M}{m} \mu(B^c)^{M-m} \int_{(B\times\bY)^m} \cW\left(\frac{1}{m}\sum_{l=1}^m \delta_{y_l}, P_x\right) \bigotimes_{\ell=1}^m\psi( \ud x_\ell \ud y_\ell) + \xi(B^c)^M R(0).
\end{align}
To compute the integral terms, observe that, for fixed $m \ge 1$, by definition of $R(m)$ in \eqref{eq:DefR}, Lemma \ref{lem:AbsDiffIntegrals} and Remark \ref{rem: rbox radius},
\begin{align}
&\nonumber \int_{(B\times\bY)^m} \cW\left(\frac{1}{m}\sum_{l=1}^m \delta_{y_l}, P_x\right)  \bigotimes_{\ell=1}^m\psi(\ud x_\ell \ud y_\ell) = \int_{B^m} \int_{\bY^m} \cW\left(\frac{1}{m}\sum_{l=1}^m \delta_{y_l}, P_x\right) \bigotimes_{l=1}^m P_{x_l}(\ud y_l) \bigotimes_{\ell=1}^m\xi(\ud x_\ell) \\
\nonumber &\quad\le \int_{B^m} \left( \int_{\bY^m} \cW\left(\frac{1}{m}\sum_{l=1}^m \delta_{y_l}, P_x\right) \bigotimes_{\ell=1}^m P_{x}(\ud y_\ell) + \frac{L}{m}\sum_{\ell=1}^m \|x_\ell-x\|_\infty\right) \bigotimes_{\ell=1}^m\xi(\ud x_\ell)  \\
&\quad\le \int_{B^m} \left(R(m) + 2Lr\right) \bigotimes_{\ell=1}^m\xi(\ud x_\ell) = (R(m)+2Lr)\xi(B)^m.
\end{align}
This together with \eqref{eq esp} implies that, for any $x\in\bX$,
\begin{align}\label{eq:UBSupEspWrBox}
\esp{ \cW(P_x,\Prbox_x) } &\le \sum_{m=1}^M \binom{M}{m} \xi(B^c)^{M-m} \xi(B)^m (R(m) + 2Lr) + \xi(B^c)^M R(0)\nonumber\\
&\le 2Lr + \sum_{m=1}^M \binom{M}{m} \xi(B^c)^{M-m} \xi(B)^m R(m) + \xi(B^c) R(0)
\end{align}

The remainder of the proof is split into three cases. In order to proceed, we will put together \eqref{eq:UBWConvEmp}, Lemma \ref{lem:sum p}, and \eqref{eq:UBSupEspWrBox}. Below we only keep track of the rate.
\begin{itemize}
\item For $d_\bY=1$, we have 
\begin{align*}
\esp{ \cW(P_x,\Prbox_x) } &\le 
2Lr + \left(\xi(B)(M+1)\right)^{-\frac12} + \left(\xi(B)(M+1)\right)^{-1} + (1-\xi(B))^M
\\
&\le 2Lr + \left(\underline c (2r)^{\dX}(M+1)\right)^{-\frac12} + \left(\underline c (2r)^{\dX}(M+1)\right)^{-1} + e^{-\underline c M r^{\dX}}.
\end{align*}
Controlling the dominating term(s) by setting $r \sim r^{-\frac{\dX}{2}}M^{-\frac12}$, we yield 
\begin{align*}
r\sim M^{-\frac1{d_\bX+2}} \quad\text{and}\quad \esp{ \cW(P_x,\Prbox_x) } \lesssim M^{-\frac1{d_\bX+2}}.
\end{align*}
\item For $d_\bY=2$, we have 
\begin{align*}
\esp{ \cW(P_x,\Prbox_x) } &\le 
2Lr + \ln(M)\left(\xi(B)(M+1)\right)^{-\frac12} + \left(\xi(B)(M+1)\right)^{-1} + (1-\xi(B))^M
\\
&\le 2Lr + \ln(M)\left(\underline c (2r)^{\dX}(M+1)\right)^{-\frac12} + \left(\underline c (2r)^{\dX} (M+1)\right)^{-1} + e^{-\underline c M r^{\dX}}.
\end{align*}
Since $r\sim \ln (M) r^{-\frac{\dX}{2}}M^{-\frac12}$ may not have a closed-form solution, we simply follow the case of $d_\bY=1$ to yield
\begin{align*}
r\sim M^{-\frac1{d_\bX+2}} \quad\text{and}\quad \esp{ \cW(P_x,\Prbox_x) } \lesssim M^{-\frac1{d_\bX+2}}\ln M.
\end{align*}
\item For $d_\bY\ge3$, we have 
\begin{align*}
\esp{ \cW(P_x,\Prbox_x) } &\le 
2Lr + \left(\xi(B)(M+1)\right)^{-\frac{1}{\dY}} + \left(\xi(B)(M+1)\right)^{-1}+(1-\xi(B))^M
\\
&\le 2Lr + \left(\underline c (2r)^{\dX}(M+1)\right)^{-\frac{1}{\dY}} + \left(\underline c (2r)^{\dX}(M+1)\right)^{-1}+ e^{-\underline c M r^{\dX}}.
\end{align*}
By setting $r\sim r^{-\frac{\dX}{\dY}}M^{-\frac{1}{\dY}}$, we yield 
\begin{align*}
r\sim M^{-\frac1{d_\bX+d_\bY}} \quad\text{and}\quad \esp{ \cW(P_x,\Prbox_x) } \lesssim M^{-\frac1{d_\bX+d_\bY}}.
\end{align*}
\end{itemize}
The proof is complete.
\end{proof}

\subsection{Proof of Theorem \ref{thm:Concenrboxnew}} \label{subsec:Pf:thm:Concenrboxnew}
\begin{proof}[Proof of Theorem \ref{thm:Concenrboxnew}]
We will proceed by using Efron-Stein inequality. Let $(X_1',Y_1')$ be an independent copy of $(X_1,Y_1)$, and define $\cD':=\set{(X_1',Y_1'), (X_2,Y_2), \dots, (X_M,Y_M)}$. In view of Assumption \ref{hyp: data} (i), by the triangle inequality of $\cW$, it is sufficient to investigate 
\begin{align*}
\frac12M\, \bE\left[ \left( \int_{\bX} \cW\left(\hat \mu^{\cD}_{\cB^r(x)}, \hat \mu^{\cD'}_{\cB^r(x)}\right) \dif \nu(x)\right)^2 \right].
\end{align*}
Notice that, by definitions \eqref{eq: mu hat},
\begin{align}\label{eq:UBEventrBoxnew}
\left\{\hat\mu^{\cD}_{\cB^r(x)}\neq\hat\mu^{\cD'}_{\cB^r(x)}\right\}  
\subseteq \Big\{X_1 \in \cB^r(x)\Big\}\cup\Big\{X_1' \in \cB^r(x)\Big\}.
\end{align}
Additionally, by definitions \eqref{eq: mu hat} again, on the event that $\left\{\hat\mu^{\cD}_{\cB^r(x)}\neq\hat\mu^{\cD'}_{\cB^r(x)}\right\}$, we have 
\begin{align}\label{eq:UBWrBoxnew}
\cW\left(\hat\mu^{\cD}_{\cB^r(x)}, \hat\mu^{\cD'}_{\cB^r(x)}\right) \le \left(1+\sum_{\ell=2}^{M} \1_{\cB^r(x)}(X_\ell) \right)^{-1}.
\end{align}
The above together with the condition that $\nu$ is dominated by $\lambda_\bX$ implies that 
\begin{align}\label{eq:UBExpnSqIntW}
&\bE\left[ \left( \int_{\bX} \cW\left(\hat \mu^{\cD}_{\cB^r(x)}, \hat \mu^{\cD'}_{\cB^r(x)}\right) \nu(\dif x)\right)^2 \right] \le \overline C^2 \bE\left[ \left( \int_{B(X_1,2r)\cup B(X_1',2r)} \cW\left(\hat \mu^{\cD}_{\cB^r(x)}, \hat \mu^{\cD'}_{\cB^r(x)}\right)  \lambda_{\bX}(\dif x)\right)^2 \right] \nonumber\\
&\quad \le \overline C^2 \bE\left[ \left( \int_{B(X_1,2r)\cup B(X_1',2r)} \left(1+\sum_{\ell=2}^{M} \1_{\cB^r(x)}(X_\ell) \right)^{-1} \lambda_{\bX}(\dif x)\right)^2 \right] \nonumber\\
&\quad \le 4 \overline C^2 \bE\left[ \left( \int_{B(X_1,2r)} \left(1+\sum_{\ell=2}^{M} \1_{\cB^r(x)}(X_\ell) \right)^{-1} \lambda_{\bX}(\dif x)\right)^2  \right] \nonumber\\
&\quad= 4 \overline C^2 \bE\left[ \lambda_{\bX}(B(X_1,2r))^2 \left( \int_{B(X_1,2r)} \left(1+\sum_{\ell=2}^{M} \1_{\cB^r(x)}(X_\ell) \right)^{-1} \frac{\lambda_{\bX}(\dif x)}{\lambda_{\bX}(B(X_1,2r))} \right)^2 \right] \nonumber\\
&\quad\le 4 \overline C^2 (4r)^{2\dX} \bE\left[ \bE\left[ \int_{B(X_1,2r)} \left(1+\sum_{\ell=2}^{M} \1_{\cB^r(x)}(X_\ell) \right)^{-2} \frac{\lambda_{\bX}(\dif x)}{\lambda_{\bX}(B(X_1,2r))} \Bigg| X_1 \right] \right], 
\end{align}
where we have used Jensen's inequality and tower property in the last line. In view of Assumption \ref{hyp: data} (i), expanding the inner conditional expectation into an integral with respect to regular conditional distribution (cf. \cite[Section 10]{Bogachev2007book}) then invoking Fubini-Tonelli theorem, we yield
\begin{align}\label{eq:ReprCondEIntRecip}
&\bE\left[ \int_{B(X_1,2r)} \left(1+\sum_{\ell=2}^{M} \1_{\cB^r(x)}(X_\ell) \right)^{-2} \frac{\lambda_{\bX}(\dif x)}{\lambda_{\bX}(B(X_1,2r))} \Bigg| X_1 \right] \nonumber\\
&\quad= \int_{B(X_1,2r)} \int_{\bX^{M-1}} \left(1+\sum_{\ell=2}^{M} \1_{\cB^r(x)}(x_\ell) \right)^{-2} \bigotimes_{\ell=2}^M \xi(\dif x_\ell)  \frac{\lambda_{\bX}(\dif x)}{\lambda_{\bX}(B(X_1,2r))}.
\end{align}
For the inner integral in \eqref{eq:ReprCondEIntRecip}, by Assumption \ref{hyp: data} (ii), we have 
\begin{align*}
&\int_{\bX^{M-1}} \left(1+\sum_{\ell=m+1}^M \1_{\cB^r(x)}(x_\ell)\right)^{-2}
\bigotimes_{\ell=2}^{M} \xi(\dif x_\ell) \\
&\quad= \sum_{\ell=0}^{M-1} \binom{M-1}{\ell} \xi\big(\cB^r(x)\big)^\ell \Big(1-\xi^r\big(\cB^r(x)\big)\Big)^{M-1-\ell} \left(1+\ell\right)^{-2} \\
&\quad= \frac{1}{M(M+1)\xi\big(\cB^r(x)\big)^2} \sum_{\ell=0}^{M-1} \binom{M+1}{\ell+2} \xi\big(\cB^r(x)\big)^{\ell+2} \Big(1-\xi\big(\cB^r(x)\big)\Big)^{M-1-\ell} \frac{\ell+2}{\ell+1} \\
&\quad\le \frac{2}{M(M+1)\xi\big(\cB^r(x)\big)^2} \sum_{\ell=2}^{M+1} \binom{M+1}{\ell} \xi\big(\cB^r(x)\big)^{\ell} \Big(1-\xi\big(\cB^r(x)\big)\Big)^{M+1-\ell} \\
&\quad \le \frac{2}{M(M+1)\xi\big(\cB^r(x)\big)^2}.
\end{align*}
This together with \eqref{eq:UBExpnSqIntW}, \eqref{eq:ReprCondEIntRecip} and Assumption \ref{hyp: data} (ii) implies
\begin{align*}
\bE\left[ \left( \int_{\bX} \cW\left(\hat \mu^{\cD}_{\cB^r(x)}, \hat \mu^{\cD'}_{\cB^r(x)}\right) \dif \nu(x)\right)^2 \right] \le 8\frac{2^{2\dX}\overline C^2}{\underline c^2 M(M+1) }.
\end{align*} 
Invoking Efron-Stein inequality, we conclude the proof.
\end{proof}

\subsection{Proof of Theorem \ref{thm:ExpectedRatekNN}}\label{subsec:Pf:thm:ExpectedRatekNN}
In order to prove  Theorem \ref{thm:ExpectedRatekNN}, we first establish a few technical lemmas. The following lemma is a first step toward finding the average rate of $k$-nearest neigbhor method. 
\begin{lem}\label{lem:EspWkNN}
Suppose Assumption \ref{hyp: kernel lip} and  \ref{hyp: data}. Let $R$ be defined in Section \ref{subsec:Aux}. Then, for any $x\in\bX$, we have 
\begin{align*}
\esp{ \cW\left(P_x, \PkNN_x\right) } \le R(k) + \frac{L}k \sum_{m=1}^k \esp{ Z_{x}^{(m)} }, 
\end{align*}
where $Z^x_{(m)}, m=1,\dots,M$ are the order statistics of $(\|X_m-x\|_\infty)_{m=1}^M$ in ascending order.
\end{lem}

\begin{proof}
We fix $x\in\bX$ for the rest of the proof. By Assumption \ref{hyp: data}, we have
\begin{align*}
&\esp{ \cW\left( P_x, \PkNN_x \right) } = M!\, \esp{ \1_{\|X_1-x\|_\infty \le \|X_2-x\|_\infty \le \dots \le \|X_M-x\|_\infty } \cW\left( P_x, \frac1k \sum_{\ell=1}^k \delta_{Y_\ell}  \right) } \\
&\quad= M! \int_{(\bX\times\bY)^M} \1_{ \|x_1-x\|_\infty \le \|x_2-x\|_\infty \le \dots \le \|x_M-x\|_\infty } \cW\left( P_x, \frac1k \sum_{\ell=1}^k \delta_{y_\ell} \right) \bigotimes_{\ell=1}^M \psi(\dif x_\ell\dif y_\ell) \\ 
&\quad= M! \int_{\bX^M} \1_{ \|x_1-x\|_\infty \le \|x_2-x\|_\infty \le \dots \le \|x_M-x\|_\infty } \int_{\bY^k} \cW\left( P_x, \frac1k \sum_{\ell=1}^k \delta_{y_\ell} \right) \bigotimes_{\ell=1}^k P_{x_\ell}(\dif y_\ell) \bigotimes_{j=1}^M\xi(\dif x_\ell).
\end{align*}
In view of Lemma \ref{lem:AbsDiffIntegrals}, replacing $P_{x_\ell}$ above with $P_x$, we have 
\begin{align*}
&\quad \esp{ \cW\left( P_x, \PkNN_x \right) }\\
&\quad \le M! \int_{\bX^M} \1_{ \|x_1-x\|_\infty \le \|x_2-x\|_\infty \le \dots \le \|x_M-x\|_\infty } \int_{\bY^k} \cW\left( P_x, \frac1k \sum_{\ell=1}^k \delta_{y_\ell} \right) \bigotimes_{\ell=1}^k P_{x}(\dif y_\ell) \bigotimes_{j=1}^M\xi(\dif x_\ell)\\
&\qquad+ \frac{L}{k} \sum_{\ell=1}^k M! \int_{\bX^M} \1_{\|x_1-x\|_\infty \le \|x_2-x\|_\infty \le \dots \le \|x_M-x\|_\infty } d_\bX(x_\ell, x) \bigotimes_{j=1}^M\xi(\dif x_\ell)\\
&\quad= \int_{\bY^k} \cW\left(\frac{1}{k}\sum_{l=1}^k \delta_{y_l}, P_x\right) \bigotimes_{\ell=1}^k P_{x}(\ud y_\ell) + \frac{L}{k} \sum_{\ell=1}^k M! \int_{\bX^M} \1_{\|x_1-x\|_\infty \le \|x_2-x\|_\infty \le \dots \le \|x_M-x\|_\infty } d_\bX(x_\ell, x) \bigotimes_{j=1}^M\xi(\dif x_\ell).
\end{align*}
In view of $R$ defined above \eqref{eq:UBWConvEmp} and $Z^x_{(m)}$ defined in the statement of this lemma, we conclude the proof.
\end{proof}

The next lemma provides an upper bound to $\sum_{m=1}^k \esp{ Z_{x}^{(m)} }$ listed in Lemma \ref{lem:EspWkNN}. 
\begin{lem}\label{lem:UBSumEspZ}
Let $Z^x_{(m)}$ be defined as in Lemma \ref{lem:EspWkNN}. Under Assumption \ref{hyp: data}, for any $x\in\bX$, we have 
\begin{align*}
\sum_{m=1}^k \esp{ Z^{x}_{(m)} } \le \frac{2}{\underline c^{\frac1{d_\bX}}d_\bX}\frac{M!}{\Gamma(M+\frac1{d_\bX}+1)} \sum_{m=1}^k \sum_{j=0}^{m-1} \frac{\Gamma(j+\frac1{d_\bX})}{j!}.
\end{align*}
\end{lem}
\begin{proof}
For any $x\in\bX$, we compute, since $Z^x_{(m)} \in [0,1]$,
\begin{align*}
    \esp{Z^x_{(m)}} = \int_0^1 \mathbb P\left[Z^x_{(m)} \ge r\right] \dif r = \int_0^1 \left(1-\mathbb P\left[Z^x_{(m)} < r\right]\right)\dif r,
\end{align*}
and we observe that $\left\{Z^x_{(m)} < r\right\}=\left\{N(x,r) \ge m\right\}$ with $N(x,r) := \sharp\left\{ 1 \le m \le M \,\middle|\, \|X_m-x\| < r \right\}$. We hence have
\begin{align*}
    \esp{Z^x_{(m)}} = \int_0^1 \left(1-\mathbb P\left[N(x,r) \ge m\right]\right)\dif r.
\end{align*}
Since $N(x,r) \sim \Binomial(M, \xi(B(x,r)))$ and $\xi(B(x,r)) \ge \underline c \lambda_\bX(B(x,r)) \ge \underline c \frac{r^{\dX}}{2^{\dX}}$ by Assumption \ref{hyp: data} (ii), we obtain that $\mathbb P\left[N(x,r) \ge m\right] \ge \mathbb P\left[N'(x,r) \ge m\right]$ with $N'(x,r) \sim \Binomial(M,\underline c \frac{r^{\dX}}{2^{\dX}})$ due to Lemma \ref{lem:BinStochDom}. This implies
\begin{align}\label{eq:EspZ(m)x}
    \esp{Z^x_{(m)}} &\le \int_0^1 \left(1-\mathbb P\left[N'(x,r) \ge m\right]\right)\dif r = \int_0^1 \mathbb P\left[N'(x,r) < m\right]\dif r \nonumber\\
    &= \sum_{j=0}^{m-1} \binom{M}{j} \int_0^1  \left(\underline c \frac{r^{\dX}}{2^{\dX}}\right)^j \left(1-\underline c \frac{r^{\dX}}{2^{\dX}}\right)^{M-j} \dif r = \frac{2}{\underline c^{\frac{1}{\dX}}\dX} \sum_{j=0}^{m-1} \binom{M}{j} \int_0^{\frac{\underline c}{2^{\dX}}}  r^{\frac{1}{\dX}+j-1} \left(1-r\right)^{M-j} \dif r \nonumber\\
    & \le 
    \frac{2}{\underline c^{\frac{1}{\dX}}\dX} \sum_{j=0}^{m-1} \frac{\Gamma(M+1)}{\Gamma(j+1)\Gamma(M-j+1)} \frac{\Gamma(\frac{1}{\dX}+j)\Gamma(M-j+1)}{\Gamma(\frac{1}{\dX}+M+1)} \nonumber\\
    &= \frac{2 M!}{\underline c^{\frac{1}{\dX}}\dX\Gamma(\frac{1}{\dX}+M+1)} \sum_{j=0}^{m-1} \frac{\Gamma(\frac{1}{\dX}+j)}{j!},
\end{align}
and the proof is over.
\end{proof}

We are now in position to prove Theorem \ref{thm:ExpectedRatekNN}. 
\begin{proof}[Proof of Theorem \ref{thm:ExpectedRatekNN}]
By combining Lemma \ref{lem:EspWkNN} and Lemma \ref{lem:UBSumEspZ}, noting that the upper bound is constant in $x$, we have 
\begin{align}\label{eq:UBSupEspWkNN}
\sup_{x\in\bX} \esp{ \cW\left(P_x, \hat\mu_{\cN^{k}(x)}\right) } \le  R(k) + \frac{L}k\frac{2M!}{\underline c^{\frac1{d_\bX}}\dX\Gamma(M+\frac1{d_\bX}+1)} \sum_{m=1}^k \sum_{j=0}^{m-1} \frac{\Gamma(j+\frac1{d_\bX})}{j!}.
\end{align} 
Below we only investigate the rate of the right hand side of \eqref{eq:UBSupEspWkNN} as $M\to\infty$, and do not keep track of the constant. We first analyze the second term in the right hand side of \eqref{eq:UBSupEspWkNN}. By Gautschi's inequality \cite[Eqs (10.6) and (12.2)]{M08}, we have 
\begin{align*}
\frac{\Gamma(j+\frac1{d_\bX})}{j!} = \frac{\Gamma(j+\frac1{d_\bX})}{\Gamma(j+1)} \le j^{\frac{1}{{d_\bX}}-1},\quad j\in\set{0}\cup\bN .
\end{align*}
Thus,
\begin{align}\label{eq:Order2ndErrkNN}
\sum_{m=1}^k \sum_{j=0}^{m-1} \frac{\Gamma(j+\frac1{d_\bX})}{j!} \le \sum_{m=1}^k  \sum_{j=0}^{m-1} j^{\frac{1}{{d_\bX}}-1} \lesssim \sum_{m=1}^k m^{\frac1{d_\bX}} \lesssim k^{1+\frac1{d_\bX}}. 
\end{align}
By Gautschi's inequality 
again, we have 
\begin{align*}
\frac{M!}{\Gamma(M+\frac1{d_\bX}+1)} = \frac{\Gamma(M+1)}{\Gamma(M+\frac1{d_\bX}+1)} \le M^{-\frac1{d_\bX}}.
\end{align*}
The above implies
\begin{align*}
\sup_{x\in\bX} \esp{ \cW\left(P_x, \hat\mu_{\cN^{k}(x)}\right) } \lesssim R(k) + M^{-\frac1{d_\bX}} k^{\frac1{d_\bX}}.
\end{align*}
We will split the remainder of the proof into three cases. 
\begin{itemize}
\item For ${d_\bY}=1$, by letting $k^{-\frac{1}{2}} \sim  M^{-\frac1{d_\bX}} k^{\frac1{d_\bX}}$, we yield
\begin{align*}
k \sim M^{\frac{2}{{d_\bX} + 2}} \quad\text{and}\quad \sup_{x\in\bX} \esp{ \cW\left(P_x, \hat\mu_{\cN^{k}(x)}\right) } \lesssim M^{-\frac{1}{{d_\bX} + 2}}
\end{align*}
\item For ${d_\bY}=2$, since the explicit solution of $k^{-\frac{1}{2}}\ln k \sim  M^{-\frac1{d_\bX}} k^{\frac1{d_\bX}}$ is elusive, we simply follow the configuration derived in the case of ${d_\bY}=1$ and yield
\begin{align*}
k \sim M^{\frac{2}{{d_\bX} + 2}} \quad\text{and}\quad \sup_{x\in\bX} \esp{ \cW\left(P_x, \hat\mu_{\cN^{k}(x)}\right) } \lesssim M^{-\frac{1}{{d_\bX} + 2}}\ln M.
\end{align*}
\item For ${d_\bY}\ge 3$, by letting $k^{-\frac{1}{{d_\bY}}} \sim  M^{-\frac1{d_\bX}} k^{\frac1{d_\bX}}$, we yield
\begin{align*}
k \sim M^{\frac{\dY}{{d_\bX} + {d_\bY}}} \quad\text{and}\quad \sup_{x\in\bX} \esp{ \cW\left(P_x, \hat\mu_{\cN^{k}(x)}\right) } \lesssim M^{-\frac{1}{{d_\bX} + {d_\bY}}}.
\end{align*}
\end{itemize}
The proof is complete.
\end{proof}

\subsection{Proof of Theorem \ref{thm:ConcenkNNnew}}
\label{subsec:Pf:thm:ConcenkNNnew}
\begin{proof}[Proof of Theorem \ref{thm:ConcenkNNnew}]
We will proceed by using Efron-Stein inequality. Let $(X_1',Y_1')$ be an independent copy of $(X_1,Y_1)$, and define $\cD':=\set{(X_1',Y_1'), (X_2,Y_2), \dots, (X_M,Y_M)}$. In view of Assumption \ref{hyp: data} (i), by the triangle inequality of $\cW$, it is sufficient to investigate
\begin{align*}
\frac12 M\,\esp{ \left( \int_\bX \cW\Big(\hat\mu^{\cD}_{\cN^{k,\cD_\bX}(x)}, \hat\mu^{\cD'}_{\cN^{k,\cD'_\bX}}(x)\Big) \nu(\dif x) \right)^2 }.
\end{align*}
Note that for $\cW\Big(\hat\mu^{\cD}_{\cN^{k,\cD_\bX}(x)}, \hat\mu^{\cD'}_{\cN^{k,\cD'_\bX}}(x)\Big)$ to be positive, the event $A_x\cup A'_x$ is necessary, where
\begin{gather*}
A_x := \left\{ X_1 \in \cN^{k,\cD_\bX}(x) \right\} \quad\text{and}\quad 
A'_x := \left\{ X_1' \in \cN^{k,\cD_\bX}(x) \right\}. 
\end{gather*}
Moreover,
\begin{align*}
\cW\Big(\hat\mu^{\cD}_{\cN^{k,\cD_\bX}(x)}, \hat\mu^{\cD'}_{\cN^{k,\cD'_\bX}}(x)\Big) \le \frac1k.
\end{align*}
It follows that 
\begin{align}
&\esp{ \left( \int_\bX \cW\Big(\hat\mu^{\cD}_{\cN^{k,\cD_\bX}(x)}, \hat\mu^{\cD'}_{\cN^{k,\cD'_\bX}}(x)\Big) \nu(\dif x) \right)^2 } \le \frac{1}{k^2} \esp{ \left( \int_\bX \1_{A_x\cup A'_x} \nu(\dif x) \right)^2 } \label{eq:UBExpnSqIntWkNN} \\
&\quad \le \frac{1}{k^2} \esp{ \int_\bX \1_{A_x\cup A'_x} \nu(\dif x) } \le \frac{2}{k^2} \int_\bX \bP\left[A_x\right] \nu(\dif x).
\end{align}
where the second inequality is due to the fact that the integral value always fall into in $[0,1]$, and we have used Fubini-Tonelli theorem and the subadditivity of probability in the third inequality. Regarding $\bP\left[A_x\right]$, by the symmetry stemming from Assumption \ref{hyp: data} (i) and the random tie-breaking rule in Definition \ref{def:kNN}, we have 
\begin{align*}
\bP\left[A_x\right] = \binom{M-1}{k-1}\binom{M}{k}^{-1} = \frac{k}{M}.
\end{align*}
Consequently, 
\begin{align*}
M\,\esp{ \left( \int_\bX \cW\Big(\hat\mu^{\cD}_{\cN^{k,\cD_\bX}(x)}, \hat\mu^{\cD'}_{\cN^{k,\cD'_\bX}}(x)\Big) \nu(\dif x) \right)^2 } \le \frac{2}{k}.
\end{align*}
Invoking Efron-Stein inequality, we conclude the proof of \eqref{eq:UBVarNoDom}.

We now assume additionally that $\nu \le \overline C \lambda_\bX$ to prove the second statement. Following from \eqref{eq:UBExpnSqIntWkNN}, by using the positivity and subadditivity of indicator functions as well as AM–GM inequality, we have 
\begin{align*}
\esp{ \left( \int_\bX \cW\Big(\hat\mu^{\cD}_{\cN^{k,\cD_\bX}(x)}, \hat\mu^{\cD'}_{\cN^{k,\cD'_\bX}}(x)\Big) \nu(\dif x) \right)^2 } &\le \frac{4}{k^2} \esp{ \left( \int_\bX \1_{A_x} \nu(\dif x) \right)^2 } \le \frac{4\overline C^2}{k^2} \esp{ \left( \int_\bX \1_{A_x} \lambda_\bX(\dif x) \right)^2 } \\
& 
\le \frac{4\overline C^2}{k^2} \int_{[0,1]} \bP\left[ \left( \int_\bX \1_{A_x} \lambda_\bX(\dif x) \right)^2 > \delta  \right] \dif\delta,
\end{align*}
where in the second inequality we have used the condition that $\nu$ is dominated by $\lambda_\bX$, and in the last one the alternative expression of expectation for positive random variables. 
Let $\mathsf{Cube}^\iota_{\bX}$ be the set of cubes within $\bX$ with edge length $\iota$. Since $\nu$ is dominated by $\lambda_\bX$, with probability $1$ we have 
\begin{gather*}
A_x = \left\{ \text{at most $(k-1)$ of $X_\ell,\ell=2,\dots,M$, falls into } B^{\|X_1-x\|_{\infty}}_{x} \right\}, \\
A'_x = \left\{ \text{at most $(k-1)$ of $X_\ell,\ell=2,\dots,M$, falls into } B^{\|X_1'-x\|_{\infty}}_{x} \right\}.
\end{gather*}
It follows that 
\begin{align*}
\left\{ \sum_{m=2}^M\1_{B}(X_m) > k, \quad \forall B\in\mathsf{Cube}_\bX^\iota \right\} \subseteq \left\{ \int_\bX \1_{A_x}\lambda(\dif x) \le (2\iota)^{d_\bX} \right\}.
\end{align*}
By combining the above and setting $\delta=(2\iota)^{2d_\bX}$, we yield
\begin{align}\label{eq:UBExpnSqIntWkNN2}
\esp{ \left( \int_\bX \cW\Big(\hat\mu^{\cD}_{\cN^{k,\cD_\bX}(x)}, \hat\mu^{\cD'}_{\cN^{k,\cD'_\bX}}(x)\Big) \nu(\dif x) \right)^2 } \le \frac{4\overline C^2}{k^2} \int_{[0,1]} \bP\left[ \frac1{M-1}\sum_{m=2}^M\1_{B}(X_m) \le \frac{k}{M-1},\,\forall B\in\mathsf{Cube}_\bX^{\frac12\delta^{\frac1{d_\bX}}} \right] \dif\delta
\end{align}
In order to proceed, we state and prove a useful technical lemma using the Rademacher complexity technique (cf. \cite[Section 4]{Wainwright2019High}). Below we let $\mathsf{Cube}_\bX$ be the set of cubes inside $\bX$ with edge lengths within $[0,1]$. 
\begin{lem}\label{lem:InfEmpConcen}
Let $X_2,\dots, X_M$ be introduced in Assumption \ref{hyp: data} (i). For $\varepsilon\ge 0$, 
\begin{align*}
\bP\left[ \frac1{M-1}\sum_{m=2}^M\1_{B}(X_m) \le  \underline c\lambda_\bX(B) - 8 \sqrt{\frac{2d_\bX\ln(M)}{M-1}} - \varepsilon, \quad\forall B\in\mathsf{Cube}_\bX \right] \le \exp\left(-\frac{M-1}{2}\varepsilon^2\right).
\end{align*}
\end{lem}
\begin{proof}
Let $\bm x^M=(x_2^M,\dots,x_M^M)\in\bX^{M-1}$. To utilize the machinery of Rademacher complexity, we will upper bound the cardinality of the set $\set{\1_B(\bm x^M): B\in \mathsf{Cube}_\bX}$, where $\1_B$ applies entry-wise. More precisely, $\1_B(\bm x^M)=(\1_B(x_2^M),\dots,\1_B(x_M^M))$. To start with, we first note that for $d=1,\dots,d_{\bX}$, the projected $(x_{2,d}^M,\dots,x_{M,d}^M)$ at most separates axis-$d$ into $M$ intervals. Additionally, each element in $\set{\1_B(\bm x^M): B\in \mathsf{Cube}_\bX}$ corresponds to selecting two intervals (one for starting and one for ending of the cube) on each axis.  Therefore, the cardinality is at most $M^{2d_\bX}$, i.e., $\mathsf{Cube}_\bX$ has polynomial discrimination $2d_\bX$. It follows from \cite[Lemma 4.14 and Theorem 4.10]{Wainwright2019High} that, for any $\varepsilon\ge 0$,
\begin{align*}
\bP\left[ \sup_{B\in\mathsf{Cube}_\bX} \left|\frac1{M-1}\sum_{m=2}^M\1_{B}(X_m) - \xi(B) \right| \ge  8 \sqrt{\frac{2d_\bX\ln(M)}{M-1}} + \varepsilon \right] \le \exp\left(-\frac{M-1}{2}\varepsilon^2\right).
\end{align*}
Finally, in view of Assumption \ref{hyp: data} (ii), we conclude the proof of Lemma \ref{lem:InfEmpConcen}. 
\end{proof}

In view of \eqref{eq:UBExpnSqIntWkNN2} and Lemma \ref{lem:InfEmpConcen}, for $\delta\in[0,1]$, we consider $\varepsilon\ge 0$ such that
$$\frac{k}{M-1} = \frac{\underline c\delta^{\frac12}}{2^{d_\bX}} - 8 \sqrt{\frac{2d_\bX\ln(M)}{M-1}} - \varepsilon.$$
Note that this is feasible only if $\frac{4^{d_\bX}}{\underline c^2}\left(8 \sqrt{\frac{2d_\bX\ln(M)}{M-1}} + \frac{k}{M-1}\right)^2\le 1$.\footnote{We do not include this condition in the statement of Theorem \ref{thm:ConcenkNNnew}, as the bound presented remains valid, albeit vacuous, if this condition is not met.} It follows that
\begin{align*}
&\bP\left[ \frac1{M-1}\sum_{m=2}^M\1_{B}(X_m) \le \frac{k}{M-1},\,\forall B\in\mathsf{Cube}_\bX^{\frac12\delta^{\frac1{2d_\bX}}} \right] \\
&\quad\le \begin{cases}
1, & \delta\in\left[0, \frac{4^{d_\bX}}{\underline c^2}\left(8 \sqrt{\frac{2d_\bX\ln(M)}{M-1}} + \frac{k}{M-1}\right)^2\right], \\
\exp\left(-\frac{M-1}{2}\left(\frac{\underline c\delta^{\frac12}}{2^{d_\bX}} - 8 \sqrt{\frac{2d_\bX\ln(M)}{M-1}} - \frac{k}{M-1}\right)^2\right), & \delta\in\left(\frac{2^{d_\bX}}{\underline c}\left(8 \sqrt{\frac{2d_\bX\ln(M)}{M-1}} + \frac{k}{M-1}\right)^2,1\right].
\end{cases}
\end{align*}
The above together with \eqref{eq:UBExpnSqIntWkNN2} implies 
\begin{align*}
&\frac12 M \esp{ \left( \int_\bX \cW\Big(\hat\mu^{\cD}_{\cN^{k,\cD_\bX}(x)}, \hat\mu^{\cD'}_{\cN^{k,\cD'_\bX}}(x)\Big) \nu(\dif x) \right)^2 } \\
&\quad\le \frac{2\overline C^2 M}{k^2} \Biggl( \frac{4^{d_\bX}}{\underline c^2}\left( 8\sqrt{\frac{2d_\bX\ln(M)}{M-1}} + \frac{k}{M-1}\right)^2 \\
&\qquad\qquad\qquad+ \int_{\frac{2^{d_\bX}}{\underline c}\left(8 \sqrt{\frac{2d_\bX\ln(M)}{M-1}} + \frac{k}{M-1}\right)}^1 \exp\left(-\frac{M-1}{2}\left(\frac{\underline c\eta}{2^{d_\bX}} - 8 \sqrt{\frac{2d_\bX\ln(M)}{M-1}} - \frac{k}{M-1}\right)^2\right)  2\eta\dif \eta \Bigg),
\end{align*}
where we have performed a change of variable $\eta=\delta^{\frac12}$ in the last line. Relating to exponential and normal density functions, we calculate the integral to obtain
\begin{align*}
&\frac12 M \esp{ \left( \int_\bX \cW\Big(\hat\mu^{\cD}_{\cN^{k,\cD_\bX}(x)}, \hat\mu^{\cD'}_{\cN^{k,\cD'_\bX}}(x)\Big) \nu(\dif x) \right)^2 } \\
&\quad \le \frac{2\overline C^2 M}{k^2} \frac{4^{d_\bX}}{\underline c^2}\Bigg( \left( 8\sqrt{\frac{2d_\bX\ln(M)}{M-1}} + \frac{k}{M-1}\right)^2  + \frac{\sqrt{2\pi}}{\sqrt{M-1}}\left(8 \sqrt{\frac{2d_\bX\ln(M)}{M-1}} + \frac{k}{M-1}\right) + \frac{4}{M-1} \Bigg),
\end{align*}
where we note the right hand side is of $O\left(\left(\frac{\sqrt{\ln(M)}}{k}+\frac{1}{\sqrt{M}}\right)^2 + \frac1k\left(\frac{\sqrt{\ln(M)}}{k}+\frac{1}{\sqrt{M}}\right) + \frac1{k^2}\right)$. Invoking Efron-Stein inequality, we conclude the proof.
\end{proof}

\subsection{Proof of Proposition \ref{prop:PThetanew}}\label{subsec:Pf:prop:PThetanew}
\begin{proof}[Proof of Proposition \ref{prop:PThetanew}]
By triangle inequality,
\begin{align*}
    &\esp{\int_\bX \cW(P_x,\PNNT_x) \ud x} \\
    &\quad\le \esp{\int_\bX \cW(P_x,P^\Theta_x) \left(\lambda_\bX-\frac1N\sum_{n=1}^{N}\delta_{\tilde X_n}\right)(\ud x)} + \esp{\frac1N\sum_{n=1}^{N} \cW(P_{\tilde X_n},\Pgen_{\tilde X_n})} + \esp{\frac1N\sum_{n=1}^{N} \cW(\Pgen_{\wt X_n},\PNNT_{\tilde X_n})}.
\end{align*}
Then, by Assumption \ref{hyp: kernel lip} and \eqref{eq:UBTrainLip}, 
\begin{align*}
    \esp{\int_\bX \cW(P_x,P^\Theta_x) \left(\lambda_\bX-\frac1N\sum_{n=1}^{N}\delta_{\tilde X_n}\right)(\ud x)} \le \esp{(L+L^\Theta)\cW\left(\lambda_\bX,\frac1N\sum_{n=1}^{N}\delta_{\tilde X_n}\right)}.
\end{align*}
In view of Assumption \ref{hyp:QueryPoints}, we have
\begin{align*}
    \esp{\frac1N\sum_{n=1}^{N} \cW(P_{\tilde X_n},\Pgen_{\tilde X_n})} = \frac1N\sum_{n=1}^{N} \esp{ \esp{ \cW(P_{\tilde X_n},\Pgen_{\tilde X_n})\Big| \tilde X_n} } = \esp{ \int_\bX \cW(P_x, \Pgen_x) \dif x }.
\end{align*}
Combining the above, we prove the first statement.

As for the second statement, consider $Q,Q':\bX\to\cP(\bY)$ that are Lipschitz-continuous with constants $L, L'$. Suppose that
\begin{align*}
\cW(Q_{x^*},Q'_{x^*}) = \sup_{x\in\bX} \cW(Q_x,Q'_x) = \delta 
\end{align*}
for some $\delta> 0$ and $x^*\in\bX$. This supremum is indeed attainable because $\bX$ is compact that $x\mapsto\cW(Q_x,Q'_x)$ is continuous. Consequently, by triangle inequality and the Lipschitz-continuity, we have
\begin{align*}
\cW(Q_x, Q'_{x}) & \ge \Big(\cW(Q_{x}, Q'_{x^*}) - \cW(Q'_{x}, Q'_{x^*})\Big) \vee 0 \ge \Big( \cW(Q_{x^*},Q'_{x^*}) -\cW(Q_{x^*},Q_x) - \cW(Q'_{x^*},Q'_{x})\Big) \vee 0\\
& \ge \Big(\delta - (L+L')\|x-x^*\|_\infty\Big) \vee 0, \quad x\in\bX.
\end{align*}
We may then lower bound $\int_\bX \cW(Q_x,Q'_x) \dif x$ with the volume of the cone on right hand side above (note the worst case is when $x^*=(0,0)$),
\begin{align*}
\int_\bX \cW(Q_x,Q'_x) \dif x \ge \int_0^{\delta} \left(\frac{\delta-z}{L+L'}\right)^{d_\bX} \dif z = \frac{\delta^{d_\bX+1}}{(d_\bX+1) (L+L')^{d_\bX}}.
\end{align*} 
It follows that
\begin{align*}
\sup_{x\in\bX} \cW(Q_x,Q'_x) \le (d_\bX+1)^{\frac{1}{d_\bX+1}}  (L+L')^{\frac{d_\bX}{d_\bX+1}} \left(\int_\bX \cW(Q_x,Q'_x) \dif x\right)^{\frac{1}{d_\bX+1}},
\end{align*}
which completes the proof.
\end{proof}

\section{Implementation details and ablation analysis}\label{sec:ImplD}

In this section, we will provide further implementation details and conduct ablation analysis of the components highlighted in Section \ref{subsec:ImplOverview}.

\subsection{Comparing ANNS-RBSP to exact NNS}\label{subsec:ANNSComparison}
Algorithm \ref{algo:RBSP} outlines a single slice of RBSP, which divides an array of $x$'s into two arrays of a random ratio along a random axis. Throughout the training, we execute RBSP $5$ times during each training epoch, yielding $2^5=32$ parts. Within each part, we then select a small batch of $8$ query points, locating the $k$ nearest neighbors for each query point within the same part. In Table \ref{tab:rbsp}, we compare the execution times of exact NNS and ANNS-RBSP. ANNS-RBSP offers considerable time savings for $M=10^6$, while exact NNS is more efficient for $M=10^5$ or fewer. 

\begin{algorithm}
\caption{Single slice of random binary space partitioning}
\footnotesize
\begin{algorithmic}[1]\label{algo:RBSP}
\renewcommand{\algorithmicrequire}{\textbf{Input:}}
\renewcommand{\algorithmicensure}{\textbf{Output:}}
\REQUIRE data $\mD_\bX=(x_i)_{i=1}^M\subset[0,1]^{d_\bX}$, arrays of indexes $\mS_d, d=1,\dots,d_{\bX}$ of length $M$ with the $j$-th entry indicating the position of the $j$-th smallest value in the $d$-th dimension of $\mD_{\bX}$, a boolean array $\mB$ of length $M$ with the $i$-th entry indicating whether $x_i$ is involved in the current slicing, a rectangle $R$ that bounds $x_i$'s involved in the current slicing, i.e., $R$ corresponds to $\mB$, a parameter $r_{\text{edge}}\in(1,\infty)$ for avoiding thin rectangles, an interval $[\underline p, \overline p]\in(0,1)$ for random bisecting ratio
\ENSURE two boolean arrays $\mB, \mB'$ of length $M$ indicating the bisected data, two bounding rectangles $R,R'$ that correspond to $\mB, \mB'$  
\STATE Randomly pick a dimension $d$
\IF{The edge ratio of $R$ exceeds $r_{\text{edge}}$}
\STATE Replace $d$ with that corresponds to the longest edge
\ENDIF
\STATE Rearrange $\mB$ according to $\mS_d$ by $\tilde\mB \leftarrow \mB[\cS_d]$
\STATE Pick out the indexes from $\mS_d$ involved in ANNS by $\tilde\mS_d \leftarrow \mS_d[\tilde\cB]$
\STATE Generate $p\sim\Uniform([\underline p, \overline p])$ and round $p$ into $\tilde p$ so that $\tilde p\,\text{\textsf{len}}(\tilde\mS_d)$ is an integer 
\STATE Bisect $\tilde\mS_d$ in two arrays with length $\tilde p\,\text{\textsf{len}}(\tilde\mS_d)$ and $(1-\tilde p)\text{\textsf{len}}(\tilde\mS_d)$, denoted by $\tilde\mS_d$ and $\tilde\mS_d'$
\STATE Form new bounding rectangles $R, R'$ using $\tilde\mS_d,\tilde\mS_d',\mD_\bX$ and the original $R$ (may enforce some overlap here)\\[-.3em]
\STATE Initialize two boolean arrays $\mB,\mB'$ with length $M$ and all entries being \textsf{False}
\STATE $\mB[\tilde\mS_d]\leftarrow\text{\textsf{True}}$, $\mB'[\tilde\mS_d']\leftarrow\text{\textsf{True}}$
\RETURN $\mB, \mB', R, R'$ 
\end{algorithmic} 
\end{algorithm}

\begin{table}[h!]
\centering
\begin{tabular}{ |c|c|c|c|  }
\hline
&  $M=10^4$ & $M=10^5$ & $M=10^6$\\
\hline
$d_{\bX}=1$   &    $(\bm{0.03},2.2)$   &  $(\bm{0.5},2.5)$ & $(9.4,\bm{2.7})$ \\
\hline
$d_{\bX}=3$   &  $(\bm{0.04},2.2)$   &  $(\bm{0.6},2.7)$ & $(12.8,\bm{2.8})$ \\
\hline
$d_{\bX}=10$   &   $(\bm{0.07},2.2)$   &  $(\bm{0.8},2.7)$ & $(26.8,\bm{3.4})$ \\
\hline
\end{tabular}
\caption{Execution times of two NNS methods.}
\label{tab:rbsp}
\medskip
\small
This table compares the execution times for 100 runs of exact NNS versus ANNS-RBSP, both utilizing parallel computing, facilitated by PyTorch, with an NVIDIA L4 GPU. Each iteration (approximately) finds the 300 nearest neighbors from $M$ samples for all of 256 randomly generated query points. The values within each parenthesis denote the seconds consumed by both methods, with the first number corresponding to exact NNS. For faster processing, exact NNS employs a 3D tensor, except in the case of $M=10^6, d_{\bX}=10$, where memory limitations necessitate using a for-loop over individual query points. ANNS-RBSP regenerates a new partition each run. The table does not include the time required to sort the data along all dimensions, which takes about $0.2$ seconds in the worst case and is not repeatedly executed.
\end{table}

It's important to note that ANNS-RBSP may introduce additional errors by inaccurately including points that are not within the $k$ nearest neighbors. As elucidated in the proof of Theorem \ref{thm:ExpectedRatekNN}, the magnitude of this induced error can be understood by comparing the excessive distance incurred to that of exact NNS. For simplicity, we investigate the difference below 
\begin{align*}
\Delta := \frac{1}{N_{\text{batch}}}\sum_{i=1}^{N_{\text{batch}}}\left(\frac{1}{k}\sum_{j=1}^k \left\|\check X'_{ij}-\tilde X_i\right\|_1 - \frac{1}{k}\sum_{j=1}^k \left\|\check X_{ij}-\tilde X_i\right\|_1\right),
\end{align*}
where $\tilde X_i$'s are query points, and $\check X_{ij}, \check X'_{ij}$ are the $k$-nearest-neighbor identified by exact NNS and ANNS-RBSP, respectively. In our experiments, we evaluated scenarios with $d_\bX=3,10$ and $k=300$. Regarding the data, we generated $M=10^4,10^5,10^6$ samples from $\Uniform([0,1]^{d_\bX})$. Once the data set is generated, we fixed the data and conducted 100 simulations of $\Delta$, each with $N_{\text{batch}}=256$ query points. This process was repeated 10 times, each with a separately generated data. The results are illustrated in Figure \ref{fig:DiffANNS}. It is expected that $\Delta$ will approach $0$ as the sample size $M$ tends to infinity. The convergence rate is likely influenced by factors such as $d_\bX$, $k$, and $N_{\text{batch}}$. Further analysis of the convergence of ANNS-RBSP will be conducted in future studies.

\begin{figure}[htbp]
\centering
\begin{subfigure}[t]{0.4\linewidth}
\centering
\makebox[10pt]{\raisebox{70pt}{\rotatebox[origin=c]{90}{$10^5 samples$}}}
\includegraphics[width=0.9\textwidth]{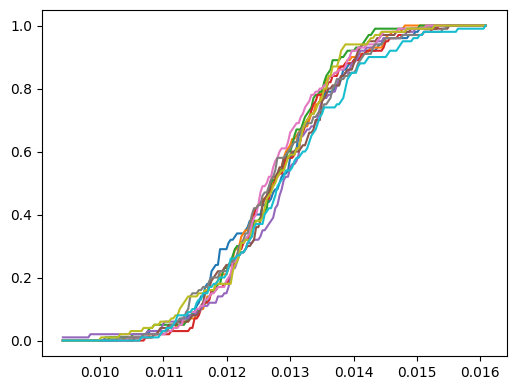} \\
\makebox[10pt]{\raisebox{50pt}{\rotatebox[origin=c]{90}{$10^6 samples$}}}
\includegraphics[width=0.9\textwidth]{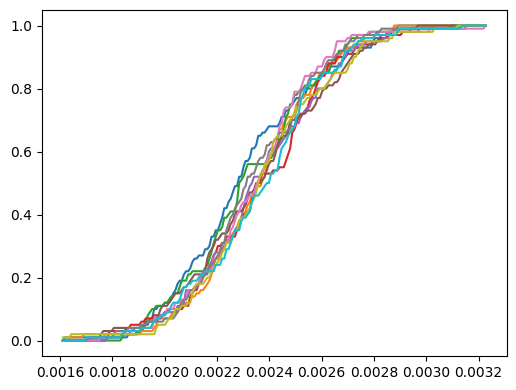}
\caption{$d_{\bX}=3$}
\end{subfigure}
\begin{subfigure}[t]{0.4\linewidth}
\centering
\includegraphics[width=0.9\textwidth]{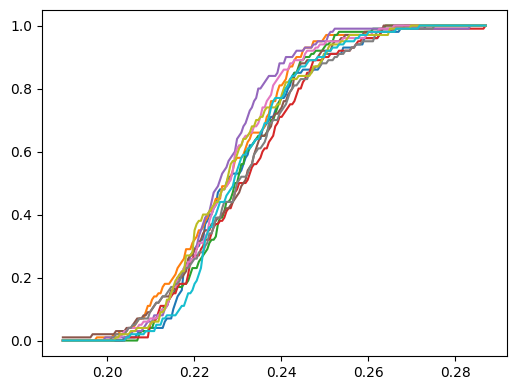} 
\includegraphics[width=0.9\textwidth]{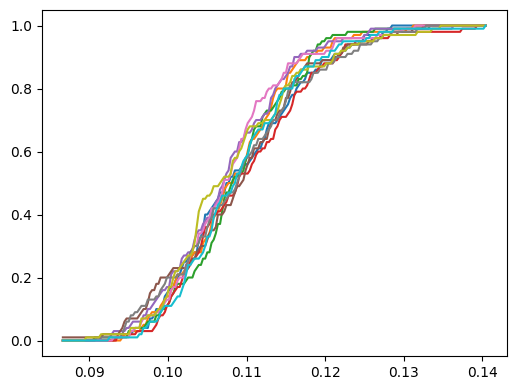}
\caption{$d_{\bX}=10$}
\end{subfigure}
\caption{Empirical CDFs of $\Delta$.}
\label{fig:DiffANNS}
\medskip
\small
We compare the empirical CDFs of $\Delta$. Each line corresponding to a independently generated set of data. Each plot includes 10 empirical CDFs. Note the difference in the $x$ axis scale.
\end{figure}

\subsection{An implementation of the Sinkhorn algorithm}\label{subsec:Sinkhorn}
In this section, we will detail our implementation of the Sinkhorn algorithm and highlight a few novel treatments that seem to enhance the training of the neural estimator. While the mechanisms  are not yet fully understood, they constitute important improvement in the accuracy of the neural estimator.

Let us first recall the iterative procedure involved in the Sinkhorn algorithm. We follow the setup in Section \ref{subsubsec:W}. In particular, the row indexes of the cost matrix stand for atoms in the empirical measures, while the column indexes stand for atoms produced by the neural estimator. We set $N_{\text{atom}}=k$ and let $\mathsf{u}^{(0)}, \mathsf{v}^{(0)}$ be column vectors of size $k$ with all entries being $k^{-1}$. We will suppress the dependence on $y$ from the notation. Upon setting
\begin{align}\label{eq:DefK}
\mK^\epsilon := \exp\left(-\frac{\mC}{\epsilon}\right)
\end{align}
with entry-wise exponential, the Sinkhorn algorithm performs repeatedly
\begin{align}\label{eq:SinkhornIter}
\mathsf{u}^{(\ell+1)} = \frac{\mathsf{u}^{(0)}}{\mK^\epsilon \mathsf{v}^{(\ell)}} \quad \text{and} \quad \mathsf{v}^{(\ell+1)} = \frac{\mathsf{v}^{(0)}}{(\mK^\epsilon)^{\top} \mathsf{u}^{(\ell+1)}},
\end{align}
where the division is also calculated entry-wise. After a certain number of iterations, denoted as $N_{\text{iter}}$, we obtain an approximate optimal transport plan for problem \eqref{eq:RegDOptTrans}:
\begin{align*}
\mT^\epsilon = \diag(\mathsf{u}^{(N_{\text{iter}})}) \mK^\epsilon \diag(\mathsf{v}^{(N_{\text{iter}})}).
\end{align*}
Let us set $\epsilon=1$ momentarily. Note that if the entries of $\mC$ are excessively large, $\mK$ effectively becomes a zero matrix, which impedes the computations in \eqref{eq:SinkhornIter}. This issue may occur at the initiation of the neural estimator or during training, possibly due to the use of stochastic gradient descent. To tackle this issue, we employ a rule-of-thump normalization that
\begin{align}\label{eq:Ktilde}
\tilde\mK^\epsilon := \exp\left(-\frac{\mC}{\tilde c\epsilon}\right) \quad\text{with}\quad \tilde c := \min_{i}\max_{j} \mC_{ij},
\end{align}
and use $\tilde\mK^\epsilon$ instead of $\mK^\epsilon$ in \eqref{eq:SinkhornIter}. Regarding the selection of $\epsilon$ and the number of iterations, we currently lack a method for adaptively determining these values. Instead, we adjust them manually based on training episodes. This manual adjustment works well for all models discussed in this paper. For more information, please see Appendix \ref{sec:Config}.

As alluded in Section \ref{subsubsec:W}, we enforce sparsity on the transport plan to improve the performance of the neural estimator. Let $\tilde\mT^\epsilon$ be the output of the Sinkhorn algorithm. We construct $\hat\mT^\epsilon$ and $\check\mT^\epsilon$ 
by setting the row-wise and column-wise maximum of $\tilde\mT^\epsilon$ to $k^{-1}$, respectively, and setting the remaining entries to $0$. We then use 
\begin{align}\label{eq:Tbar}
\overline\mT^\epsilon = \gamma \hat\mT^\epsilon + (1-\gamma) \check\mT^\epsilon,
\end{align}
where $\gamma\in[0,1]$ is a hyper-parameter, in gradient descent \eqref{eq:DphiDy}. We observe that $\hat\mT^\epsilon$ relates each atom in the empirical measure to a single corresponding atom from the neural estimator, and $\check\mT^\epsilon$ does the same in reverse. The optimal choice of $\gamma$ remains an open question, though we have set $\gamma=0.5$ in all three models.

Next, we explore the impact of enforcing sparsity and varying the choices of $\gamma$. Figure \ref{fig:1DSparse} compares the performance in Model 1 under different sparsity parameters. When no sparsity is enforced, the neural estimator tend to handles singularities more adeptly, but may overlooks points located on the periphery of the empirical joint distribution, potentially resulting in overly concentrated atoms from the neural estimator (see around $x=0.1, 0.9$). Compare Figure \ref{fig:1DAllW} and \ref{fig:1DAllWNoSparsity} for the extra error due to the lack of enforced sparsity. This phenomenon is more noticeable in Model 3. We refer to panel (2,3) of Figure \ref{fig:3DCondCDFNosparse} in Appendix \ref{sec:MorePlots} for an example. Moreover, Figure \ref{fig:HistProjWNoSparsity}, which is obtained without enforced sparsity, indicates a downgrade in accuracy when compared to Figure \ref{fig:HistProjW}. 

\begin{figure}[htbp]
\centering
\begin{subfigure}[t]{0.24\linewidth}
\centering
\includegraphics[width=0.90\textwidth]{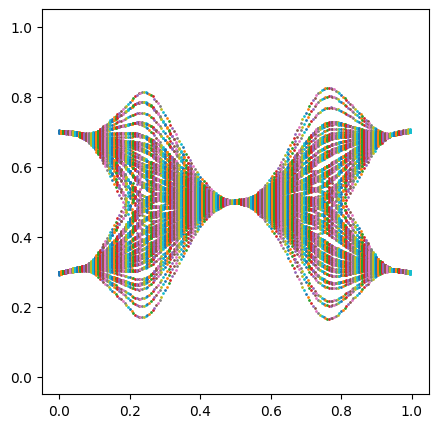}
\caption{No sparsity enforced}
\end{subfigure}
\begin{subfigure}[t]{0.24\linewidth}
\centering
\includegraphics[width=0.90\textwidth]{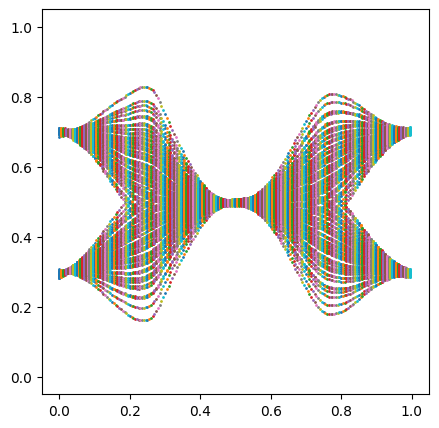}
\caption{$\gamma=0$}
\end{subfigure}
\begin{subfigure}[t]{0.24\linewidth}
\centering
\includegraphics[width=0.90\textwidth]{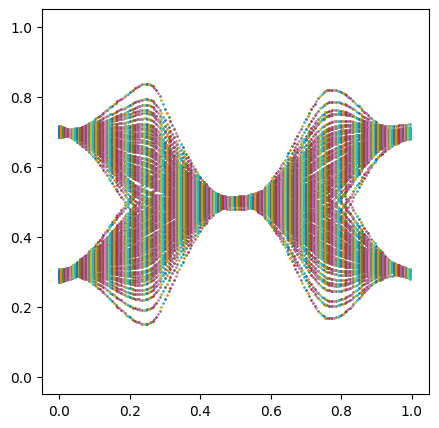}
\caption{$\gamma=0.5$}
\end{subfigure}
\begin{subfigure}[t]{0.24\linewidth}
\centering
\includegraphics[width=0.90\textwidth]{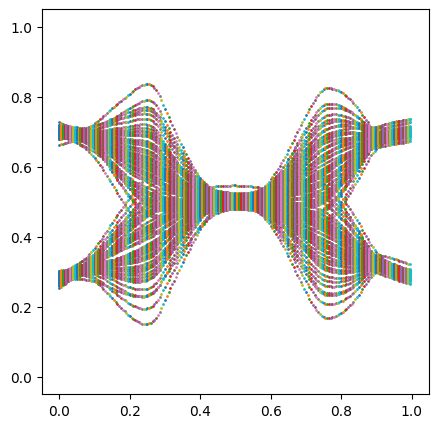}
\caption{$\gamma=1$}
\end{subfigure}
\caption{LipNet under Model 1 with different sparsity enforcement.}
\label{fig:1DSparse}
\end{figure}

\begin{figure}[!htbp]
\centering
\begin{subfigure}[t]{1\linewidth}
\centering
\includegraphics[width=0.80\textwidth]{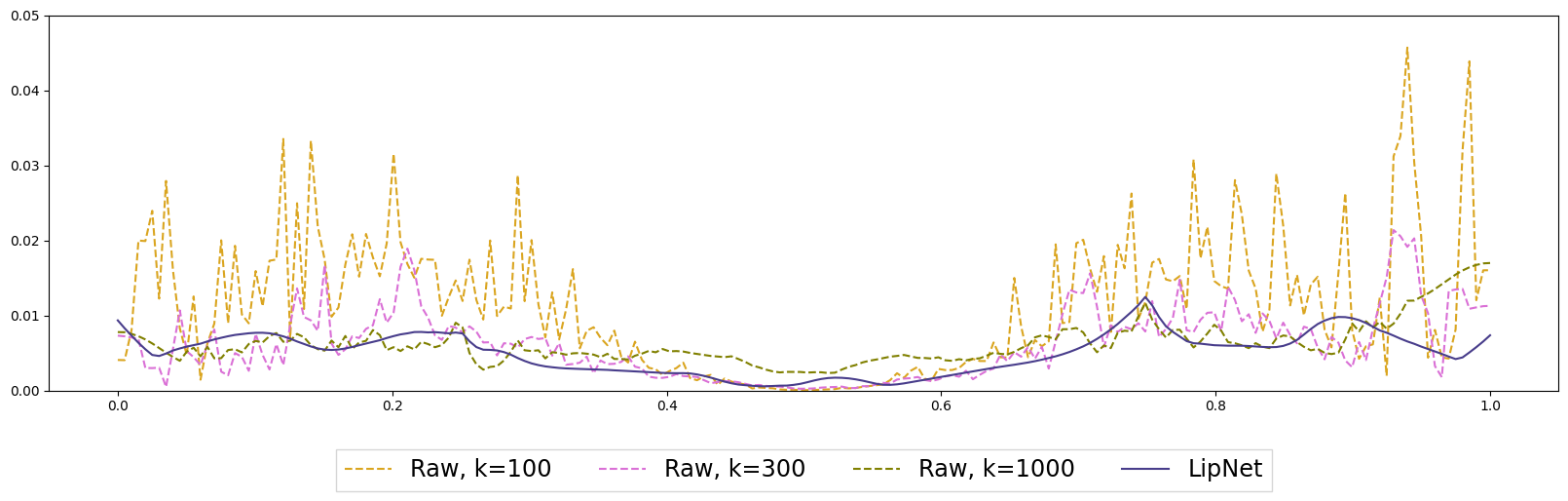}
\end{subfigure}
\caption{\centering Errors at different $x$'s of various estimators under Model 1, LipNet is trained without enforced sparsity.}
\label{fig:1DAllWNoSparsity}
\medskip
\small
We compute the $\cW$-distance between estimators and the true conditional distribution at different $x$'s. The setting is similar to Figure \ref{fig:1DAllW}, but LipNet is trained without enforcing sparsity on the transport plan. The errors of LipNet at around $x=0.1,0.9$ are slightly higher than those in Figure \ref{fig:1DAllW}.
\end{figure}

Finally, it is not recommended to use $\overline\mT^\epsilon$ at the early stages of training, as our empirical experiments suggest this could deteriorates performance. In training, we start by not enforcing sparsity and then begin to enforce it in later episodes. We refer to Appendix \ref{sec:Config} for further details of the training configuration.

\subsection{More on LipNet}\label{subsec:LipNet}

We will investigates the impact of various hyper-parameters on the performance of LipNet. The LipNets presented in this section are trained with the same hyper-parameters as in Section \ref{subsec:Experiments} (see also Appendix \ref{sec:Config}), expect for those specified otherwise.

\subsubsection{Activation function}\label{subsubsec:Actvn}
Switching the activation function from ELU to Rectified Linear Unit (ReLU) appears to retain the adaptive continuity property.  In Figure \ref{fig:1DActvn}, we illustrate the joint distribution and the average absolute derivatives of all atoms of LipNet with ReLU activation. The outcomes are on par with those achieved using ELU activation as shown in Figure \ref{fig:1D}. 

\begin{figure}[htbp]
\centering
\begin{subfigure}[t]{0.24\linewidth}
\centering
\includegraphics[width=0.90\textwidth]{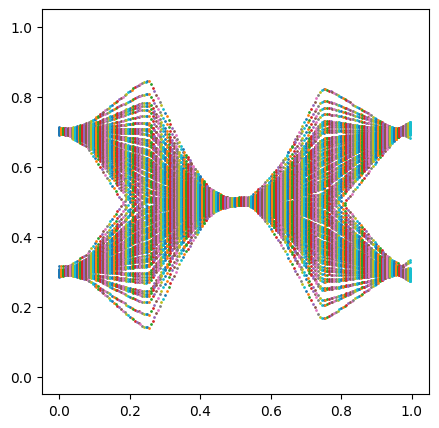}
\caption{Model 1, all atoms}
\end{subfigure}
\begin{subfigure}[t]{0.24\linewidth}
\centering
\includegraphics[width=0.90\textwidth]{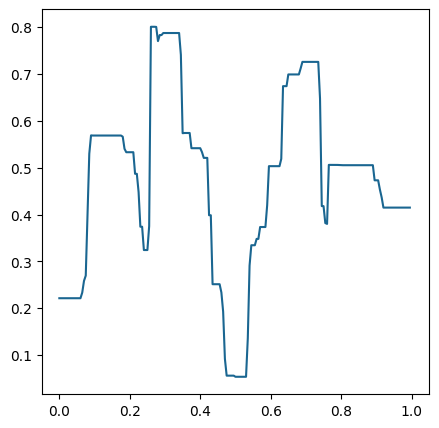}
\caption{Model 2, ave.\! abs.\! der.}
\end{subfigure}
\begin{subfigure}[t]{0.24\linewidth}
\centering
\includegraphics[width=0.90\textwidth]{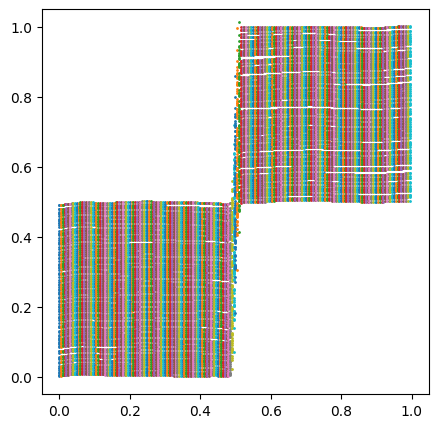}
\caption{Model 2, all atoms}
\end{subfigure}
\begin{subfigure}[t]{0.24\linewidth}
\centering
\includegraphics[width=0.90\textwidth]{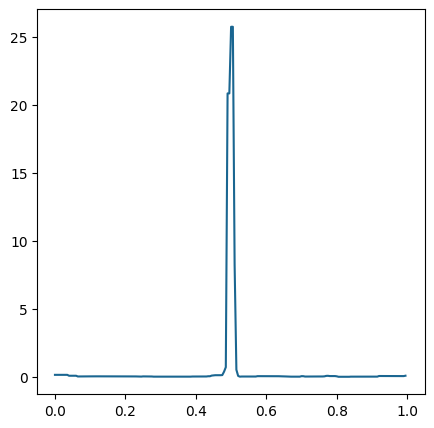}
\caption{Model 2, ave.\! abs.\! der.}
\end{subfigure}
\caption{LipNet under Model 1 with ReLU activation.}
\label{fig:1DActvn}
\end{figure}

\subsubsection{Value of $L$ in (\ref{eq:OutputLayer})}\label{subsubsec:L}
Note that the LipNets discussed in Section \ref{subsec:Experiments} were trained with $L=0.1$. If the normalizing constants in LipNet are exactly computed, $L$ reflects the Lipschitz constant of LipNet, upto the discrepancy in the choice of norms in different layers.  The effect of $L$ in our implementation, however, is rather obscure. Figure \ref{fig:1DL} showcases the performance of LipNets across various $L$ values in Model 1. The comparison in Model 2 is presented in Figure \ref{fig:1DLModel2} in Appendix \ref{sec:MorePlots}. The best choice of $L$ appears to depend on the ground truth model. For Model 3, we compared the performance of $L=0.1$ and $L=1$ and observed no significant differences. Generally, we prefer a smaller $L$; however, smaller values of $L$ tend to exhibit greater sensitivity to other training parameters. For instance, in Model 3, with $L=0.1$, starting enforcing sparsity too soon leads to significantly poorer performance, while the impact on the outcomes for $L=1$ is much less noticeable.
\begin{figure}[htbp]
\centering
\begin{subfigure}[t]{0.24\linewidth}
\centering
\makebox[1pt]{\raisebox{50pt}{\rotatebox[origin=c]{90}{All atoms scattered}}}
\includegraphics[width=0.90\textwidth]{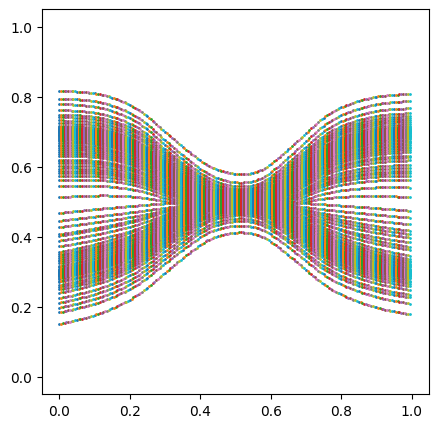}\\
\makebox[1pt]{\raisebox{50pt}{\rotatebox[origin=c]{90}{Traj.\! of 20 atoms}}}
\includegraphics[width=0.90\textwidth]{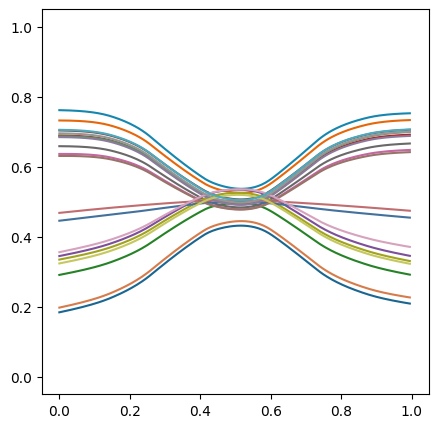}
\caption{$L=0.01$}
\end{subfigure}
\begin{subfigure}[t]{0.24\linewidth}
\centering
\includegraphics[width=0.90\textwidth]{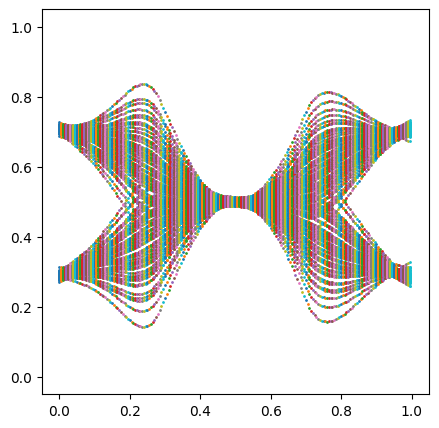}
\includegraphics[width=0.90\textwidth]{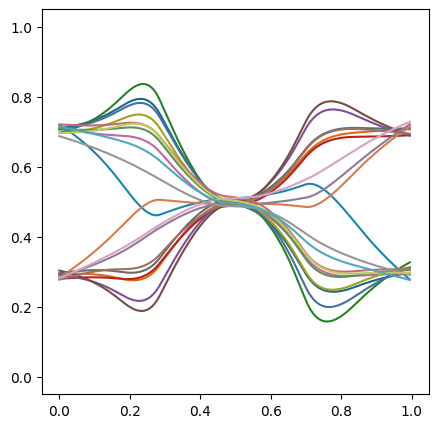}
\caption{$L=0.03$}
\end{subfigure}
\begin{subfigure}[t]{0.24\linewidth}
\centering
\includegraphics[width=0.90\textwidth]{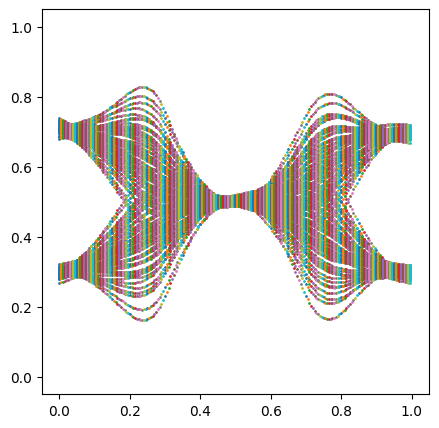}
\includegraphics[width=0.90\textwidth]{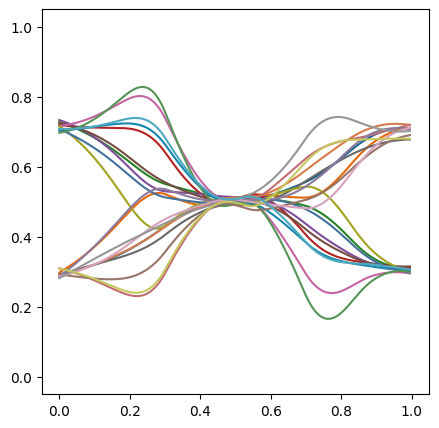}
\caption{$L=1$}
\end{subfigure}
\begin{subfigure}[t]{0.24\linewidth}
\centering
\includegraphics[width=0.90\textwidth]{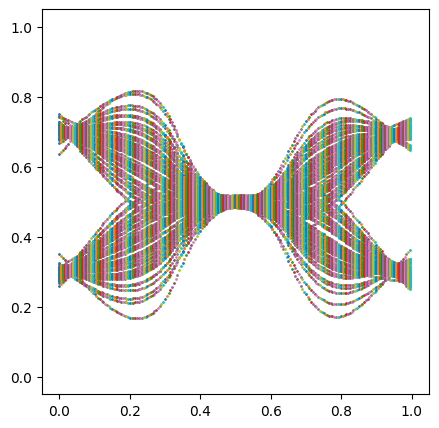}
\includegraphics[width=0.90\textwidth]{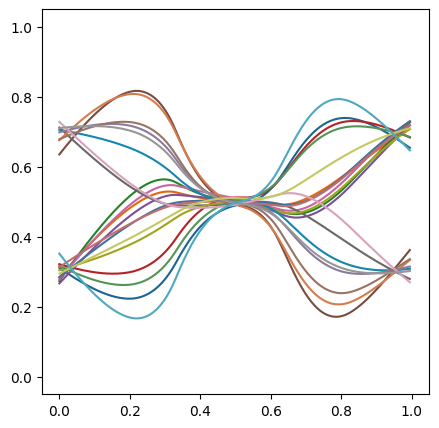}
\caption{$L=3$}
\end{subfigure}
\caption{LipNet under Model 1 with various $L$'s.}
\label{fig:1DL}
\end{figure}

\subsubsection{Momentum $\tau$ in Algorithm \ref{algo:PwrIter}}\label{subsubsec:tau}
In our training of LipNet, we use $\tau=10^{-3}$. Figure \ref{fig:1Dtau} demonstrates the impact of various $\tau$ values on neural estimator's performance in Model 1. It is clear that the performance declines with a $\tau$ that is too large. While we initially speculated that a smaller $\tau$ might cause atoms to exhibit more erratic movements as $x$ changes, observations contradict this hypothesis. We now believe that a suitable $\tau$ value helps prevent neurons from stagnating in the plateau region of the ELU activation function. This is supported by the outcomes observed with $\tau=10^{-6}$, where atom movements are overly simplistic. Additional comparisons in Model 2 are presented in Figure \ref{fig:1Dtau2}.
\begin{figure}[htbp]
\centering
\begin{subfigure}[t]{0.24\linewidth}
\centering
\makebox[1pt]{\raisebox{50pt}{\rotatebox[origin=c]{90}{All atoms scattered}}}
\includegraphics[width=0.90\textwidth]{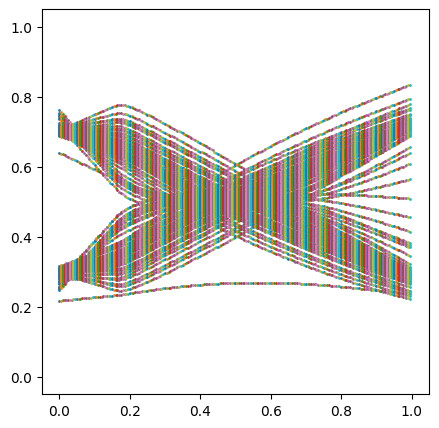}\\
\makebox[1pt]{\raisebox{50pt}{\rotatebox[origin=c]{90}{Traj.\! of 20 atoms}}}
\includegraphics[width=0.90\textwidth]{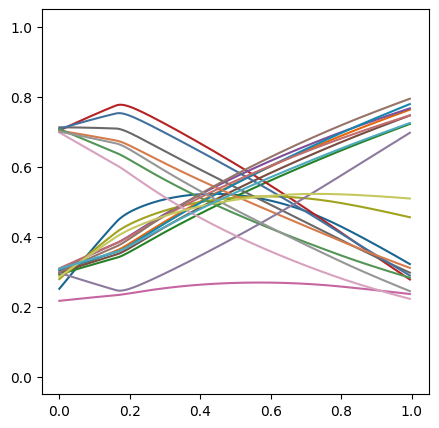}
\caption{$\tau=10^{-1}$}
\end{subfigure}
\begin{subfigure}[t]{0.24\linewidth}
\centering
\includegraphics[width=0.90\textwidth]{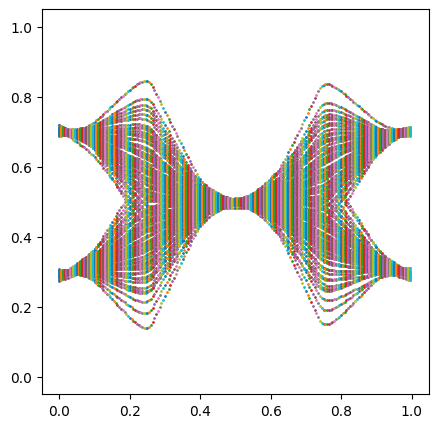}
\includegraphics[width=0.90\textwidth]{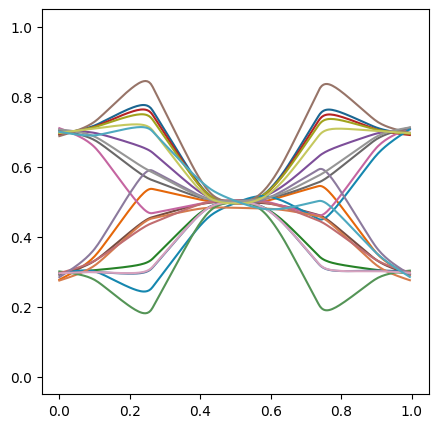}
\caption{$\tau=10^{-2}$}
\end{subfigure}
\begin{subfigure}[t]{0.24\linewidth}
\centering
\includegraphics[width=0.90\textwidth]{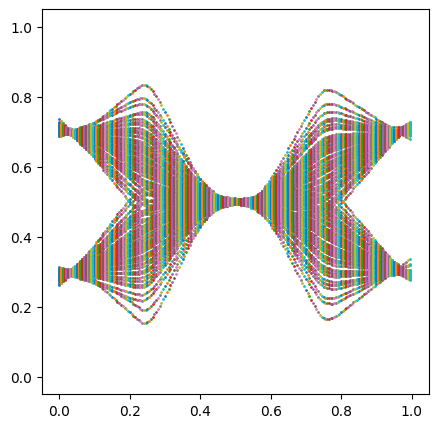}
\includegraphics[width=0.90\textwidth]{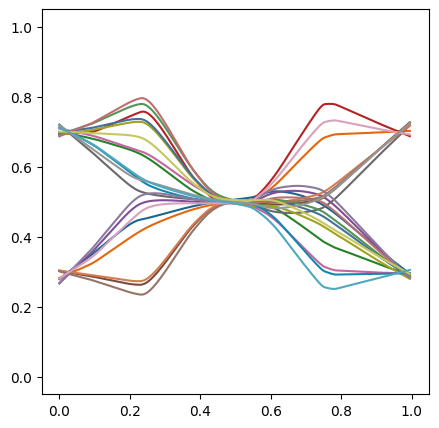}
\caption{$\tau=10^{-5}$}
\end{subfigure}
\begin{subfigure}[t]{0.24\linewidth}
\centering
\includegraphics[width=0.90\textwidth]{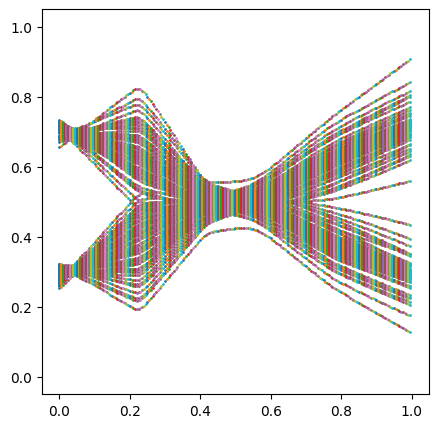}
\includegraphics[width=0.90\textwidth]{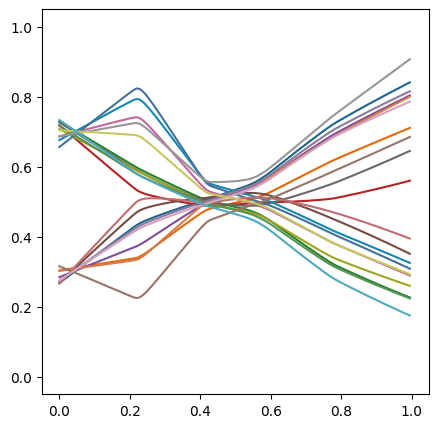}
\caption{$\tau=10^{-6}$}
\end{subfigure}
\caption{LipNet under Model 1 with various $\tau$'s.}
\label{fig:1Dtau}
\end{figure}

Despite considering as a potential improvement the inclusion of batch normalization in the convex potential layer \eqref{eq:CPLayer}, right after the affine transformation, along with a corresponding offset in the position of $\|\mW\|_2$, our experiments with both ELU and ReLU activations, using the default batch normalization momentum of $0.1$, resulted in reduced performance. Lowering said batch normalization momentum often leads to a \textsf{NaN} network.

\section{Weakness and potential improvement}\label{sec:Weakness}
In this section, we provide some discussion on the weakness and possible improvement of our implementation in Section \ref{sec:ImplNN}.

\textbf{Extra correction.}
In more realistic scenarios, the true conditional distribution is often unknown or intractable. In such cases, it is unclear whether a neural estimator offers extra correction over raw estimators. A potential solution to this issue is to train StdNet and LipNet simultaneously. If StdNet and LipNet align more closely with each other than with the raw estimator involved in their training, it is possible that the neural estimators are providing extra corrections.

\textbf{Hyper-parameters for Sinkhorn algorithm.} Our implementation of the Sinkhorn algorithm involves several hyper-parameters: (i) $k$ in Definition \ref{def:kNN}; (ii) $N_{\text{atom}}$ in \eqref{eq:LagrangianDisc}; (iii) $\epsilon$ in \eqref{eq:Ktilde}; (iv) $\gamma$ in \eqref{eq:Tbar}; and (v) additional hyper-parameters listed in Section \ref{sec:Config}. The impact of these hyper-parameters is not yet fully understood. Additionally, an adaptive $\epsilon$ that balances the accuracy and stability of the Sinkhorn iteration is desirable. Furthermore, as illustrated in Section \ref{subsec:Sinkhorn}, enforcing sparsity on the transport plan generally yields better approximations at $x$ where the conditional distribution is more diffusive, but may performs worse where the conditional distribution exhibits atoms. This observation motivates further investigation into a sparsity policy that adjusts according to the indications from the raw estimator.

\textbf{Adaptive continuity.} The impact of hyper-parameters in LipNet also warrants further investigation. In addition, despite the results presented in this study, more evidence is needed to understand how LipNet and its variations perform under various conditions.

\textbf{Scalability.} While the implementation produces satisfactory results when $M$ and $k$ are relatively small (recall that we set $N_{\text{atom}}=k$), our further experiments indicate a scalability bottleneck. For example, in Model 1, significantly increasing $M$ and $k$ does not necessarily improve the performance of neural estimators in a comparable manner. To address this issue, we could experiment with varying the ratios between $N_{\text{atoms}}$ and $k$, rather than setting them equal, in hopes of reducing the strain on the Sinkhorn algorithm. We note that varying the ratio between $N_{\text{atoms}}$ and $k$ requires adjusting the enforced sparsity accordingly. Another issue relates to the dimensions of $\bX$ and $\bY$. In view of the curse of dimensionality in Theorem \ref{thm:ExpectedRatekNN}, our method is inherently suited for low-dimensional settings. Fortunately, in many practical scenarios, the data exhibits low-dimensional structures, such as: (i) the sampling distribution of $X$ concentrating on a low-dimensional manifold; and (ii) the mapping $x \mapsto P_x$ exhibiting low-dimensional dependence. For (i), we might resort to dimension reduction techniques, although an extension of the results in Section \ref{sec:Theory} has yet to be established. For (ii), a data-driven method that effectively leverages the low-dimensional dependence is of significant interest.

\textbf{Conditional generative models.} Utilizing a conditional generative model could potentially lead to further improvements. One advantage of conditional generative models is the ease of incorporating various training objectives. For instance, it can easily adapt to the training objectives in \eqref{eq:DefObj} to accommodate multiple different hyper-parameters simultaneously. We may also incorporate the joint empirical measure in the training process. This flexibility also allows for the integration of specific tail conditions as needed. 

Lastly, we would like to point out an issue observed in our preliminary experiments when utilizing a na\"ive conditional generative model: it may assign excessive probability mass to the blank region between two distinct clusters (for example, in Model 1 around $(x,y)=(0.1,0.5)$). This possibly stems from the inherent continuity of neural networks. One possible solution is to consider using a mixture of multiple conditional generative models.

\newpage

\appendix

\section{Additional plots}\label{sec:MorePlots}

\begin{figure}[H] 
\centering
\begin{subfigure}[H]{1\linewidth}
\centering
\includegraphics[width=0.90\textwidth]{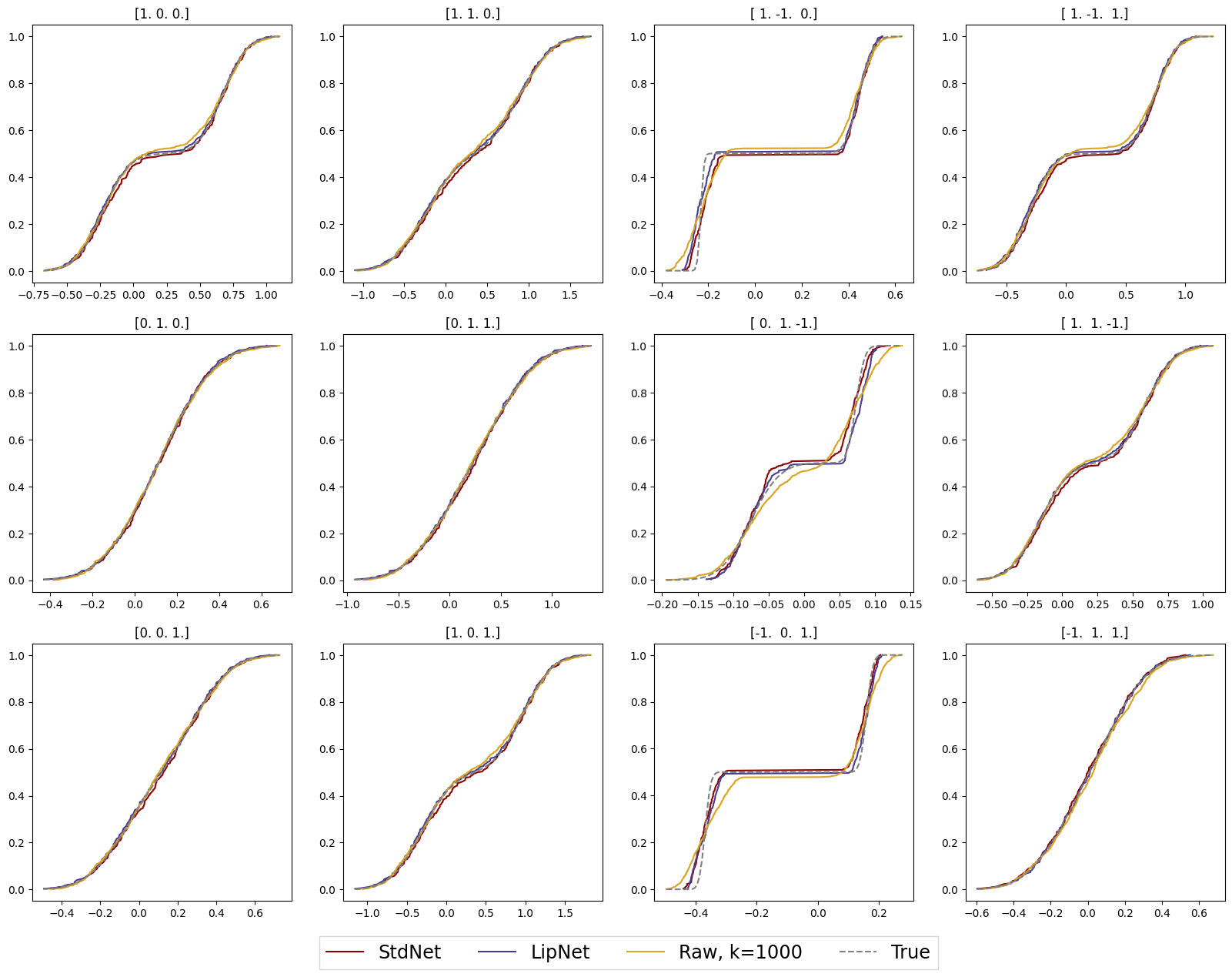}
\end{subfigure}

\caption{\centering Various estimators under Model 3, projections of conditional CDFs, $k=1000$ for $k$-NN estimator.}
\label{fig:3DCondCDFk1e3}
\medskip
\small
This follows the setting of Figure \ref{fig:3DCondCDF}, expect that we set $k=1000$ for the $k$-NN estimator plotted here. The neural estimator is trained under the same setting as that in Figure \ref{fig:3DCondCDF}  Plot titles display the vectors used for projection. Note the difference in the $x$ axis scale. 
\end{figure}

\begin{figure}[htbp]
\centering
\begin{subfigure}[t]{1\linewidth}
\centering
\includegraphics[width=0.90\textwidth]{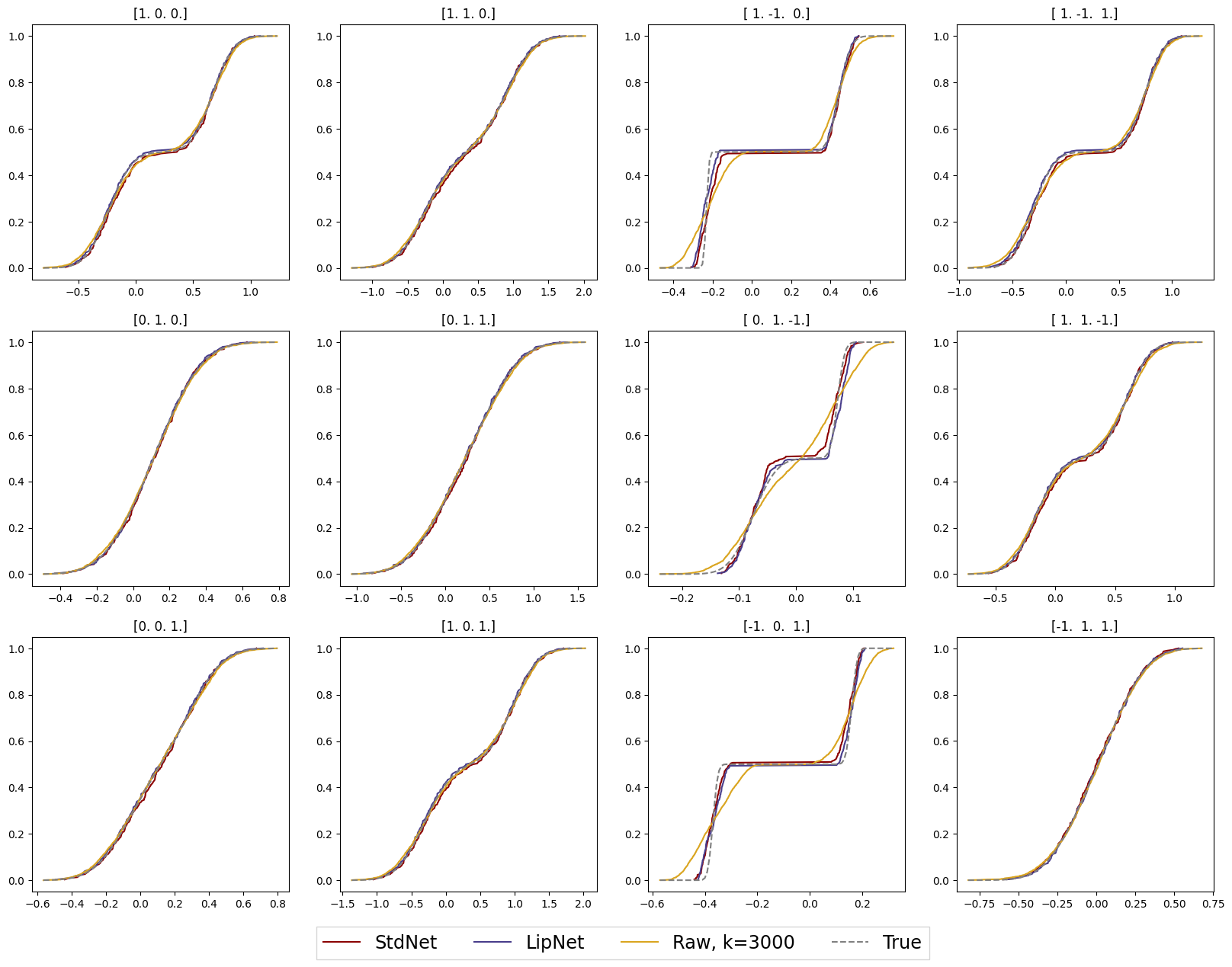}
\end{subfigure}

\caption{\centering Various estimators under Model 3, projections of conditional CDFs, $k=3000$ for $k$-NN estimator.}
\label{fig:3DCondCDFk3e3}
\medskip
\small
This follows the setting of Figure \ref{fig:3DCondCDF}, expect that we set $k=3000$ for the $k$-NN estimator plotted here. The neural estimator is trained under the same setting as that in Figure \ref{fig:3DCondCDF}  Plot titles display the vectors used for projection. Note the difference in the $x$ axis scale. 
\end{figure}

\begin{figure}[htbp]
\centering
\begin{subfigure}[t]{1\linewidth}
\centering
\includegraphics[width=0.90\textwidth]{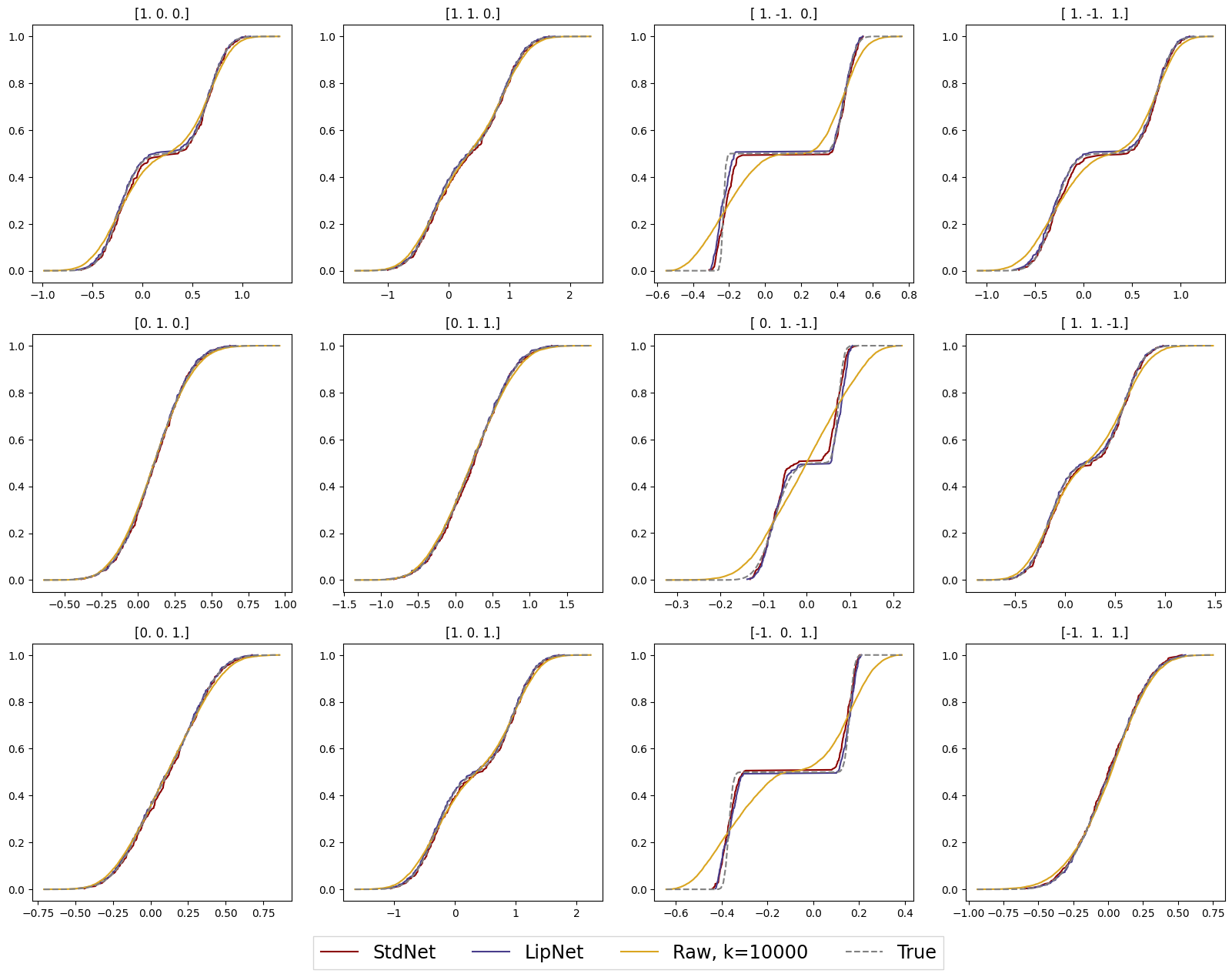}
\end{subfigure}

\caption{\centering Various estimators under Model 3, projections of conditional CDFs, $k=10^4$ for $k$-NN estimator.}
\label{fig:3DCondCDFk1e4}
\medskip
\small
This follows the setting of Figure \ref{fig:3DCondCDF}, expect that we set $k=10^4$ for the $k$-NN estimator plotted here. The neural estimator is trained under the same setting as that in Figure \ref{fig:3DCondCDF}  Plot titles display the vectors used for projection. Note the difference in the $x$ axis scale. 
\end{figure}

\begin{figure}[htbp]
\centering
\begin{subfigure}[t]{1\linewidth}
\centering
\includegraphics[width=0.90\textwidth]{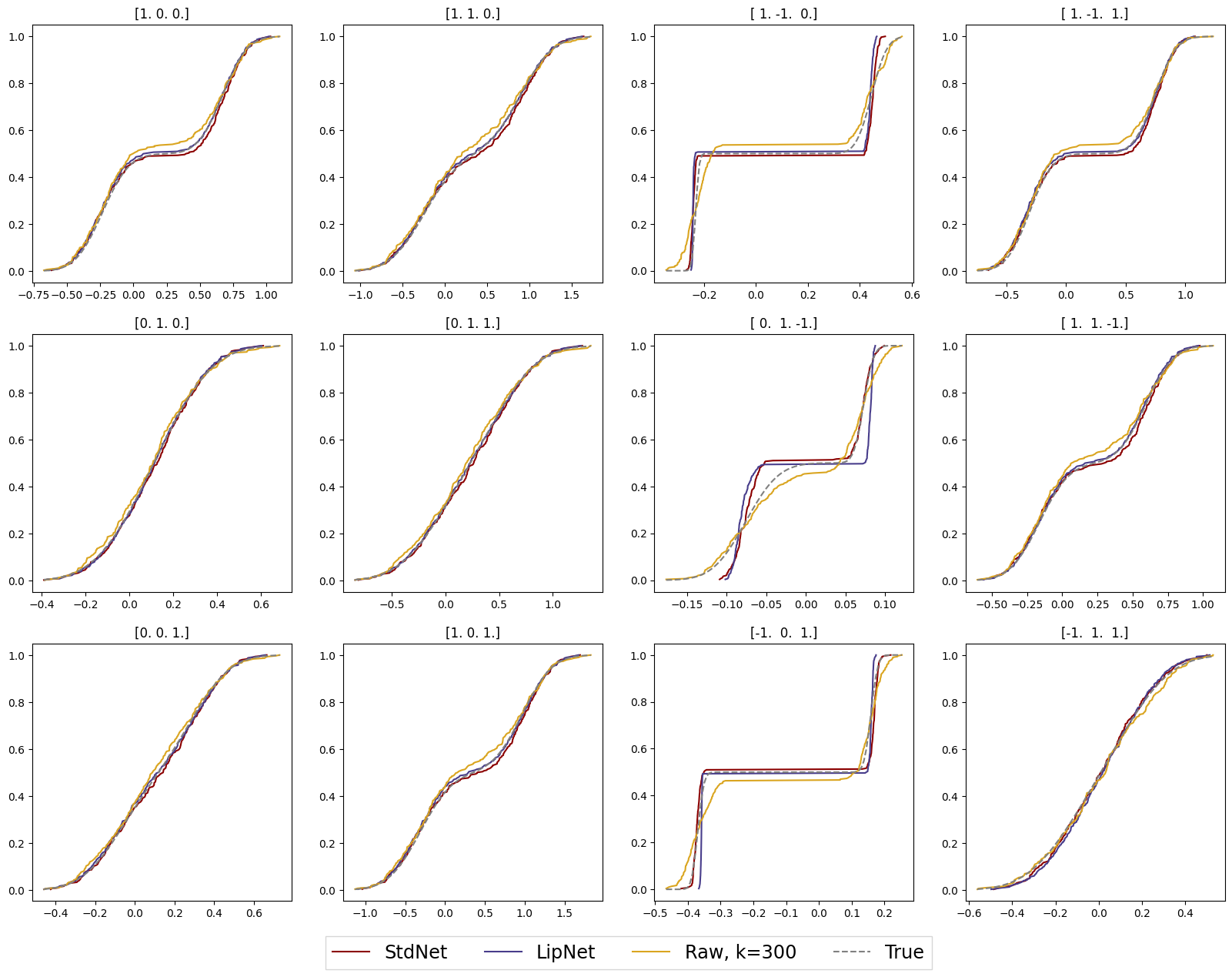}
\end{subfigure}

\caption{\centering Various estimators under Model 3, projections of conditional CDFs, both StdNet and Lipnet are trained without enforced sparsity.}
\label{fig:3DCondCDFNosparse}
\medskip
\small
This follows the setting of Figure \ref{fig:3DCondCDF}, expect that we do not enforce sparsity on the transport plan during the training of the neural estimator. Plot titles display the vectors used for projection. Note the difference in the $x$ axis scale. 
\end{figure}

\begin{figure}[htbp]
\centering
\begin{subfigure}[t]{0.24\linewidth}
\centering
\includegraphics[width=0.95\textwidth]{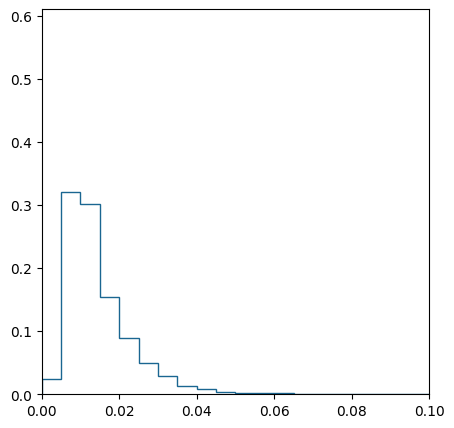}
\caption*{Raw $k=1000$}
\end{subfigure}
\begin{subfigure}[t]{0.24\linewidth}
\centering
\includegraphics[width=0.95\textwidth]{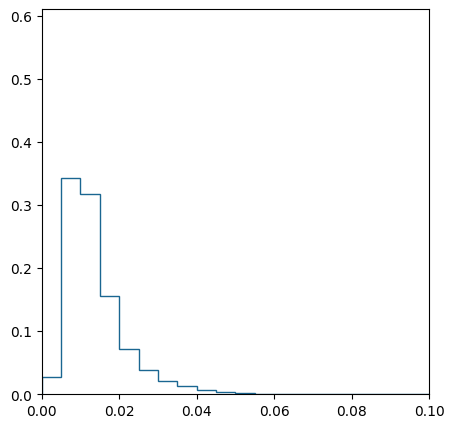}
\caption*{Raw $k=3,000$}
\end{subfigure}
\begin{subfigure}[t]{0.24\linewidth}
\centering
\includegraphics[width=0.95\textwidth]{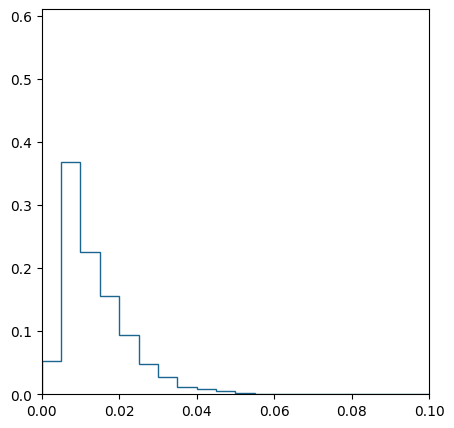}
\caption*{StdNet}
\end{subfigure}
\begin{subfigure}[t]{0.24\linewidth}
\centering
\includegraphics[width=0.95\textwidth]{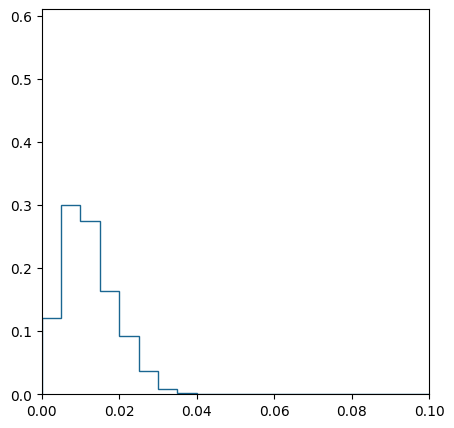}
\caption*{LipNet}
\end{subfigure}
\caption{Histogram of $10,000$ projected Wasserstein-$1$ errors, no enforced sparsity on transport plan.}
\label{fig:HistProjWNoSparsity}
This follows the setting of Figure \ref{fig:HistProjW}, expect that we do not enforce sparsity on the transport plan during the training of the neural estimator. Histograms for raw estimators remain the same. The histogram consists of $20$ uniformly positioned bins between $0$ to $0.1$. The errors of different estimators are computed with the same set of query points and projection vectors. Errors larger than $0.1$ will be placed in the right-most bins. 
\end{figure}



\begin{figure}[htbp]
\centering
\begin{subfigure}[t]{0.24\linewidth}
\centering
\makebox[1pt]{\raisebox{50pt}{\rotatebox[origin=c]{90}{All atoms scattered}}}
\includegraphics[width=0.90\textwidth]{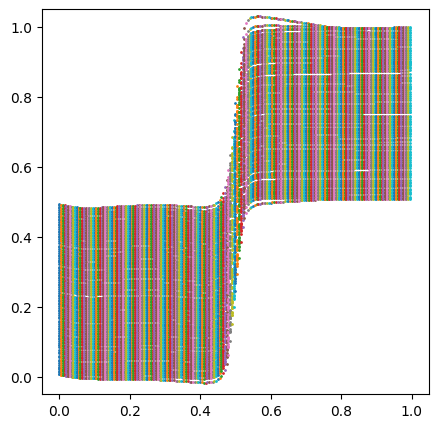}\\
\makebox[1pt]{\raisebox{50pt}{\rotatebox[origin=c]{90}{Traj.\! of 20 atoms}}}
\includegraphics[width=0.90\textwidth]{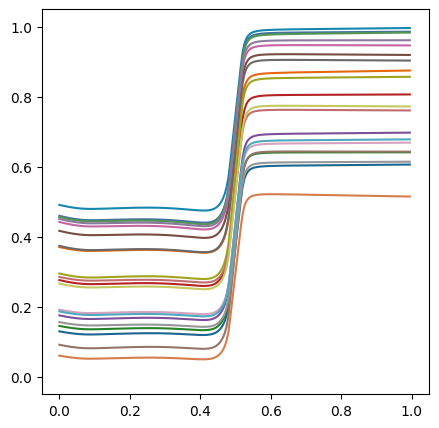}
\caption{$L=0.01$}
\end{subfigure}
\begin{subfigure}[t]{0.24\linewidth}
\centering
\includegraphics[width=0.90\textwidth]{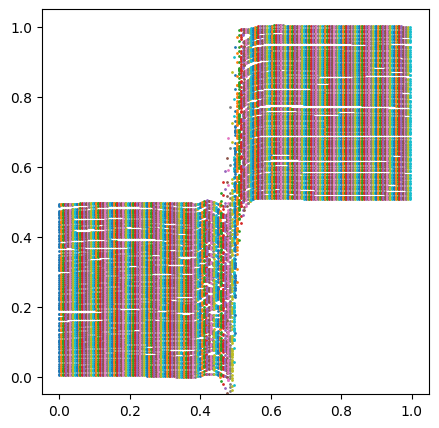}
\includegraphics[width=0.90\textwidth]{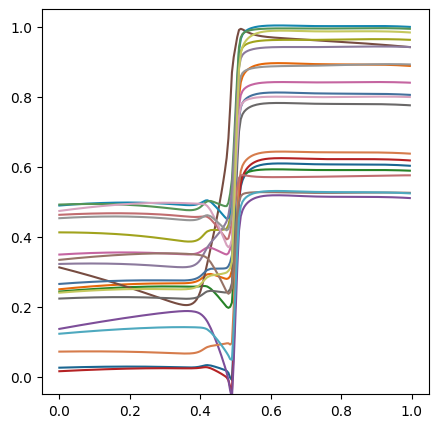}
\caption{$L=0.03$}
\end{subfigure}
\begin{subfigure}[t]{0.24\linewidth}
\centering
\includegraphics[width=0.90\textwidth]{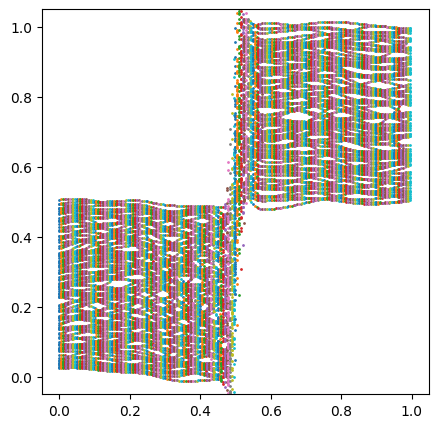}
\includegraphics[width=0.90\textwidth]{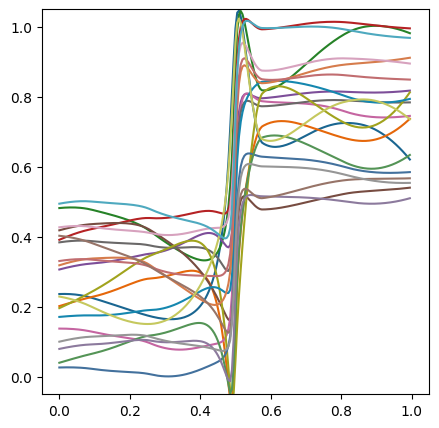}
\caption{$L=1$}
\end{subfigure}
\begin{subfigure}[t]{0.24\linewidth}
\centering
\includegraphics[width=0.90\textwidth]{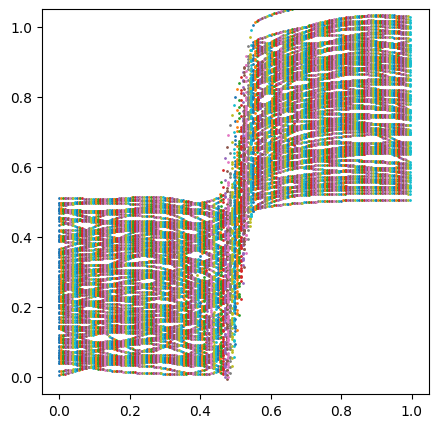}
\includegraphics[width=0.90\textwidth]{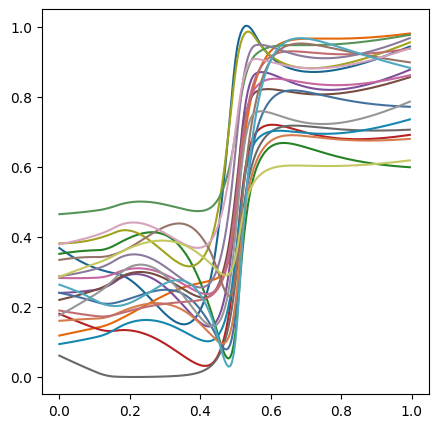}
\caption{$L=3$}
\end{subfigure}
\caption{LipNet under Model 2 with various $L$'s.}
\label{fig:1DLModel2}
\end{figure}

\begin{figure}[htbp]
\centering
\begin{subfigure}[t]{0.24\linewidth}
\centering
\makebox[1pt]{\raisebox{50pt}{\rotatebox[origin=c]{90}{All atoms scattered}}}
\includegraphics[width=0.90\textwidth]{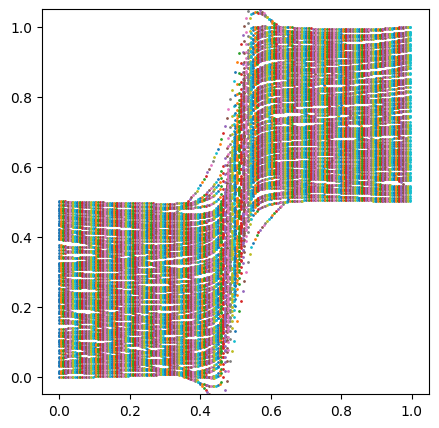}\\
\makebox[1pt]{\raisebox{50pt}{\rotatebox[origin=c]{90}{Traj.\! of 20 atoms}}}
\includegraphics[width=0.90\textwidth]{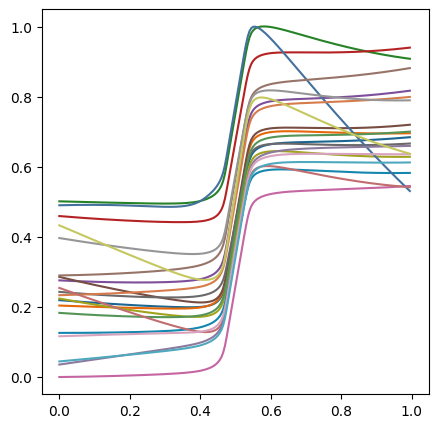}
\caption{$\tau=10^{-1}$}
\end{subfigure}
\begin{subfigure}[t]{0.24\linewidth}
\centering
\includegraphics[width=0.90\textwidth]{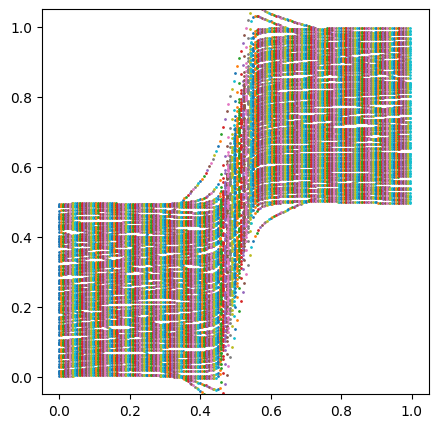}
\includegraphics[width=0.90\textwidth]{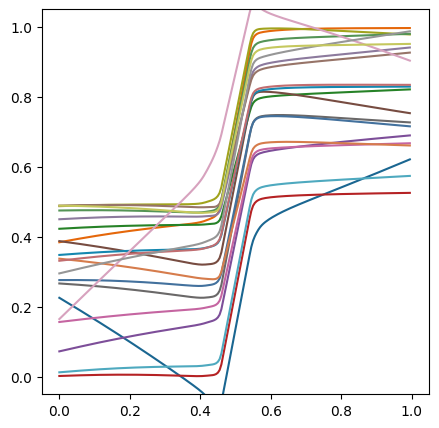}
\caption{$\tau=10^{-2}$}
\end{subfigure}
\begin{subfigure}[t]{0.24\linewidth}
\centering
\includegraphics[width=0.90\textwidth]{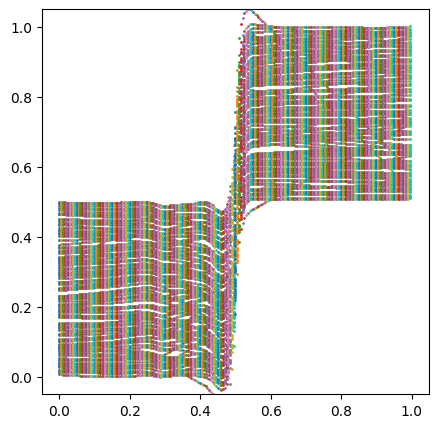}
\includegraphics[width=0.90\textwidth]{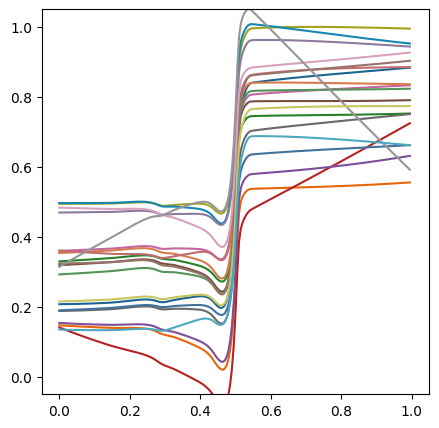}
\caption{$\tau=10^{-5}$}
\end{subfigure}
\begin{subfigure}[t]{0.24\linewidth}
\centering
\includegraphics[width=0.90\textwidth]{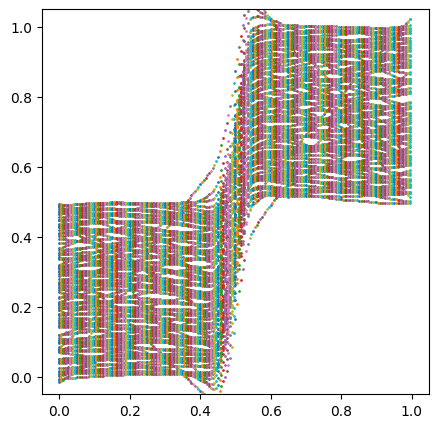}
\includegraphics[width=0.90\textwidth]{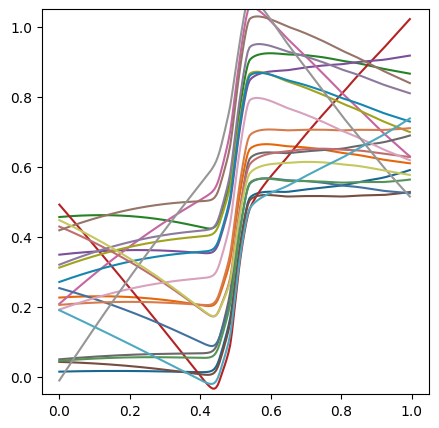}
\caption{$\tau=10^{-6}$}
\end{subfigure}
\caption{LipNet under Model 2 with various $\tau$'s.}
\label{fig:1Dtau2}
\end{figure}

\newpage

\begin{minipage}{0.92\linewidth}
\section{Configuration of network components and training parameters}\label{sec:Config}
The table below summarizes the configuration of the neural network and the training procedure. It applies to both StdNet and LipNet in all models.
\begin{center}
\resizebox{\columnwidth}{!}{
\begin{tabular}[ht]{c c c} 
\hline
\bf{\makecell{Network Component\\[-.2em] /Training parameters}} & \makecell{\bf{Configuration}} & \bf{Note}\\[.3em]
\hline\\[-.7em]

Sample size & 1e4 for Model 1 \& 2, 1e6 for Model 3 & \\[.5em]

$k$ & 100 for Model 1 \& 2, 300 for Model 3 & See Definition \ref{def:kNN} \\[.5em]

Network stucture & \makecell{StdNet: \makecell{Layer-wise residual connection \cite{He2016Deep},\\[-.3em] batch normalization \cite{Ioffe2015Batch} after\\[-.3em] affine transformation} \\[-.1em] LipNet: \makecell{Layer-wise residual connection \cite{He2016Deep}\\[-.3em] with convex potential layer \cite{Meunier2022Dynamical}} } & \\[.5em]

Input dimension & $d_\bX$ & \\[.5em]

Output dimension & $d_\bY\times N_{\text{atom}}$ & \makecell{$N_{\text{atom}}=k$, see \eqref{eq:LagrangianDisc}}\\[.5em]

Number of hidden layers & 5  & \\[.5em]

\makecell{Number of neurons\\[-.3em] each hidden layer} & $2k$  & $k$ as in Definition \ref{def:kNN}\\[.5em]

\makecell{Activation function} & \makecell{StdNet: ReLU\\[-.3em] LipNet: ELU}  & See Section \ref{subsubsec:Actvn} \\[.5em]

$L$ & 0.1 & Introduced in \eqref{eq:OutputLayer} \\[.5em]

$\tau$ & 1e-3 & See Algorithm \ref{algo:PwrIter} \\[.5em]

Optimizer & \makecell{Adam \cite{Kingma2017Adam} with learning rate $10^{-3}$} & \makecell{Learning rate is $0.01$ \\[-.3em] for StdNet in Model 1 \& 2} \\[.5em]

Batch size & \makecell{100 for Model 1 \& 2\\[-.3em] 256 for Model 3} &  \\[.5em]

Number of episodes & 5e3 for Model 1 \& 2, 1e4 for Model 3 & \\[.5em]

RBSP setting & $2^5$ partition, $8$ query points each part & See Section \ref{subsubsec:ANNS}\\[.5em]

Random bisecting ratio & $\sim\Uniform([0.45,0.55])$ & \makecell{ Introduced in Section \ref{subsubsec:ANNS}\\[-.3em]
See also Algorithm \ref{algo:RBSP}} \\[.5em]

\makecell{Ratio for mandatory slicing\\[-.2em] along the longest edge} & 5 & \makecell{ Introduced in Section \ref{subsubsec:ANNS}\\[-.3em]
See also $r_{\text{edge}}$ in Algorithm \ref{algo:RBSP}}\\[.5em]

Number of Sinkhorn iterations & \makecell{$5$, if epoch $\le 500$\\[-.3em] $10$, if epoch $> 500$ } &  \\[.5em]

$\epsilon$ & \makecell{1, if epoch $\le 100$\\[-.3em] $0.1$, if epoch $\in[100,500]$\\[-.3em] $0.05$, if epoch $> 500$ } & Introduced in \eqref{eq:Ktilde}\\[.5em]

Enforced sparsity & \makecell{Off, if epoch $\le 500$\\[-.3em] On, if epoch $>500$ } & See Section \ref{subsec:Sinkhorn} \\[.5em]

$\gamma$ & 0.5 & Introduced in \eqref{eq:Tbar} \\[.5em]

\hline\\
\end{tabular}
}

\end{center}
\end{minipage}

\section{Another set of results on fluctuation}\label{sec:AnotherFluc}

\subsection{On $r$-box estimator}
\begin{theorem}\label{thm:Concenrbox}
Under Assumptions \ref{hyp: kernel lip} and \ref{hyp: data}, and choosing $r$ as in Theorem \ref{thm:ExpectedRaterbox}, let $\nu \in \cP(\bX)$ be dominated by $\lambda_\bX$ with constant $\overline C>0$. Then, there is a constant $C>0$ (which depends only on ${d_\bX},\underline c$, $\overline C$ and the constants involved in $r$), such that, for any $\varepsilon\ge 0$, we have
\begin{align}\label{eq:Concenrbox}
\bP\left[ \int_{\bX} \cW\left(P_x, \Prbox_x\right) \dif \nu(x) \ge \esp{ \int_{\bX} \cW\left(P_x, \Prbox_x\right) \dif \nu(x) } + \varepsilon \right] \le \begin{cases}
\exp\left(-CM^{\frac{2}{{d_\bX}+2}}\varepsilon^2\right), & {d_\bY}=1,2,\\[.5em]
\exp\left(-CM^{\frac{{d_\bX}}{{d_\bX}+{d_\bY}}}\varepsilon^2\right), & {d_\bY}\ge 3.
\end{cases}
\end{align}
\end{theorem}

\begin{proof}
Let $\nu \in \cP(\bX)$ as in the statement of the Theorem. We define 
$$Z := \int_{\bX} \cW(P_x,  \Prbox_x) \ud \nu(x),$$ 
and introduce the following discrete time filtration:  $\sF_0:=\set{\emptyset,\Omega}$ and $\sF_m:=\sigma(\bigcup_{i=1}^m\sigma(X_i,Y_i))$ for $m=1,\dots,M$. We consider the Doob's martingale $Z_m := \esp{Z \,\middle|\, \cF_m},\, m=1,\dots,M$. Note that $Z_M=Z$. We will apply Azuma-Hoeffding inequality (cf. \cite[Corollary 2.20]{Wainwright2019High}) to complete the proof. 

Let us define 
\begin{align}\label{eq: def Dm}
\cD^m := \set{(X_1,Y_1),\dots, (X_m,Y_m), (x_{m+1},y_{m+1}), \dots, (x_M,y_M)}, \quad m=1,\dots,M,
\end{align}
$\cD^0:=\set{(x_\ell,y_\ell)}_{\ell=0}^M$, and $\cD^M:=\cD$. Note that, for all $m<M$, we have, by Assumptions \ref{hyp: data} (i), conditional Fubini-Tonelli theorem, and independent lemma, 
\begin{align*}
Z_m = \int_{\bX} \int_{\left(\bX\times\bY\right)^{M-m}}  \cW\left(P_x, \hat\mu^{\cD^m}_{\cB^r(x)}\right)
\bigotimes_{\ell=m+1}^{M} \psi(\dif x_\ell\dif y_\ell) \nu(\ud x).
\end{align*}
This together with the linearity of integral, the fact that $\psi$ is a probability, and the triangular inequality of $\cW$ implies that for $m=1,\dots,M$,
\begin{align}\label{eq:UBAbsDiffZ}
\left|Z_m-Z_{m-1}\right| \le \int_{\bX}\int_{\left(\bX\times\bY\right)^{M-m+1}}\cW\left(\hat\mu^{\cD^m}_{\cB^r(x)},\hat\mu^{\cD^{m-1}}_{\cB^r(x)}\right)
\bigotimes_{\ell=m}^{M} \psi(\dif x_\ell\dif y_\ell) \nu(\dif x).
\end{align}
Notice that, by definitions \eqref{eq: mu hat} and \eqref{eq: def Dm}, 
\begin{align}\label{eq:UBEventrBox}
\left\{\hat\mu^{\mD_m}_{\cB^r(x)}\neq\hat\mu^{\mD_{m-1}}_{\cB^r(x)}\right\}  
\subseteq \Big\{X_m \in \cB^r(x)\Big\}\cup\Big\{x_m \in \cB^r(x)\Big\}.
\end{align}   
Additionally, by definitions \eqref{eq: mu hat} and \eqref{eq: def Dm} again, on the event that $\left\{\hat\mu^{\mD_m}_{\cB^r(x)}\neq\hat\mu^{\mD_{m-1}}_{\cB^r(x)}\right\}$, we have 
\begin{align}\label{eq:UBWrBox}
\cW\left(\hat\mu^{\cD^m}_{\cB^r(x)}, \hat\mu^{\cD^{m-1}}_{\cB^r(x)}\right) \le \left(1+\sum_{\ell=1}^{m-1} \1_{\cB^r(x)}(X_\ell) + \sum_{\ell = m+1}^M \1_{\cB^r(x)}(x_\ell)\right)^{-1} \le \left(1+\sum_{\ell=m+1}^M \1_{\cB^r(x)}(x_\ell)\right)^{-1}.
\end{align}
Combining \eqref{eq:UBAbsDiffZ},\eqref{eq:UBEventrBox}, \eqref{eq:UBWrBox}, and Fubini-Tonelli theorem, we get
\begin{align}\label{eq:UBAbsDiffZ2}
\left|Z_m-Z_{m-1}\right| & \le \int_{\bX} \int_{B(X_m,2r)\cup B(x_m,2r)}\int_{\bX^{M+1-m}} \left(1+\sum_{\ell=m+1}^M \1_{\cB^r(x)}(x_\ell)\right)^{-1} \bigotimes_{\ell=m+1}^{M} \xi(\dif x_\ell) \nu(\dif x) \xi(\dif x_m)\nonumber\\
&\le \sup_{x_m\in\bX} \int_{B(X_m,2r)\cup B(x_m,2r)}\int_{\bX^{M+1-m}} \left(1+\sum_{\ell=m+1}^M \1_{\cB^r(x)}(x_\ell)\right)^{-1}
\bigotimes_{\ell=m+1}^{M} \xi(\dif x_\ell) \nu(\dif x)
\end{align}
where the $2r$ in the domain of the integral stems from the usage of $\beta^r$ in the definition of $\cB^r$ (see Definition \ref{def:rbox}).  Now, for fixed $x,x_m \in \bX$, we have
\begin{align*}
&\int_{\bX^{M-m+1}} \left(1+\sum_{\ell=m+1}^M \1_{\cB^r(x)}(x_\ell)\right)^{-1}
\bigotimes_{\ell=m+1}^{M} \xi(\dif x_\ell) \\
&\quad= \sum_{\ell=0}^{M-m} \binom{M-m}{\ell} \xi\big(\cB^r(x)\big)^\ell \Big(1-\xi^r\big(\cB^r(x)\big)\Big)^{M-m-\ell} \left(1+\ell\right)^{-1} \\
&\quad= \frac{1}{(M-m+1)\xi\big(\cB^r(x)\big)} \sum_{\ell=0}^{M-m} \binom{M-m+1}{\ell+1} \xi\big(\cB^r(x)\big)^{\ell+1} \Big(1-\xi\big(\cB^r(x)\big)\Big)^{M-m-\ell} \\
&\quad= \frac{1}{(M-m+1)\xi\big(\cB^r(x)\big)} \sum_{\ell=1}^{M-m+1} \binom{M-m+1}{\ell} \xi\big(\cB^r(x)\big)^\ell \Big(1-\xi\big(\cB^r(x)\big)\Big)^{M-m+1-\ell} \\
&\quad= \frac{1- \Big(1-\xi\big(\cB^r(x)\big)\Big)^{M-m+1}}{(M-m+1)\xi\big(\cB^r(x)\big)} \le 1 \wedge \left( (M-m+1)\xi\big(\cB^r(x)\big)\right)^{-1} \le 1 \wedge \left( (M-m+1) \underline c (2r)^{\dX}\right)^{-1}, 
\end{align*}
where we have used Assumption \ref{hyp: data} (ii) in the last inequality. Recall $\overline C$ introduced in Theorem \ref{thm:Concenrbox}. In view of \eqref{eq:UBAbsDiffZ2}, we have 
\begin{align*}
\left|Z_m-Z_{m-1}\right| &\le \sup_{x_m\in\bX} \int_{B(X_m,2r) \cup B(x_m,2r)} 1 \wedge \left( (M-m+1) \underline c (2r)^{\dX}\right)^{-1} \nu(\ud x) \\
&\le 2 \overline C (4r)^{\dX} \left( 1 \wedge \left( (M-m+1) \underline c (2r)^{\dX}\right)^{-1} \right) = \left(\overline C 2^{2\dX+1}r^{\dX}\right) \wedge \frac{\overline C 2^{\dX+1}}{\underline c (M-m+1)} := C_m.
\end{align*}
By Azuma-Hoeffding inequality (cf. \cite[Corollary 2.20]{Wainwright2019High}), one obtains
\begin{align}\label{eq:ConceIneqrBox}
\bP(Z-\esp{Z} \ge \varepsilon) \le \exp\left(-\frac{2\varepsilon^2}{\sum_{m=1}^M C_m^2}\right).
\end{align}
To complete the proof, we substitute in the configuration of Theorem \ref{thm:ExpectedRaterbox}. Since we only aim to investigate the rate of $\sum_{m=1}^M C_m^2$ as $M\to\infty$, we simply set 
\begin{align*}
r = M^{-\frac{1}{d_\bX+d}} \quad \text{with} \quad d := 2 \vee \dY.
\end{align*}
It follows that
\begin{align*}
\sum_{m=1}^M C_m^2 &\sim \sum_{m=1}^{M} M^{-\frac{2\dX}{\dX+d}} \wedge m^{-2} \lesssim \int_1^\infty  M^{-\frac{2\dX}{\dX+d}} \wedge z^{-2} \ud z \sim \int_1^{M^{\frac{\dX}{\dX+d}}} M^{-\frac{2\dX}{\dX+d}} \ud z + \int_{ M^{\frac{\dX}{\dX+d}}}^{\infty} z^{-2} \ud z \sim M^{-\frac{\dX}{\dX+d}},
\end{align*}
which completes the proof.
\end{proof}

\subsection{On $k$-nearest-neighbor estimator}
\begin{theorem}\label{thm:ConcenkNN}
Under Assumptions \ref{hyp: kernel lip} and \ref{hyp: data}, and the choice of $k$ as in Theorem \ref{thm:ExpectedRatekNN}, there is a constant $C>0$ (which depends only on $\underline c$ and the constants involved in $k$), such that, for any $\nu \in \cP(\bX)$ and $\varepsilon\ge 0$, we have
\begin{align}\label{eq:ConcenkNN}
\bP\left[ \int_{\bX} \cW\left(P_x, \PkNN_x\right) \nu(\dif x) \ge \esp{ \int_{\bX} \cW\left(P_x, \PkNN_x\right) \nu(\dif x) } + \varepsilon \right] \le \begin{cases}
\exp\left(-CM^{\frac{2}{{d_\bX}+2}}\varepsilon^2\right), & {d_\bY}=1,2,\\[.5em]
\exp\left(-CM^{\frac{{d_\bY}}{{d_\bX}+{d_\bY}}}\varepsilon^2\right), & {d_\bY}\ge 3.
\end{cases}
\end{align}
\end{theorem}

\begin{proof}[Proof of Theorem \ref{thm:ConcenkNN}]
For notational convenience, we will write $\hat\mu^\mD_{\cN^{k}(x)}$ for $\hat\mu^{\mD}_{\cN^{k,\mD}(x)}$. Clearly, with $\mD=\cD$, we recover $\hat\mu^\cD_{\cN^{k}(x)}=\hat\mu^\cD_{\cN^{k,\cD}(x)}=\PkNN_x$. In what follows, we let
\begin{align*}
Z := \int_{\bX} \cW\left(P_x, \PkNN_x\right) \nu(\dif x). 
\end{align*}
We also define $\sF_0:=\set{\emptyset,\Omega}$ and $\sF_m:=\sigma(\bigcup_{i=1}^m\sigma(X_i,Y_i))$ for $m=1,\dots,M$. The proof relies on an application of Azuma-Hoeffding inequality (cf. \cite[Corollary 2.20]{Wainwright2019High}) to the Doob's martingale $\set{\esp{Z|\sF_m}}_{m=0}^M$. In order to proceed, we introduce a few more notations:
\begin{gather}
\bm x := (x_1,\dots,x_M), \quad \cX := (X_1,\dots,X_M), \\
\cX^m := (X_1,\dots,X_m, x_{m+1},\dots, x_M),\\
\cD^m := \set{(X_1,Y_1),\dots, (X_m,Y_m), (x_{m+1},y_{m+1}), \dots, (x_M,y_M)},\\
\eta^{k,\bm x}_x := \text{ the $k$-th smallest of } \set{\|x_m-x\|_\infty}_{m=1}^M. 
\end{gather}
By independence lemma, we have 
\begin{align*}
\esp{ Z \big| \sF_m } &= \int_{(\bX\times\bY)^{M-m}} \int_{\bX} \cW\left(P_x, \hat\mu^{\cD^m}_{\cN^{k}(x)}\right) \nu(\dif x) \bigotimes_{\ell=m+1}^M \psi(\dif x_\ell\dif y_\ell) \\
&= \int_{(\bX\times\bY)^{M-m}} \int_\bX \cW\left( P_x, \frac1k\left( \sum_{i=1}^m \1_{\|X_i-x\|_\infty \le \eta^{k,\cX^m}_x} \delta_{Y_i} + \sum_{\ell=m+1}^M \1_{\|x_\ell-x\|_\infty \le \eta^{k,\cX^m}_x} \delta_{y_\ell} \right) \right) \nu(\dif x)  \bigotimes_{\ell=m+1}^M \psi(\dif x_\ell\dif y_\ell),
\end{align*}
where we note that $(\bX\times\bY)^{M-m}$ and $\bigotimes_{\ell=m+1}^M \psi(\dif x_\ell\dif y_\ell)$ in the right hand side can be replaced by $(\bX\times\bY)^{M-m+1}$ and $\bigotimes_{\ell=m}^M \psi(\dif x_\ell\dif y_\ell)$ as the integrand is constant in $x_m$ and $\psi$ is a probability measure. Therefore, by Fubini's theorem  and triangle inequality for $\cW$, we have 
\begin{align}\label{eq:UBAbsDiffCondEsp}
&\left| \esp{ Z \big| \sF_m } - \esp{ Z \big| \sF_{m-1} } \right| \nonumber\\
&\quad\le    \int_\bX \int_{(\bX\times\bY)^{M-m+1}}\!\!\!\! \cW\left( \frac1k\left( \sum_{i=1}^m \1_{\|X_i-x\|_\infty \le \eta^{k,\cX^m}_x} \delta_{Y_i} \! + \!\!\! \sum_{\ell=m+1}^M \1_{\|x_\ell-x\|_\infty \le \eta^{k,\cX^m}_x} \delta_{y_\ell} \right) \right.,   \nonumber \\
&\;\quad\qquad\qquad\qquad\qquad \left. \frac1k\left( \sum_{i=1}^{m-1} \1_{\|X_i-x\|_\infty \le \eta^{k,\cX^{m-1}}_x} \delta_{Y_i} \! + \! \sum_{\ell=m}^M \1_{\|x_\ell-x\|_\infty \le \eta^{k,\cX^{m-1}}_x} \delta_{y_\ell} \right) \! \right)  \bigotimes_{\ell=m}^M \psi(\dif x_\ell\dif y_\ell) \dif x.
\end{align}
Above, the only difference between the two measures inside $\cW$ is the $m$-th summand. Due to the definition of $\cW$ and the boundedness of $\bX$, the transport cost induced by altering the $m$-th summand is at most $k^{-1}$. It follows that 
\begin{align}\label{eq:UBAbsDiffEspZFm1}
\left| \esp{ Z \big| \sF_m } - \esp{ Z \big| \sF_{m-1} } \right| \le \frac1k, \quad m=1,\dots,M.
\end{align}

Below we further refine the upper bound of the absolute difference in the left hand side of \eqref{eq:UBAbsDiffCondEsp} when $m=1,\dots,M-k$. For the integrand in the right hand side of \eqref{eq:UBAbsDiffCondEsp} to be positive, it is necessary that 
\begin{align*}
\1_{\|X_m-x\|_\infty \le \eta^{k,\cX^m}_x} + \1_{\|x_m-x\|_\infty \le \eta^{k,\cX^{m-1}}_x} \ge 1.
\end{align*}
This, together with the tie breaking rule stipulated in Definition \ref{def:kNN}, further implies that 
$$\1_{A^m_1} + \1_{A^m_2} \ge 1, $$ 
where
\begin{gather*}
A_1^m := \left\{ \text{at most $(k-1)$ of $x_\ell,\ell=m+1,\dots,M-m$, falls into } B^{\|X_m-x\|_{\infty}}_{x} \right\}, \\
A_2^m := \left\{ \text{at most $(k-1)$ of $x_\ell,\ell=m+1,\dots,M-m$, falls into } B^{\|x_m-x\|_{\infty}}_{x} \right\}.
\end{gather*}
Combining the above with the reasoning leading to \eqref{eq:UBAbsDiffEspZFm1}, we yield
\begin{align*}
&\left| \esp{ Z \big| \sF_m } - \esp{ Z \big| \sF_{m-1} } \right| \nonumber \\
&\quad \le \frac1k \left( \int_\bX \int_{(\bX\times\bY)^{M-m}} \!\! \1_{A_1^m}  \bigotimes_{\ell=m+1}^M \xi(\dif x_\ell) \nu(\dif x) + \int_\bX \int_{(\bX\times\bY)^{M-m+1}} \!\! \1_{A_2^m} \bigotimes_{\ell=m}^M \xi(\dif x_\ell) \nu(\dif x)  \right)
\end{align*}
Above, we have replaced $\psi$ in \eqref{eq:UBAbsDiffCondEsp} by $\xi$ because $A_1^m$ and $A_2^m$ no longer depend on $y_\ell,\ell=m+1,\dots,M$. The analogue applies to the domain of integral as well. We continue to have
\begin{align}\label{eq:UBAbsDiffEspZFmDecomp}
&\left| \esp{ Z \big| \sF_m } - \esp{ Z \big| \sF_{m-1} } \right| \nonumber \\
&\quad \le \frac{1}k \left( \int_\bX \int_{(\bX\times\bY)^{M-m}} \!\! \1_{A_1^m}  \bigotimes_{\ell=m+1}^M \xi(\dif x_\ell) \nu(\dif x) + \int_\bX \int_{(\bX\times\bY)^{M-m+1}} \!\! \1_{A_2^m} \bigotimes_{\ell=m}^M \xi(\dif x_\ell) \nu(\dif x)  \right) =: \frac{1}k (I_1^m + I_2^m).
\end{align}

Regarding $I^1_m$ defined in \eqref{eq:UBAbsDiffEspZFmDecomp}, note that by Assumption \ref{hyp: data},
\begin{align*}
\int_{(\bX\times\bY)^{M-m}} \1_{A_1^m}  \bigotimes_{\ell=m+1}^M \xi(\dif x_\ell) = \bP\left[ \text{at most $(k-1)$ of $\check X_1,\dots,\check X_{M-m}$ falls into } B^{\|x'-x\|_{\infty}}_{x} \right] \Big|_{x'=X_m},
\end{align*}
where $\check X_1,\dots,\check X_{M-m}\stackrel{\text{i.i.d.}}{\sim}\xi$. Below we define a CDF $G(r):=\underline c r^d, r\in[0,\underline c^{-\frac1d}]$. By Assumption \ref{hyp: data} (ii), for any $x,x'\in\bX$, we have
\begin{align*}
\int_{\bX} \1_{\check x\in B_x^{\|x'-x\|_\infty}} \xi(\dif\check x) \ge \underline c \int_{\bX} \1_{\check x\in B_x^{\|x'-x\|_\infty}} \dif \check x \ge \underline c \int_{\bX} \1_{\check x\in B_{\bm 0}^{\|x'-x\|_\infty}} \dif \check x = G(\|x'-x\|_\infty),
\end{align*}
where we have used the fact that $\|x'-x\|_\infty\le 1 \le \underline c^{-\frac1d}$ in the last equality. It follows from Lemma \ref{lem:BinStochDom} that 
\begin{align*}
\int_{(\bX\times\bY)^{M-m}} \1_{A_1^m}  \bigotimes_{\ell=m+1}^M \xi(\dif x_\ell) \le \sum_{j=0}^{k-1} \binom{M-m}{j} G(\|X_m-x\|_\infty)^{j} \big(1-G(\|X_m-x\|_\infty)\big)^{M-m-j},
\end{align*}
and thus, by letting $U\sim\Uniform(\bX)$,
\begin{align*}
I_1^m &\le \int_{\bX} \sum_{j=0}^{k-1} \binom{M-m}{j} G(\|X_m-x\|_\infty)^{j} \big(1-G(\|X_m-x\|_\infty)\big)^{M-m-j} \dif x \\
&= \esp{ \sum_{j=0}^{k-1} \binom{M-m}{j} G(\|x'-U\|_\infty)^{j} \big(1-G(\|x'-U\|_\infty)\big)^{M-m-j} } \Big|_{x'=X_m},
\end{align*}
where we note that the upper bounded no longer involves $\nu$. For $x'\in\bX$, it is obvious that
\begin{align*}
\bP\big[\|x'-U\|_\infty \le r\big] \ge \bP\big[\|U\|_\infty\le r\big], \quad r\in\bR,
\end{align*}
i.e., $\|U\|_\infty$ stochastically dominates $\|x'-U\|_\infty$. Note additionally that, by Lemma \ref{lem:BinStochDom} again, below is a non-decreasing function,
\begin{align*}
r \mapsto \sum_{j=0}^{k-1}\binom{M-m}{j} G(r)^{j} \big(1-G(r)\big)^{M-m-j}.
\end{align*}
Consequently,
\begin{align*}
I_1^m \le \esp{ \sum_{j=0}^{k-1} \binom{M-m}{j} G(\|U\|_\infty)^{j} \big(1-G(\|U\|_\infty)\big)^{M-m-j} \dif x }.
\end{align*}
Since $\|U\|_\infty$ has CDF $r\mapsto r^{d_\bX}, r\in[0,1]$ and $G(r)=\underline c r^d, r\in[0,\underline c^{-\frac1d}]$, we continue to obtain
\begin{align*}
I^m_1 \le  \sum_{j=0}^{k-1} \binom{M-m}{j}  \int_{r=0}^1  \underline c r^{{d_\bX}j} (1- \underline c r^{d_\bX})^{M-m-j} \dif r^{d_\bX} \le \underline c^{-1} \sum_{j=0}^{k-1} \frac{(M-m)!}{j! (M-m-j)!}  \int_0^1 r^j (1-r)^{M-m-j} \dif r.
\end{align*}
With a similar calculation as in \eqref{eq:EspZ(m)x}, which involves beta distribution and gamma function, we arrive at 
\begin{align}\label{eq:UBI1m}
I^m_1 \le \underline c \sum_{j=0}^{k-1} \frac{(M-m)!}{j! (M-m-j)!}  \frac{j! (M-m-j)!}{(M-m+1)!} \le  \frac{\underline c^{-1} k}{M-m}.
\end{align}

Regarding $I^m_2$ defined in \eqref{eq:UBAbsDiffEspZFmDecomp}, we first let $\check X_0, \check X_1,\dots,\check X_{M-m}\stackrel{\text{i.i.d.}}{\sim}\xi$. Then, note that 
\begin{align*}
\int_{(\bX\times\bY)^{M-m}} \1_{A_2^m}  \bigotimes_{\ell=m+1}^M \xi(\dif x_\ell) &\le \bP\left[ \text{at most $(k-1)$ of $\check X_1,\dots,\check X_{M-m}$ falls into } B^{\|\check X_0-x\|_{\infty}}_{x} \right] \\
&\le \binom{M-m-1}{k-1}\binom{M-m}{k}^{-1} = \frac{k}{M-m},
\end{align*}
where the inequality in the second line is due to the symmetry stemming from Assumption \ref{hyp: data} (i), and the fact that congestion along with the tie-breaking rule specified in Definition \ref{def:kNN} may potentially rules out certain permutations. Consequently,
\begin{align}\label{eq:UBI2m}
I^m_2 \le \frac{k}{M-m}.
\end{align}

Putting together \eqref{eq:UBAbsDiffEspZFm1}, \eqref{eq:UBAbsDiffEspZFmDecomp}, \eqref{eq:UBI1m}, and \eqref{eq:UBI2m}, we yield
\begin{align*}
\left| \esp{ Z \big| \sF_m } - \esp{ Z \big| \sF_{m-1} } \right| \le C_m := \frac{\overline{C}(\underline c^{-1}+1)}{M-m}\wedge\frac1k,\quad m=1,\dots,M.
\end{align*}
By Azuma-Hoeffding inequality (cf. \cite[Corollary 2.20]{Wainwright2019High}),
\begin{align}\label{eq:ConceIneqkNN}
\bP\left[ \int_{\bX} \cW\left(P_x, \hat\mu^\cD_{\cN^{k}(x)}\right) \nu(\dif x) - \esp{ \int_{\bX} \cW\left(P_x, \hat\mu^\cD_{\cN^{k}(x)}\right) \nu(\dif x) } \ge \varepsilon \right] \le \exp\left( - \frac{\varepsilon^2}{2\sum_{m=1}^M C_m^2} \right), \quad \varepsilon\ge 0.
\end{align}
To complete the proof, we substitute in the configuration of Theorem \ref{thm:ExpectedRatekNN}. Below we only investigate the rate of $\sum_{m=1}^M C_m^2$ as $M\to\infty$, and do not keep track of the constant.  For simplicity, we set
\begin{align*}
k = k\sim M^{\frac{d}{d_\bX+d}} \quad \text{with} \quad d := 2 \vee \dY 
\end{align*}
It follows that 
\begin{align*}
\sum_{m=1}^M C_m^2 \sim \sum_{m=1}^{\lfloor M - M^{\frac{d}{{d_\bX}+{d}}} \rfloor} \frac{1}{(M-m)^2} + \frac{M^{\frac{d}{{d_\bX}+d}}}{M^{\frac{2d}{{d_\bX}+d}}} \sim \int_{M^{\frac{d}{{d_\bX}+d}}}^\infty \frac{1}{r^2} \dif r + \frac{1}{M^{\frac{d}{{d_\bX}+d}}} \sim M^{-\frac{d}{d_\bX+d}},
\end{align*} 
which completes the proof.
\end{proof}

\bibliographystyle{alpha}
\bibliography{refs}

\end{document}